\pdfoutput=1
\newif\ifdraft \draftfalse 
\newif\ifjasa \jasafalse

\jasatrue

\documentclass[11pt]{article}
\usepackage{jcd}
\usepackage[numbers]{natbib}
\usepackage[section]{placeins}
\usebiggeometry
\usepackage{xspace}
\usepackage{xcolor}
\definecolor{darkblue}{rgb}{0,0,.5}
\usepackage[colorlinks=true,allcolors=darkblue]{hyperref}
\usepackage{algorithm}
\usepackage[noend]{algpseudocode}
\usepackage{bbm}
\usepackage{macros}
\usepackage{tikz}
\usetikzlibrary{matrix}
\usepackage{pgfplots}
\usepackage{overpic}
\usepackage{makecell}
\usepgfplotslibrary{fillbetween}
\usetikzlibrary{patterns}
\usepackage{authblk}

\def\centerarc(#1)(#2:#3:#4){%
  ($(#1)+({#4*cos(#2)},{#4*sin(#2)})$) arc (#2:#3:#4); }
\usetikzlibrary{calc}

\makeatletter
\long\def\@makecaption#1#2{
  \vskip 0.8ex
  \setbox\@tempboxa\hbox{\small {\bf #1:} #2}
  \parindent 1.5em  
  \dimen0=\hsize
  \advance\dimen0 by -3em
  \ifdim \wd\@tempboxa >\dimen0
  \hbox to \hsize{
    \parindent 0em
    \hfil 
    \parbox{\dimen0}{\def\baselinestretch{0.96}\small
      {\bf #1.} #2
    } 
    \hfil}
  \else \hbox to \hsize{\hfil \box\@tempboxa \hfil}
  \fi
}
\makeatother

\usepackage{soul}

\usepackage{color-edits}
\addauthor{Ryan}{red}
\newcommand{\rynote}[1]{\ifdraft \Ryancomment{#1}\else\ignorespaces\fi}
\addauthor{John}{green}

\addauthor{Abhishek}{blue}

\addauthor{Aaron}{purple}

\newcommand{\decimal}{\mathsf{Dec}}
\newcommand{\loss}{\ell}

\begin{document}

\title{Protection Against Reconstruction and Its Applications in Private
  Federated Learning}

\author[1]{Abhishek Bhowmick}
\author[1,2]{John Duchi}
\author[1]{Julien Freudiger}
\author[1]{Gaurav Kapoor}
\author[1]{Ryan Rogers}
\affil[1]{ML Privacy Team, Apple, Inc.}
\affil[2]{Stanford University}
\date{}

\maketitle 

\begin{abstract}
  In large-scale statistical learning, data collection and model fitting are
  moving increasingly toward peripheral devices---phones, watches, fitness
  trackers---away from centralized data collection. Concomitant with this
  rise in decentralized data are increasing challenges of maintaining
  privacy while allowing enough information to fit accurate, useful
  statistical models. This motivates local notions of privacy---most
  significantly, local differential privacy, which provides strong
  protections against sensitive data disclosures---where data is obfuscated
  before a statistician or learner can even observe it, providing strong
  protections to individuals' data. Yet local privacy as traditionally
  employed may prove too stringent for practical use, especially in modern
  high-dimensional statistical and machine learning problems.  Consequently,
  we revisit the types of disclosures and adversaries against which we
  provide protections, considering adversaries with limited prior
  information and ensuring that with high probability, ensuring they cannot
  reconstruct an individual's data within useful tolerances. By
  reconceptualizing these protections, we allow more useful data
  release---large privacy parameters in local differential privacy---and we
  design new (minimax) optimal locally differentially private mechanisms for
  statistical learning problems for \emph{all} privacy levels.  We thus
  present practicable approaches to large-scale locally private model
  training that were previously impossible,
  showing theoretically and empirically that we can fit large-scale image
  classification and language models with little degradation in utility.
\end{abstract}

\section{Introduction}

New, more powerful computational hardware and access to substantial amounts
of data has made fitting accurate models for image classification, text
translation, physical particle prediction, astronomical observation, and
other predictive tasks possible with previously infeasible
accuracy~\cite{BaldiSaWh14,Adelman-McCarthyEtAl08, LeCunBeHi15}.  In many
modern applications, data comes from measurements on small-scale devices
with limited computation and communication ability---remote sensors,
phones, fitness monitors---making fitting large scale predictive models
both computationally and statistically challenging.  Moreover, as more modes
of data collection and computing move to peripherals---watches,
power-metering, internet-enabled home devices, even lightbulbs---issues of
privacy become ever more salient.

Such large-scale data collection motivates substantial work.  Stochastic
gradient methods are now the \emph{de facto} approach to large-scale
model-fitting~\cite{Zhang04, BottouBo07, NemirovskiJuLaSh09, DuchiHaSi11},
and recent work of \citet{McMahanMoRaHaAr17} describes systems (which they
term \emph{federated learning}) for aggregating multiple stochastic
model-updates from distributed mobile devices. Yet even if only updates to a
model are transmitted, leaving all user or participant data on user-owned
devices, it is easy to compromise the privacy of users~\citep{FJR15,MSCS18}.
To see why this
issue arises, consider any generalized linear model based on a data vector
$x$, target $y$, and with loss of the form $\loss(\theta; x, y) =
\phi(\<\theta, x\>, y)$. Then $\nabla_\theta \loss(\theta;
x, y) = c\cdot  x$ for some $c\in \R$,
a scalar multiple of the user's clear data $x$---a clear
compromise of privacy.  In this paper, we describe an approach to fitting
such large-scale models both privately and practically.

A natural approach to addressing the risk of information disclosure is to
use differential privacy~\cite{DworkMcNiSm06}, in which one defines a mechanism
$M$, a randomized mapping from a sample $\bbx$ of data points to some space
$\mc{Z}$, which is \emph{$\diffp$-differentially private} if
\begin{equation}
  \label{eqn:central-diffp}
  \frac{\P(M(\bbx) \in S)}{\P(M(\bbx') \in S)}
  \le e^\diffp
\end{equation}
for all samples $\bbx$ and $\bbx'$ differing in at most one entry. Because
of its strength and protection properties, differential privacy and its
variants are the standard privacy definition in data
analysis and machine learning~\cite{ChaudhuriMoSa11,DworkRo14,
  DajaniLaSiKiReMaGaDaGrKaKiLeScSeViAb17}. Nonetheless, implementing such an
algorithm presumes a level of trust between users and a centralized data
analyst, which may be undesirable or even untenable, as the data analyst has
essentially unfettered access to a user's data.  Another approach to
protecting individual updates is to use secure multiparty computation (SMC),
sometimes in conjunction with differential privacy
protections~\cite{SMCFedLearn17}.  Traditional approaches to SMC require
substantial communication and computation, making them untenable for
large-scale data collection schemes, and \citet{SMCFedLearn17} address a
number of these, though individual user communication and computation still
increases with the number of users submitting updates and requires multiple
rounds of communication, which may be unrealistic when estimating models
from peripheral devices.

An alternative to these approaches is to use \emph{locally private}
algorithms~\cite{Warner65, EvfimievskiGeSr03, DuchiJoWa18}, in which an
individual keeps his or her data private even from the data collector. Such
scenarios are natural in distributed (or federated) learning scenarios,
where individuals provide data from their devices~\cite{FedLearn17,
  ApplePrivacy17} but wish to maintain privacy. In our learning context,
where a user has data $x \in \mc{X}$ that he or she wishes to remain
private, a randomized mechanism $M : \mc{X} \to \mc{Z}$ is
\emph{$\diffp$-local differentially private} if for all $x, x' \in \mc{X}$
and sets $S \subset \mc{Z}$,
\begin{equation}
  \label{eqn:diffp}
  \frac{\P(M(x) \in S)}{\P(M(x') \in S)}
  \le e^\diffp.
\end{equation}
Roughly, a mechanism satisfying inequality~\eqref{eqn:diffp} guarantees that
even if an adversary knows that the initial data is one of $x$ or $x'$, the
adversary cannot distinguish them given an outcome $Z$ (the probability of
error must be at least $1 / (1 + e^\diffp)$)~\cite{WassermanZh10}.  Taking
as motivation this testing definition, the ``typical'' recommendation for
the parameter $\diffp$ is to take $\diffp$ as a small
constant~\cite{WassermanZh10, DworkMcNiSm06, DworkRo14}.

Local privacy protections provide numerous benefits: they allow easier
compliance with regulatory strictures, reduce risks (such as hacking)
associated with maintaining private data, and allow more transparent
protection of user privacy, because unprotected private data never leaves an
individual's device. Yet substantial work in the statistics, machine
learning, and computer science communities has shown that local differential
privacy and its relaxations cause nontrivial challenges for learning
systems. Indeed, \citet*{DuchiJoWa13_focs,DuchiJoWa18} and \citet{DuchiRo19}
show that in a minimax (worst case population distribution) sense, learning
with local differential privacy \emph{must} suffer a degradation in sample
complexity that scales as $d / \min\{\diffp, \diffp^2\}$, where $d$ is the
dimension of the problem; taking $\diffp$ as a small constant here suggests
that estimation in problems of even moderate dimension will be
challenging. \citet{DuchiRu18b} develop a complementary approach, arguing
that a worst-case analysis is too conservative and may not accurately
reflect the difficulty of problem instances one actually encounters, so that
an instance-specific theory of optimality is necessary. In spite of this
instance-specific optimality theory for locally private procedures---that
is, fundamental limits on learning that apply to the \emph{particular
  problem at hand}---\citeauthor{DuchiRu18b}'s results suggest that local
notions of privacy as currently conceptualized restrict some of the
deployments of learning systems.

We consider an alternative
conceptualization of privacy protections and the concomitant guarantees from
differential privacy and the likelihood ratio bound~\eqref{eqn:diffp}. The
testing interpretation of differential privacy suggests that when $\diffp
\gg 1$, the definition~\eqref{eqn:diffp} is almost vacuous.  We argue that, at
least in large-scale learning scenarios, this testing interpretation is
unrealistic, and allowing mechanisms with $\diffp \gg 1$ may provide
satisfactory privacy protections, especially when there are multiple
layers of protection. Rather than providing protections against
arbitrary inferences, we wish to provide protection against
accurate \emph{reconstruction} of an individual's data $x$. In the large
scale learning scenarios we consider, an adversary given a random
observation $x$ likely has little prior information about $x$, so that
protecting only against reconstructing (functions of) $x$ under some assumptions on the adversary's prior knowledge allows
substantially improved model-fitting.

Our formal setting is as follows.
For a parameter space $\Theta \subset \R^d$ and loss $\loss : \Theta
\times \mc{X} \to \R_+$, we wish to solve the risk minimization problem
\begin{equation}
  \label{eqn:pop-risk}
  \minimize_{\theta \in \Theta} ~ \risk(\theta) \defeq
  \E[\loss(\theta, X)].
\end{equation}
The standard approach~\cite{HastieTiFr09} to such problems is to construct
the empirical risk minimizer $\what{\theta}_n = \argmin_{\theta} \frac{1}{n}
\sum_{i = 1}^n \loss(\theta, X_i)$ for data $X_i \simiid P$, the
distribution defining the expectation~\eqref{eqn:pop-risk}. In this paper,
however, we consider the stochastic minimization
problem~\eqref{eqn:pop-risk} while providing both local privacy to
individual data $X_i$ and---to maintain the satisfying guarantees of
centralized differential privacy~\eqref{eqn:central-diffp}---stronger
guarantees on the global privacy of the output $\what{\theta}_n$ of our
procedure.  With this as motivation, we describe our contributions at a high
level. As above, we argue that large local privacy~\eqref{eqn:diffp}
parameters, $\diffp \gg 1$, still provide reasonable privacy protections. We
develop new mechanisms and privacy protecting schemes that more carefully
reflect the statistical aspects of problem~\eqref{eqn:pop-risk}; we
demonstrate that these mechanisms are minimax optimal for \emph{all} ranges
of $\diffp \in [0, d]$, a broader range than all prior work (where $d$ is
problem dimension).  A substantial portion of this work is devoted to
providing \emph{practical} procedures while providing meaningful local
privacy guarantees, which currently do not exist. Consequently, we provide
extensive empirical results that demonstrate the tradeoffs between private
federated (distributed) learning scenarios, showing that it is possible to
achieve performance comparable to federated learning procedures without
privacy safeguards.


\subsection{Our approach and results}

We propose and investigate a two-pronged approach to model fitting under
local privacy. Motivated by the difficulties associated with local
differential privacy we discuss in the immediately subsequent section, we
reconsider the threat models (or types of disclosure) in locally private
learning. Instead of considering an adversary with access to all data, we
consider ``curious'' onlookers, who wish to decode individuals' data but
have little prior information on them.  Formalizing this (as we discuss in
Section~\ref{sec:protections})
allows us to consider substantially relaxed values for the
privacy parameter $\diffp$, sometimes even scaling with the dimension of the
problem, while still providing protection.
While this brings us away from the standard
guarantees of differential privacy, we can still provide privacy guarantees
for the type of onlookers we consider.

This privacy model is natural in distributed model-fitting (federated
learning) scenarios~\cite{FedLearn17, ApplePrivacy17}. By providing
protections against curious onlookers, a company can protect its users
against reconstruction of their data by, for example, internal employees; by
encapsulating this more relaxed local privacy model within a broader central
differential privacy layer, we can still provide satisfactory privacy
guarantees to users, protecting them against strong external adversaries as
well. 

We make several contributions to achieve these goals. In
Section~\ref{sec:protections}, we describe a model for curious adversaries,
with appropriate privacy guarantees, and demonstrate how (for these curious
adversaries) it is still nearly impossible to accurately reconstruct
individuals' data. We then detail a prototypical private federated learning
system in Section~\ref{sec:pfl}.  In this direction, we develop new (minimax
optimal) privacy mechanisms for privatization of high-dimensional vectors in
unit balls (Section~\ref{sec:mechanisms}).  These mechanisms yield
substantial improvements over the schemes
\citet{DuchiJoWa18,DuchiJoWa13_focs} develop, which are optimal only in the
case $\diffp \le 1$, providing order of magnitude improvements over
classical noise addition schemes, and we provide a unifying theorem showing
the asymptotic behavior of a stochastic-gradient-based private learning
scheme in Section~\ref{sec:asymptotics-local}. As a consequence of this
result, we conclude that we have (minimax) optimal procedures for many
statistical learning problems for all privacy parameters $\diffp \le d$, the
dimension of the problem, rather than just $\diffp \in [0, 1]$. We conclude
our development in Section~\ref{sec:results} with several large-scale
distributed model-fitting problems, showing how the tradeoffs we make allow
for practical procedures.  Our approaches allow substantially improved
model-fitting and prediction schemes; in situations where local differential
privacy with smaller privacy parameter fails to learn a model at all, we can
achieve models with performance near non-private schemes.

\subsection{Why local privacy makes model fitting challenging}
\label{sec:local-is-hard}

To motivate our approaches, we discuss why local privacy causes
some difficulties in a classical learning problem.  \citet{DuchiRu18b} help to elucidate the precise
reasons for the difficulty of estimation under $\diffp$-local differential
privacy, and we can summarize them heuristically here, focusing on the
machine learning applications of interest.  To do so, we begin with a brief
detour through classical statistical learning and the usual convergence
guarantees that are (asymptotically) possible~\cite{VanDerVaart98}.

Consider the population risk minimization problem~\eqref{eqn:pop-risk}, and
let $\theta\opt = \argmin_{\theta} \risk(\theta)$ denote its minimizer. We
perform a heuristic Taylor expansion to understand the difference between
$\what{\theta}_n$ and $\theta\opt$.  Indeed, we have
\begin{align*}
  0 = \frac{1}{n} \sum_{i=1}^n \nabla \loss(\what{\theta}_n, X_i)
  & = \frac{1}{n} \sum_{i = 1}^n \nabla \loss(\theta\opt, X_i)
  + \frac{1}{n} \sum_{i = 1}^n (\nabla^2 \loss(\theta\opt, X_i)
  + \mbox{error}_i) (\what{\theta}_n - \theta\opt) \\
  & = \frac{1}{n} \sum_{i = 1}^n \nabla \loss(\theta\opt, X_i)
  + (\nabla^2 \risk(\theta\opt) + o_P(1)) (\what{\theta}_n - \theta\opt),
\end{align*}
(for $\mbox{error}_i$ an error term in the Taylor expansion of $\loss$),
which---when carried out rigorously---implies
\begin{equation}
  \label{eqn:influence-function-form}
  \what{\theta}_n - \theta\opt
  = \frac{1}{n} \sum_{i = 1}^n 
  \underbrace{-\nabla^2 \risk(\theta\opt)^{-1} \nabla \loss(\theta\opt; X_i)}_{
    \eqdef \psi(X_i)}
  + o_P(1 / \sqrt{n}).
\end{equation}
The \emph{influence
  function} $\psi$~\cite{VanDerVaart98} of the parameter $\theta$
measures the effect that changing a single observation $X_i$ has
on the resulting estimator $\what{\theta}_n$.

All (regular) statistically efficient estimators asymptotically have the
form~\eqref{eqn:influence-function-form}~\cite[Ch.~8.9]{VanDerVaart98},
and typically
a problem is ``easy'' when the variance of the function $\psi(X_i)$ is
small---thus, individual observations do not change the estimator
substantially.  In the case of \emph{local} differential privacy, however,
as \citet{DuchiRu18b} demonstrate, (optimal) locally private estimators
typically have the form
\begin{equation}
  \label{eqn:pita-form}
  \what{\theta}_n - \theta\opt
  = \frac{1}{n} \sum_{i = 1}^n \left[\psi(X_i) + W_i\right]
\end{equation}
where $W_i$ is a noise term that must be taken so that $\psi(x)$ and
$\psi(x')$ are indistinguishable for \emph{all} $x, x'$. Essentially, a
local differentially private procedure cannot leave $\psi(x)$ small even
when it is typically small (i.e.\ the problem is easy) because it
\emph{could} be large for some value $x$. In locally private procedures,
this means that differentially private tools for typically ``insensitive''
quantities (cf.~\cite{DworkRo14}) cannot apply, as an individual term
$\psi(X_i)$ in the sum~\eqref{eqn:pita-form} is (obviously) sensitive to
arbitrary changes in $X_i$. An alternative perspective comes via
information-theoretic ideas~\cite{DuchiRo19}: differential privacy
constraints are essentially equivalent to limiting the bits of information
it is possible to communicate about individual data $X_i$, where privacy
level $\diffp$ corresponds to a communication limit of $\diffp$ bits, so
that we expect to lose in efficiency over non-private or
non-communication-limited estimators at a rate roughly of $d / \diffp$ (see
\citet{DuchiRo19} for a formalism). The consequences of this are striking,
and extend even to weakenings of local differential privacy~\cite{DuchiRo19,
  DuchiRu18b}: adaptivity to easy problems is essentially impossible
for standard $\diffp$-locally-differentially private procedures, at least
when $\diffp$ is small, and there must be substantial dimension-dependent
penalties in the error $ \what{\theta}_n - \theta\opt$.  Thus, to enable
high-quality estimates for quantities of interest in machine learning tasks,
we explore locally differentially private settings with larger privacy
parameter $\diffp$.



\section{Privacy protections}
\label{sec:protections}

In developing any statistical machine learning system providing privacy
protections, it is important to consider the types of attacks that we wish
to protect against. In distributed model fitting and federated learning
scenarios, we consider two potential attackers: the first is a curious
onlooker who can observe all updates to a model and communication from
individual devices, and the second is from a powerful external adversary who
can observe the final (shared) model or other information about individuals
who may participate in data collection and model-fitting. For the latter
adversary, essentially the only effective protection is to use a small privacy parameter in a localized or
centralized differentially private scheme~\cite{McMahanMoRaHaAr17,
  DworkMcNiSm06, Mironov17}. For the curious onlookers, however---for
example, internal employees of a company fitting large-scale models---we
argue that protecting against reconstruction is reasonable.


\subsection{Differential privacy and its relaxations}
\label{sec:diffp-definitions}

We begin by discussing the various definitions of privacy we consider,
covering both centralized and local definitions of privacy.  We begin with
the standard centralized definitions, which extend the basic differential
privacy definition~\eqref{eqn:central-diffp}, and allow a trusted curator to
view the entire dataset. We have
\begin{definition}[Differential privacy,
    Dwork et al.~\cite{DworkMcNiSm06,DworkKeMcMiNa06}]
  \label{definition:dp}
  A randomized
  mechanism $M : \mc{X}^n \to \mc{Z}$ is \emph{$(\diffp, \delta)$
    differentially private} if for all samples $\bx, \bx' \in \mc{X}^n$
  differing in at most a single example, for all measurable sets $S \subset
  \mc{Z}$
  \begin{equation*}
    \P(M(\bx) \in S) \le e^\diffp \P(M(\bx') \in S) + \delta.
  \end{equation*}
\end{definition}

Other variants of privacy require that likelihood ratios are near one on
average, and include concentrated and R\'{e}nyi differential
privacy~\cite{DworkRo16,BunSt16,Mironov17}. Recall that the R\'{e}nyi
$\renparam$-divergence between distributions $P$ and $Q$ is
\begin{equation*}
  \drenyi{P}{Q} \defeq \frac{1}{\renparam - 1}
  \log \int \left(\frac{dP}{dQ}\right)^\renparam dQ,
\end{equation*}
where $\drenyi{P}{Q} = \dkl{P}{Q}$ when $\renparam = 1$ by taking a limit.
We abuse notation, and for random variables $X$ and $Y$ distributed
as $P$ and $Q$, respectively, write $\drenyi{X}{Y} \defeq \drenyi{P}{Q}$,
which allows us to make the following
\begin{definition}[R\'{e}nyi-differential privacy~\cite{Mironov17}]
  \label{definition:rdp}
  A randomized mechanism $M : \mc{X}^n \to \mc{Z}$ is
  \emph{$(\diffp, \renparam)$-R\'{e}nyi differentially private}
  if for all samples $\bx, \bx' \in \mc{X}^n$ differing in
  at most a single example,
  \begin{equation*}
    \drenyi{M(\bx)}{M(\bx')} \le \diffp.
  \end{equation*}
\end{definition}
\noindent
\citet{Mironov17} shows that if $M$ is
$\diffp$-differentially private, then for any
$\renparam \ge 1$, it is
$(\min\{\diffp, 2\renparam \diffp^2\}, \renparam)$-R\'{e}nyi private,
and conversely, if $M$ is $(\diffp, \renparam)$-R\'{e}nyi private,
it also satisfies
$(\diffp + \frac{\log(1 / \delta)}{\renparam - 1}, \delta)$-differential
privacy for all $\delta \in [0, 1]$.

We often use local notions of these definitions, as we consider protections
for individual data providers without trusted curation; in this case, the
mechanism $M$ applied to an individual data point $x$ is \emph{locally}
private if the Definition~\ref{definition:dp} or~\ref{definition:rdp} holds
with $M(\bx)$ and $M(\bx')$ replaced by $M(x)$ and $M(x')$, where $x, x' \in
\mc{X}$ are arbitrary.

\subsection{Reconstruction breaches}
\newcommand{\cF}{\mathcal{F}}
\newcommand{\reconloss}{L_{\rm rec}}
\newcommand{\lreconstruct}{\reconloss}
\newcommand{\ones}{\mathbf{1}}
\newcommand{\zeros}{\mathbf{0}}

Abstractly, we work in a setting in which a user or study participant has
data $X$ he or she wishes to keep private. Via some process, this
data is transformed into a vector $W$---which may simply be an identity
transformation, but $W$ may also be a gradient of the loss
$\loss$ on the datum $X$ or other derived statistic.  We then
privatize $W$ via a randomized mapping $W \to Z$. An onlooker
wishes to estimate some function $f(X)$ on the private data
$X$.  In most
scenarios with a curious onlooker, if $X$ or $f(X)$ is suitably
high-dimensional, the onlooker has limited prior information about $X$, so
that relatively little obfuscation is required in the mapping from $W \to Z$.

As a motivating example, consider an image processing scenario. A user has
an image $X$, where $W \in \R^d$ are wavelet coefficients of $X$ (in some
prespecified wavelet basis)~\cite{Mallat08}; without loss of generality, we
assume we normalize $W$ to have energy $\ltwo{W} = 1$. Let $f(X)$ be a
low-dimensional version of $X$ (say, based on the first 1/8 of wavelet
coefficients); then (at least intuitively)
taking $Z$ to be a noisy version of $W$ such that $\ltwo{Z - W} \gtrsim
1$---that is, noise on the scale of the energy $\ltwo{W}$---should be
sufficient to guarantee that the observer is unlikely to be able to
reconstruct $f(X)$ to any reasonable accuracy.  Moreover, a simple
modification of the techniques of \citet{HardtTa10} shows that for $W \sim
\uniform(\sphere^{d-1})$, any $\diffp$-differentially private quantity $Z$
for $W$ satisfies $\E[\ltwo{Z - W}] \gtrsim 1$ whenever $\diffp \le d - \log
2$. That is, we might expect that in the definition~\eqref{eqn:diffp}
of local differential privacy, even very large $\diffp$ provide
protections against reconstruction.


With this in mind, let us formalize a reconstruction breach in our
scenario. Here, the onlooker (or adversary) has a prior $\prior$ on $X \in
\mc{X}$, and there is a known (but randomized) mechanism $M : \mc{W} \to
\mc{Z}$, $W \mapsto Z = M(W)$. We then have the following definition.


\begin{definition}[Reconstruction breach]
  \label{defn:recon-breach}
  Let $\prior$ be a prior on $\mc{X}$, and let $X, W, Z$ be generated with
  Markov structure $X \to W \to Z = M(W)$ for a mechanism $M$.
  Let $f : \mc{X} \to \R^k$ be the target of reconstruction and $\reconloss
  : \R^k \times \R^k \to \R_+$ be a loss function.  Then an estimator
  $\what{v} : \mc{Z} \to \R^k$ provides an $(\alpha,
  p, f)$-\emph{reconstruction breach} for the loss $\reconloss$ if there exists
  $z$ such that
  \begin{equation}
    \P\left(\reconloss(f(X), \what{v}(z)) \le \alpha \mid M(W) = z \right)
    >  p.
    \label{eq:recon}
  \end{equation}
  If for every estimator $\what{v} : \mc{Z} \to \R^k$,
  \begin{equation*}
    \sup_{z \in \mc{Z}}
    \P\left( \reconloss(f(X), \what{v}(z)) \le \alpha \mid M(W)
    = z \right)\leq p,
  \end{equation*}
  then the mechanism $M$ is $(\alpha, p, f)$-\emph{protected against
    reconstruction} for the loss $\reconloss$.
\end{definition}
\noindent
Key to Definition~\ref{defn:recon-breach} is that it applies uniformly
across all possible observations $z$ of the mechanism $M$---there are no
rare breaches of privacy.\footnote{We ignore measurability
  issues; in our setting, all random variables are mutually absolutely
  continuous and follow regular conditional probability
  distributions, so the conditioning on $z$ in Def.~\ref{defn:recon-breach} has
  no issues~\cite{Kallenberg97}.} This requires somewhat stringent
conditions on mechanisms and also disallows relaxed privacy definitions
beyond differential privacy.

\subsection{Protecting against reconstruction}
\label{sec:reconstruction-protection-lemmas}

We can now develop guarantees against reconstruction.
The simple insight is that if an adversary has a diffuse prior
on the data of interest $f(X)$---is a priori unlikely
to be able to accurately reconstruct $f(X)$ given no information---the
adversary remains unlikely to be able to reconstruct $f(X)$ given
differentially private views of $X$ even for very large $\diffp$. Key to
this is the question of how ``diffuse'' we might expect an
adversary's prior to be.  We detail a few examples here, providing
what we call best-practices recommendations for limiting information, and
giving some strong heuristic calculations for reasonable prior information.

We begin with the simple claim that prior beliefs change little under
differential privacy, which follows immediately from Bayes' rule.
\begin{lemma}
  \label{lemma:bayesian-updating}
  Let the mechanism $M$ be $\diffp$-differentially private
  and let $V = f(X)$ for a measurable function $f$. Then for
  any $\prior \in \cP$ on $\mc{X}$ and measurable sets $A, A' \subset
  f(\mc{X})$,
  the posterior distribution $\prior_V(\cdot \mid z)$ (for
  $z = M(X)$) satisfies
  \begin{equation*}
    \frac{\prior_V(A \mid z, z)}{\prior_V(A' \mid z, z)}
    \le   e^{\diffp} \frac{\prior_V(A)}{\prior_V(A')}.
  \end{equation*}
\end{lemma}


Based on Lemma~\ref{lemma:bayesian-updating}, we can show the following
result, which guarantees that difficulty of reconstruction of a signal is
preserved under private mappings.

\begin{lemma}
  \label{lemma:preserve-difficulty}
  Assume that the prior $\prior_0$ on $X$ is such that
  for a tolerance $\alpha$, probability $p(\alpha)$, target
  function $f$, and loss $\reconloss$, we have
  \begin{equation*}
    \P_{\prior_0}(\reconloss(f(X), v_0) \le \alpha) \le p(\alpha)
    ~~ \mbox{for~all~fixed}~ v_0.
  \end{equation*}
  If $M$ is $\diffp$-differentially private,
  then it is $(\alpha, e^{\diffp} \cdot
  p(\alpha), f)$-protected against reconstruction for $\reconloss$.
\end{lemma}
\begin{proof}
  Lemma~\ref{lemma:bayesian-updating} immediately implies that for any
  estimator $\what{v}$ based on $Z = M(X)$, we have for any
  realized $z$ and $V= f(X)$
  \begin{align*}
    \P(\reconloss(V, \what{v}(z)) \le \alpha \mid Z = z)
    & = \int \1{\reconloss(v, \what{v}(z)) \le \alpha}
    d\prior_V(v \mid z) \\
    & \le e^{\diffp}
    \int \1{\reconloss(v, \what{v}(z)) \le \alpha}
    d\prior_V(v).
  \end{align*}
  The final quantity is $e^{\diffp}
  \P(\reconloss(f(X), v_0) \le \alpha)
  \le e^{\diffp} p(\alpha)$ for $v_0 = \what{v}(z)$, as desired.
\end{proof}

Let us make these ideas a bit more concrete through two examples:
one when it is reasonable to assume a diffuse prior, one
for much more peaked priors.

\subsubsection{Diffuse priors and reconstruction of linear functions}
\label{sec:diffuse-linear-reconstruction}

For a prior $\prior$ on $X$ and
$f : \mc{X} \to C \subset \R^k$, where $C$ is a compact subset of $\R^k$,
let $\prior_f$ be the push-forward (induced prior) on $f(X)$ and let
$\prior_0$ be some base measure on $C$ (typically, this will be a uniform
measure). Then for $\rho_0 \in [0, \infty]$ define the set of
\emph{plausible priors}
\begin{equation}
  \label{eq:priorPA}
  \cP_f(\rho_0) \defeq \left\{\prior ~ \mbox{on}~ X ~ \mbox{s.t.}~
  \log \frac{d \prior_f(v)}{d \prior_{0}(v)}
  \leq \rho_0 , ~ \mbox{for~} v \in C\right\}.
\end{equation}
For example, consider an image processing situation, where we
wish to protect against reconstruction even of low-frequency information, as
this captures the basic content of an image.  In this case, we consider an
orthonormal matrix $A \in \C^{k \times d}$, $AA^* = I_k$, and
an adversary wishing to reconstruct the normalized projection
\begin{equation}
  f_A(x) =  \frac{Ax}{\ltwo{Ax}} = \frac{A u}{\ltwo{A u}}
  ~~ \mbox{for~} u = x / \ltwo{x}.
  \label{eq:reconstruct}
\end{equation}
For example, $A$ may be the first $k$ rows of the Fourier transform matrix,
or the first level of a wavelet hierarchy, so the adversary
seeks low-frequency information about $x$. In either case, the
``low-frequency'' $Ax$ is enough to give a sense of the private data, and
protecting against reconstruction is more challenging for small $k$.


In the particular orthogonal reconstruction case, we take the initial prior
$\prior_0$ to be uniform on $\sphere^{k-1}$---a reasonable choice when
considering low frequency information as above---and consider $\ell_2$
reconstruction with $\reconloss(u, v) = \ltwo{u/\ltwo{u} - v / \ltwo{v}}$
(when $v \neq 0$, otherwise setting $\reconloss(u, v) = \sqrt{2}$).  For $V$
uniform and $v_0 \in \sphere^{k-1}$, we have $\E[\ltwo{V - v_0}^2] = 2$, so
that thresholds of the form $\alpha = \sqrt{2 - 2a}$ with $a$ small are the
most natural to consider in the reconstruction~\eqref{eq:recon}.  We have
the following proposition on reconstruction after a locally differentially
private release.
\begin{proposition}
  \label{proposition:breach}
  Let $M$ be $\diffp$-locally differentially private~\eqref{eqn:diffp} and
  $k \geq 4$.  Let $f = f_A$ as in Eq.~\eqref{eq:reconstruct} and $\prior
  \in \cP_f(\rho_0)$ as in Eq.~\eqref{eq:priorPA}.  Then for $a \in [0, 1]$,
  $M$ is $(\sqrt{2 - 2a}, p(a), f_A)$-protected against reconstruction for
  \begin{equation*}
    p(a) =  \begin{cases}
      \exp\left( \diffp + \rho_0
      + \frac{k}{2} \cdot \log(1 - a^2) \right)
    &  \mbox{if~} a \in [0, 1/\sqrt{2}] \\
    \exp\left(\diffp + \rho_0 + \frac{k-1}{2} \cdot \log(1 - a^2) - \log(2 a \sqrt{k}) \right)
    & \mbox{if~} a \in [\sqrt{2/k}, 1].
      \end{cases}
  \end{equation*}
\end{proposition}
\noindent
Simplifying this slightly and rewriting, assuming the reconstruction
$\what{v}$ takes values in $\sphere^{k-1}$, we have
\begin{equation*}
  \P(\ltwo{f(X) - \what{v}(Z)} \le \sqrt{2 - 2a} \mid Z = z)
  \le \exp\left(-\frac{k a^2}{2}\right) \exp(\diffp + \rho_0)
\end{equation*}
for $f(x) = Ax / \norm{Ax}$ and $a \le 1 / \sqrt{2}$.
That is, unless $\diffp$ or $\rho_0$ are of the order of $k$, the
probability of obtaining reconstructions better than (nearly) random
guessing is extremely low.

\begin{proof}
  Let $Y \sim \uniform\left(\sphere^{k-1}\right)$ and $v_0 \in
  \sphere^{k-1}$.  We then have
  \begin{equation*}
    \ltwo{ Y - v_0}^2 = 2 \cdot \left(1 - \<Y,v_0 \> \right).
  \end{equation*}
  We collect a number of standard facts on the uniform
  distribution on $\sphere^{k-1}$ in
  Appendix~\ref{sec:preliminaries-uniform-concentration}, which we reference
  frequently. Lemma~\ref{lemma:uniform-bounds}
  implies that for all $v_0 \in
  \sphere^{k-1}$ and $a \in [0, 1/\sqrt{2}]$ that
  \begin{equation*}
    \P_{\prior_{\rm uni}}(\ltwo{Y - v_0} < \sqrt{2 - 2 a})  = \P \left( \<Y, v_0\> > a \right)
    \le (1 - a^2)^{k/2}
  \end{equation*}
  Because $V = f_A(X)$ has prior $\prior_V$ such that $d\prior_V /
  d\prior_{\rm uni} \le e^{\rho_0}$,
  we obtain
  \begin{equation*}
    \P_{\prior_V}\left(  \ltwo{V - v_0} < \sqrt{2 - 2 a} \right) \leq
    e^{\rho_0} \cdot \left( 1 - a^2  \right)^{k/2}.
  \end{equation*}
  Then Lemma~\ref{lemma:preserve-difficulty} gives the first result
  of the proposition.

  When the desired accuracy is higher (i.e.\ $a \in [\sqrt{2/k},1]$),
  Lemma~\ref{lemma:uniform-bounds}
  with our assumed ratio bound between $\prior_V$ and $\prior_{\rm uni}$
  implies
  \begin{equation*}
    \P_{\prior_V}(\ltwo{V - v_0} \le \sqrt{2 - 2a})
    \le e^{\rho_0} \P_{\prior_{\rm uni}}
    (\ltwo{Y - v_0} < \sqrt{2 - 2 a})
    \le e^{\rho_0} \frac{(1 - a^2)^{\frac{k-1}{2}}}{2 a \sqrt{k}}.
  \end{equation*}
  Applying Lemma~\ref{lemma:preserve-difficulty} completes the proof.
\end{proof}

\subsubsection{Reconstruction protections against sparse data}

When it is unreasonable to assume that an individual's data is
near uniform on a $d$-dimensional space, additional strategies are
necessary to limit prior information.  We now view an individual
data provider as having multiple ``items'' that an adversary wishes to
investigate. For example, in the setting of fitting a word model on mobile
devices---to predict next words in text messages to use as suggestions when
typing, for example---the items might consist of all pairs and triples of
words the individual has typed. In this context, we combine two approaches:
\begin{enumerate}[(i)]
\item \label{item:large-minibatch} An individual contributes data only if
  he/she has sufficiently many data points locally (for example, in our word
  prediction example, has sent sufficiently many messages)
\item \label{item:stratify}
  An individual's data must be diverse or sufficiently stratified
  (in the word prediction example, the individual sends many
  distinct messages).
\end{enumerate}
As Lemma~\ref{lemma:preserve-difficulty} makes clear, if $M$ is
$\diffp$-differentially private and for \emph{any} fixed $v \in \mc{Z}$ we
have $\P(\lreconstruct(v, f(X)) \le a) \le p_0$, then for all functions
$\what{f}$,
\begin{equation}
  \label{eqn:posterior-reconstruction}
  \P(\lreconstruct(\what{f}(M(X)), f(X))
  \le a) \le e^\diffp p_0.
\end{equation}

\newcommand{\precision}{\mathsf{precision}}
\newcommand{\recall}{\mathsf{recall}}

We consider an example of sampling a histogram---specifically thinking of
sampling messages or related discrete data.  We call the elements
\emph{words} in a dictionary of size $d$, indexed by $j = 1, \ldots, d$. To
stratify the data (approach~\eqref{item:stratify}), we treat a user's data
as a vector $x \in \mc{X} \in \{0, 1\}^d$, where $x_j = 0$ if the user has
not used word $j$, otherwise $x_j = 1$. We do not allow a user to
participate until $x^T \ones \ge m$ for a particular ``mini batch'' size $m$
(approach~\eqref{item:large-minibatch}). Now, let us discuss the prior
probability of reconstructing a user's vector $x$.  We consider
reconstruction via precision and recall. Let $v \in \{0, 1\}^d$ denote a
vector of predictions of the used words, where $v_j = 1$ if we predict word
$j$ is used. Then we define
\begin{equation*}
  \precision(x, v) \defeq
  \frac{v^T x}{v^T \ones}
  ~~ \mbox{and} ~~
  \recall(x, v) \defeq
  \frac{v^T x}{\ones^T x},
\end{equation*}
that is, precision measures the fraction of predicted words that are correct,
and recall the fraction of used words the adversary predicts correctly.
We say that the signal $x$ has been reconstructed for some $p, r$ if
$\precision(x, v) \ge p$ and $\recall(x, v) \ge r$. Let us
bound the probability of each of these events under appropriate
priors on the vector $X \in \{0, 1\}^d$.

Using Zipfian models of text and discretely sampled
data~\cite{ClausetShNe09},
a reasonable a priori model of the sequence $X$, when we assume that a user
must draw at least $m$ elements, is that independently
\begin{equation}
  \label{eqn:zipf-model}
  \P(X_j = 1) = \min\left\{\frac{m}{j}, 1\right\}.
\end{equation}
With this model for a prior, we may bound the probability
of achieving good precision or recall:
\begin{lemma}
  \label{lemma:precision-recall-bounds}
  Let $\gamma \ge 2$ and assume the vector $X$ satisfies
  the Zipfian model~\eqref{eqn:zipf-model}.
  Assume that the vector $v \in \{0, 1\}^d$ satisfies
  $v^T \ones \ge \gamma m$. Then
  \begin{equation*}
    \P(\precision(v, X) \ge p)
    \le \exp\left(-\min\left\{\frac{\hinge{p \gamma - 1
        - \log \gamma}^2 m}{2 \log \gamma},
    \frac{3}{4} \hinge{p \gamma - 1 - \log \gamma} m
    \right\} \right).
  \end{equation*}
  Conversely, assume that the vector $v \in \{0, 1\}^d$ satisfies
  $v^T \ones \ge \gamma m$,
  and define
  $\tau(r, d, m, \gamma) \defeq
  r(1 + \log \frac{d}{m+1}) - 1 - \log \gamma$.
  Then
  \begin{equation*}
    \P(\recall(v, X) \ge r)
    \le \exp\left(-\min\left\{
    \frac{\hinge{\tau(r,d,m,\gamma)}^2 m}{
      4 (1 - r^2) \log \frac{d}{m}},
    \frac{3}{4}
    \hinge{\tau(r, d, m, \gamma)} m
    \right\}\right).
  \end{equation*}
\end{lemma}
\noindent
See Appendix~\ref{sec:proof-precision-recall-bounds} for a proof.

Considering Lemma~\ref{lemma:precision-recall-bounds}, we can make a few
simplifications to see the (rough) beginning reconstruction guarantees we
consider---with explicit calculations on a per-application basis.  In
particular, we see that for any fixed precision value $p$ and recall value
$r$, we may take $\gamma = \frac{2}{p} \log\frac{1}{p}$ to obtain that as
long as $r \log \frac{d}{m} \ge 2\log \frac{1}{p}$, then
\begin{equation*}
  \P\left(\precision(v, X) \ge p
  ~ \mbox{and} ~ \recall(v, X) \ge r\right)
  \le \max\left\{\exp\left(-c m \log\frac{1}{p}\right),
  \exp\left(-c m \log\frac{d}{m}\right)\right\}
  =: p_{r,d,m}.
\end{equation*}
for a numerical constant $c > 0$.
Thus, we have the following protection guarantee.
\begin{proposition}
  Let the conditions of Lemma~\ref{lemma:precision-recall-bounds} hold.
  Define the reconstruction loss
  $\lreconstruct(v, X)$ to be $1$ if
  $\precision(v, X) \ge p$ and $\recall(v, x) \ge r$, 0 otherwise,
  where $r \log \frac{d}{m} \ge 2 \log \frac{1}{p}$. Then
  if $M$ is $\diffp$-locally differentially private,
  $M$ is $(0, e^\diffp  p_{r,d,m})$-protected against reconstruction.
\end{proposition}

Consequently, we make the following recommendation: in the case that signals
are dictionary-like, a best practice is to aggregate together at least $m =
d^\rho$ such signals, for some power $\rho$, and use (local) privacy budget
$\diffp$ in differentially private mechanisms of at most
$\diffp = c m$, where $c$ is a small (near 0) constant.  We revisit
this in the language modeling applications in the experiments.

\section{Applications in federated learning}
\label{sec:pfl}

Our overall goal is to implement federated learning, where distributed units
send private updates to a shared model to a centralized location. Recalling
our population risk~\eqref{eqn:pop-risk}, basic distributed learning
procedures (without privacy) iterate as follows~\cite{BertsekasTs89,
  DeanCoMoChDeMaRaSeTuYaNg12, BoydPaChPeEc11}:
\begin{enumerate}[1.]
\item \label{item:central-distribution}
  A centralized parameter $\theta$ is distributed among
  a batch of $b$ workers, each with a local sample $X_i$, $i = 1, \ldots, b$.
\item \label{item:local-updates}
  Each worker computes an update $\Delta_i \defeq \theta_i - \theta$
  to the model parameters.
\item \label{item:central-aggregate}
  The centralized procedure aggregates $\{\Delta_i\}_{i=1}^b$ into
  a global update $\Delta$
  and updates $\theta \leftarrow \theta + \Delta$.
\end{enumerate}
\noindent
In the prototypical stochastic gradient method,
$\Delta_i = -\stepsize \nabla \loss(\theta, X_i)$ for some stepsize
$\stepsize > 0$ in step~\ref{item:local-updates}, and $\Delta = \frac{1}{b}
\sum_{i = 1}^b \Delta_i$ is the average of the stochastic gradients
at each sample $X_i$ in step~\ref{item:central-aggregate}.

In our private distributed learning context, we elaborate
steps~\ref{item:local-updates} and~\ref{item:central-aggregate} so that each
provides privacy protections: in the local update
step~\ref{item:local-updates}, we use locally private mechanisms to protect
individuals' private data $X_i$---satisfying
Definition~\ref{defn:recon-breach} on protection against reconstruction
breaches. Then in the central aggregation step~\ref{item:central-aggregate},
we apply centralized differential privacy mechanisms to guarantee that any
model $\theta$ communicated to users in the
broadcast~\ref{item:central-distribution} is globally private.  The overall
feedback loop provides meaningful privacy guarantees, as a user's data
is never transmitted clearly to the centralized server, and strong
centralized privacy guarantees mean that the final and intermediate
parameters $\theta$ provide no sensitive disclosures.

\subsection{A private distributed learning system}

Let us more explicitly describe the implementation of a distributed learning
system. The outline of our system is similar to the development of
\citet[Sec.~5.2]{DuchiJoWa13_focs, DuchiJoWa18} and the system that
\citet{DPFedLearn17} outline; we differ in that we allow more general
updates and privatize individual users' data before communication, as the
centralized data aggregator may not be completely trusted.

The stochastic optimization proceeds as follows.  The central aggregator
maintains the global model parameter $\theta \in \R^d$, and in each
iteration, chooses a random subset (mini-batch) $\cB$ of expected size $qN$,
where $q \in (0, 1)$ is the subsampling rate and $N$ the total population
size available. Each individual $i \in \cB$ sampled then computes a local
update, which we describe abstractly by a method $\Update$ that takes as
input the local sample $\bbx_i = \{x_{i,1}, \ldots, x_{i,m}\}$ and central
parameter $\theta$, then
\begin{equation*}
  \theta_{i} \gets \Update\left( \bbx_{i}, \theta \right).
\end{equation*}
Many updates are possible.
Perhaps the most popular rule is to apply a gradient update,
where from an initial model $\theta_0$ and for stepsize $\stepsize$ we apply
\begin{equation*}
  \Update(\bbx_i, \theta_0)
  \defeq \theta_0 - \stepsize
  \frac{1}{m} \sum_{j=1}^{m}\nabla\loss(\theta_0, x_{i,j})
  = \argmin_{\theta} \left\{
  \frac{1}{m} \sum_{j = 1}^m
  \<\nabla \loss(\theta_0, x_{i,j}), \theta - \theta_0\>
  + \frac{1}{2 \stepsize} \ltwo{\theta - \theta_0}^2
  \right\}.
\end{equation*}
An alternative is to stochastic proximal-point-type updates~\cite{AsiDu19siopt,
  KulisBa10, KarampatziakisLa11, Bertsekas11, DavisDr18a}, which update
\begin{equation}
  \Update(\bbx_i, \theta_0) \defeq
  \argmin_\theta \left\{ \frac{1}{m} \sum_{j = 1}^m \loss(\theta, x_{i,j})
  + \frac{1}{2 \stepsize_i} \ltwo{\theta - \theta_0}^2 \right\},
  \label{eqn:prox-point-update}
\end{equation}
or their relaxations to use approximate models~\cite{AsiDu19siopt,AsiDu19}.

\newcommand{\cliprad}{\rho}
\newcommand{\projclip}{\mathsf{proj}_\cliprad}
\newcommand{\cumulativediffp}{\diffp_{\renparam,\textup{tot}}}

After computing the local update $\theta_i$, we privatize the scaled local
difference $\Delta_i \defeq \frac{1}{\stepsize} (\theta_i - \theta_0)$,
which is the (stochastic) \emph{gradient mapping} for typical model-based
updates~\cite{AsiDu19siopt,AsiDu19}, as this scaling by stepsize enforces a
consistent expected update magnitude. We let
\begin{equation}
  \label{eqn:privatize-local-update}
  \what{\Delta}_i
  = M(\Delta_i)
\end{equation}
where $M$ is a private mechanism, be an unbiased (private) view of
$\Delta_i$, detailing mechanisms $M$ in Section~\ref{sec:mechanisms}.
Given the privatized local updates $\{\what{\Delta}_i\}_{i \in \cB}$,
we project the update of each onto an $\ell_2$-ball of fixed
radius $\cliprad$, so that for $\projclip(v) \defeq
\min\{\cliprad / \ltwo{v}, 1\} \cdot v$ we consider the averaged gradient
mapping
\begin{equation}
  \what{\Delta} \gets \frac{\eta}{qN}
  \left(\sum_{i \in \cB} \projclip(\what{\Delta}_i) + Z\right)
  ~~~
  \mbox{and}~~~
  \theta^{(t+1)} = \theta^{(t)} + \what{\Delta}.
  \label{eq:server-agg}
\end{equation}
The projection operation $\projclip$ limits the contribution of any
individual update, while the vector $Z \sim \normal(0, \sigma^2)$ is
Gaussian and provides a centralized privacy guarantee, where we describe
$\sigma^2$ presently.  In the case that the loss functions $\loss$ are
Lipschitz---typically the case in statistical learning scenarios with
classification, for example, logistic regression---the projection is
unnecessary as long as the data vectors $x_i$ lie in a compact space.

It remains to discuss the global privacy guarantees we provide via the noise
addition $Z$. For any individual $i$, we have
$\frac{1}{qN}\ltwos{\projclip(\what{\Delta}_i)} \le \cliprad / (qN)$; thus
we may use Abadi et al.'s ``moments accountant''
analysis~\cite{Abadietal16}, which reposes on R\'{e}nyi-differential
privacy (Def.~\ref{definition:rdp}).  We first
present an intuitive explanation; the precise parameter settings we explain
in the experiments, which make the privacy guarantees as sharp as
possible using computational evaluations of the
privacy parameters~\cite{Abadietal16}.
First, if $\channel$ denotes the
distribution~\eqref{eq:server-agg} of $\what{\Delta}$ and $\channel_0$
denotes its distribution when we remove a fixed user $i_0$, then the
R\'{e}nyi-2-divergence between the two~\cite[Lemma~3]{Abadietal16} satisfies
\begin{equation*}
  \drentwo{\channel}{\channel_0}
  \le \log\left[1 + \frac{q^2}{1 - q}
  \left(\exp\Big(\frac{\ltwos{\what{\Delta}_{i_0}}^2}{\sigma^2} - 1\Big)
  \right)\right]
  \le \frac{q^2}{1 - q} \left(e^{\cliprad^2 / \sigma^2} - 1\right),
\end{equation*}
and the R\'{e}nyi-$\renparam$-divergence is
$\dren{\channel}{\channel_0} \le \frac{\renparam (\renparam - 1) q^2}{1 - q}
(e^{\cliprad^2 / \sigma^2} - 1) + O(q^3 \renparam^3 / \sigma^3)$ for
$\renparam \le \sigma^2 \log \frac{1}{q \sigma}$.
Thus, letting $\cumulativediffp$ denote the cumulative R\'{e}nyi-$\renparam$
privacy
loss after $T$ iterations of the update~\eqref{eq:server-agg}, we have
\begin{equation*}
  \cumulativediffp
  \le T \frac{q^2}{1 - q} \left[\exp\left(\frac{\cliprad^2}{\sigma^2}
    \right) - 1 \right]
  + O\left(\frac{q^3 \renparam^3}{\sigma^3}\right).
\end{equation*}
This remains below a fixed $\diffp$ 
for
\begin{equation*}
  \sigma^2 \ge \frac{\cliprad^2}{\log
    (1 + \frac{\diffp(1 - q)}{T q^2})} (1 + o(1))
  \approx \frac{T q^2 \cliprad^2}{\diffp},
\end{equation*}
where $o(1) \to 0$ as $q \renparam \to 0$,
and thus for any choice of $q = m / N$---using batches of size $m$---as long
as we have roughly $\sigma^2 \ge \frac{T m^2 \cliprad^2}{N^2 \diffp_\renparam}$,
we guarantee centralized $(\diffp_\renparam, \renparam)$-R\'{e}nyi-privacy.

\subsection{Asymptotic Analysis}
\label{sec:asympt}

To provide a fuller story and demonstrate the general consequences of our
development, we now turn to an asymptotic analysis of our distributed
statistical learning scheme for solving problem~\eqref{eqn:pop-risk} under
locally private updates~\eqref{eqn:privatize-local-update}.  We ignore the
central privatization term, that is, addition of $Z$ in
update~\eqref{eq:server-agg}, as it is asymptotically negligible in our
setting. To set the stage, we consider minimizing the population risk
$\risk(\theta) \defeq \E[\loss(\theta, X)]$ using an i.i.d.\ sample $X_1,
\ldots, X_N \simiid P$ for some population $P$.

The simplest strategy in this case is to use the stochastic gradient method,
which (using a stepsize sequence $\stepsize_t$) performs updates
\begin{equation*}
  \theta^{(t + 1)} \gets \theta^{(t)} - \stepsize_t g_t
\end{equation*}
where for $X_t \simiid P$ and defining the $\sigma$-field
$\mc{F}_t \defeq \sigma(X_1, \ldots, X_t)$ we have $\theta^{(t)}
\in \mc{F}_{t-1}$ and
\begin{equation*}
  \E[g_t \mid \mc{F}_{t-1}]
  = \nabla \loss(\theta^{(t)}; X_t).
\end{equation*}
In this case, under the assumptions that $\risk$ is
$\mc{C}^2$ in a neighborhood of $\theta\opt = \argmin_\theta
\risk(\theta)$ with $\nabla^2 \risk(\theta\opt) \succ 0$ and that
for some $C < \infty$, we have
\begin{equation*}
  \E[\norm{g_t}^2 \mid \mc{F}_{t-1}]
  \le C \left(1 + \norms{\theta^{(t)} - \theta\opt}^2\right)
  ~~ \mbox{and} ~~
  \E[g_t g_t^\top \mid \mc{F}_{t-1}] \cas \Sigma
\end{equation*}
\citet{PolyakJu92} provide the following result.
\begin{corollary}[Theorem 2~\cite{PolyakJu92}]
  \label{corollary:polyak-juditsky}
  Let $\bar{\theta}^{(T)} = \frac{1}{T}\sum_{t = 1}^T \theta^{(t)}$. Assume
  the stepsizes $\stepsize_t \propto t^{-\beta}$ for some $\beta \in (1/2, 1)$.
  Then under the preceding conditions,
  \begin{equation*}
    \sqrt{T} \left(\bar{\theta}^{(T)}  - \theta\opt \right)
    \cd \Normal{0}{ \nabla^2L(\theta\opt)^{-1}
      \Sigma  \nabla^2 L(\theta\opt)^{-1}}.
  \end{equation*}
\end{corollary}

We now consider the impact that local privacy has on this result.  Let $M$
be a local privatizing mechanism~\eqref{eqn:privatize-local-update},
and define $Z(\theta; x) = M( \nabla\ell( \theta; x))$.
We assume that each application of the mechanism $M$ is (conditional
on the pair $(\theta, x)$) independent of all others.
In this case, the stochastic gradient update becomes
\begin{equation*}
  \theta^{(t+1)} \gets \theta^{(t)} - \stepsize_t  \cdot Z(\theta^{(t)}; X_t).
\end{equation*}
In all of our privatization schemes to come, we have continuity of the
privatization $Z$ in $\theta$ so that $\lim_{\theta \to \theta\opt}
\E[Z(\theta; X) Z(\theta; X)^\top] = \E[Z(\theta\opt; X) Z(\theta\opt;
  X)^\top]$. Additionally, we have the unbiasedness---as we show---that
conditional on $\theta$ and $x$, $\E[Z(\theta; x)] = \nabla \loss(\theta;
x)$.  When we make the additional assumption that the gradients of the loss
are bounded---which holds, for example, for logistic regression as long as
the data vectors are bounded---we have the following a consequence of
Corollary~\ref{corollary:polyak-juditsky}.
\begin{corollary}
  \label{corollary:priv-asympt}
  Let the conditions of
  Corollary~\ref{corollary:polyak-juditsky} and the preceding
  paragraph hold.
  Assume that $\sup_{\theta \in
    \Theta} \sup_{x \in \cX} \ltwo{\nabla \ell(\theta; x)} \leq r_{\max} <
  \infty$. Let
  $\Sigma^{\textup{priv}} \defeq \E[Z(\theta\opt; X) Z(\theta\opt; X)^\top]$.
  Then
  \begin{equation*}
    \sqrt{T} \left(\bar{\theta}^{(T)}  - \theta\opt \right)
    \cd \Normal{0}{ \nabla^2L(\theta\opt)^{-1}
      \Sigma^{\rm priv}  \nabla^2 L(\theta\opt)^{-1}}.
  \end{equation*}
\end{corollary}

Key to Corollary~\ref{corollary:priv-asympt} is that---as we describe in the
next section---we can design mechanisms for which
\begin{equation*}
  \Sigma^{\textup{priv}}
  \preceq \Sigma\subopt
  + C \left[\frac{d}{\diffp \wedge \diffp^2}
    \Sigma\subopt + \frac{\tr(\Sigma\subopt)}{(\diffp \wedge \diffp^2)}
    I\right]
\end{equation*}
for a numerical constant $C < \infty$, where $\Sigma\subopt = \cov(\nabla
\loss(\theta\opt; X))$.  This is (in a sense) the ``correct'' scaling of the
problem with dimension and local privacy level $\diffp$, is
minimax optimal~\cite{DuchiRo19}, and is in contrast
to previous work in local privacy~\cite{DuchiJoWa18}. Describing this more
precisely requires description of our privacy mechanisms and alternatives,
to which we now turn.



\section{Separated Private Mechanisms for High Dimensional Vectors}
\label{sec:mechanisms}

The main application of the privacy mechanisms we develop is to minimax
rate-optimal private
(distributed) statistical learning scenarios; accordingly,
we now carefully consider mechanisms to use in the private
updates~\eqref{eqn:privatize-local-update}.
Motivated by the difficulties we outline in Section~\ref{sec:local-is-hard}
for locally private model fitting---in particular, that estimating the
\emph{magnitude} of a gradient or influence function is challenging,
and the scale of an update is essentially important---we
consider mechanisms that transmit information $W \in \R^d$ by privatizing a
pair $(U, R)$, where $U = W / \ltwo{W}$ is the direction and $R = \ltwo{W}$
the magnitude, letting $Z_1 = M_1(U)$ and $Z_2 = M_2(R)$ be their privatized
versions (see Fig.~\ref{fig:graphical-structure}). We consider mechanisms
satisfying the following definition.
\begin{figure}[h!]
  \centering
  \begin{overpic}[width=.35\columnwidth]{
      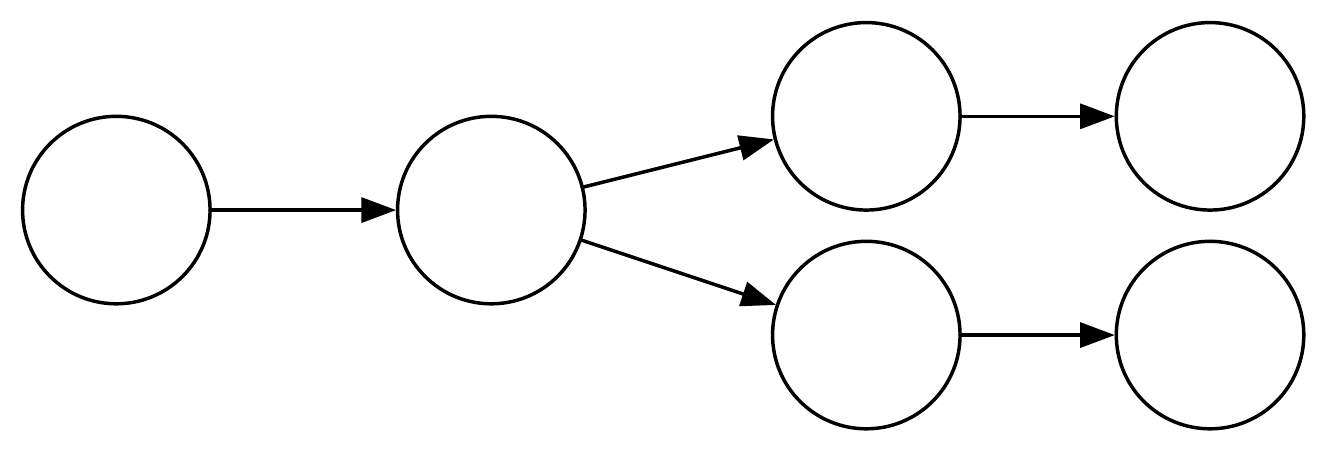}
    \put(6,16){$X$}
    \put(34,16){$W$}
    \put(63,6){$U$}
    \put(63,23){$R$}
    \put(88,6){$Z_1$}
    \put(88,23){$Z_2$}
  \end{overpic}
  \caption{Markovian graphical structure between data $X$ and privatized pair $(Z_1, Z_2)$ \label{fig:graphical-structure}}
\end{figure}
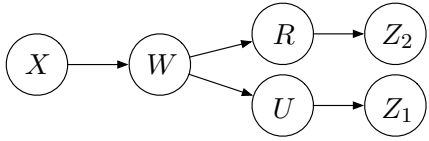

\begin{definition}[Separated Differential Privacy]
  \label{definition:separated-privacy}
  A pair of mechanisms $M_1, M_2$ mapping from $\mc{U} \times \mc{R}$ to
  $\mc{Z}_1 \times \mc{Z}_2$ (i.e.\ a channel with the Markovian structure
  of Fig.~\ref{fig:graphical-structure}) is $(\diffp_1,
  \diffp_2)$-\emph{separated differentially private} if $M_i$ is
  $\diffp_i$-locally differentially private~\eqref{eqn:diffp} for $i =
  1, 2$.
\end{definition}
\noindent
The basic composition properties of differentially private
channels~\cite{DworkRo14} guarantee that $M = (M_1, M_2)$ is $(\diffp_1 +
\diffp_2)$-locally differentially private, so such mechanisms enjoy the
benefits of differentially private algorithms---group privacy, closure under
post-processing, and composition protections~\citep{DworkRo14}---and they
satisfy the reconstruction guarantees we detail in
Section~\ref{sec:reconstruction-protection-lemmas}.

The key to efficiency in all of our applications is to have accurate
estimates of the actual update $\Delta \in \R^d$---frequently simply the
stochastic gradient $\nabla \loss(\theta; x)$. We consider two regimes of
the most interest: Euclidean settings~\cite{PolyakJu92, NemirovskiJuLaSh09}
(where we wish to privatize vectors belonging to $\ell_2$ balls) and the
common non-Euclidean scenarios arising in high-dimensional estimation and
optimization (mirror descent~\cite{NemirovskiJuLaSh09, BeckTe03}), where we
wish to privatize vectors belonging to $\ell_\infty$ balls.
We thus describe mechanisms for releasing unit vectors,
after which we show how to release scalar
magnitudes; the combination allows us to release (optimally accurate)
unbiased vector estimates, which we can employ in distributed and online
statistical learning problems. We conclude the section by revisiting the
asymptotic normality results of Corollary~\ref{corollary:priv-asympt},
which unifies our entire
development, providing a minimax optimal convergence guarantee---for all
privacy regimes $\diffp$---better than those available for previous locally
differentially private learning procedures.



\subsection{Privatizing unit $\ell_2$ vectors with high accuracy}
\label{sec:private-l2}

\begin{figure}
  \begin{center}
    \begin{tabular}{cc}
      \begin{overpic}[width=.5\columnwidth]{
          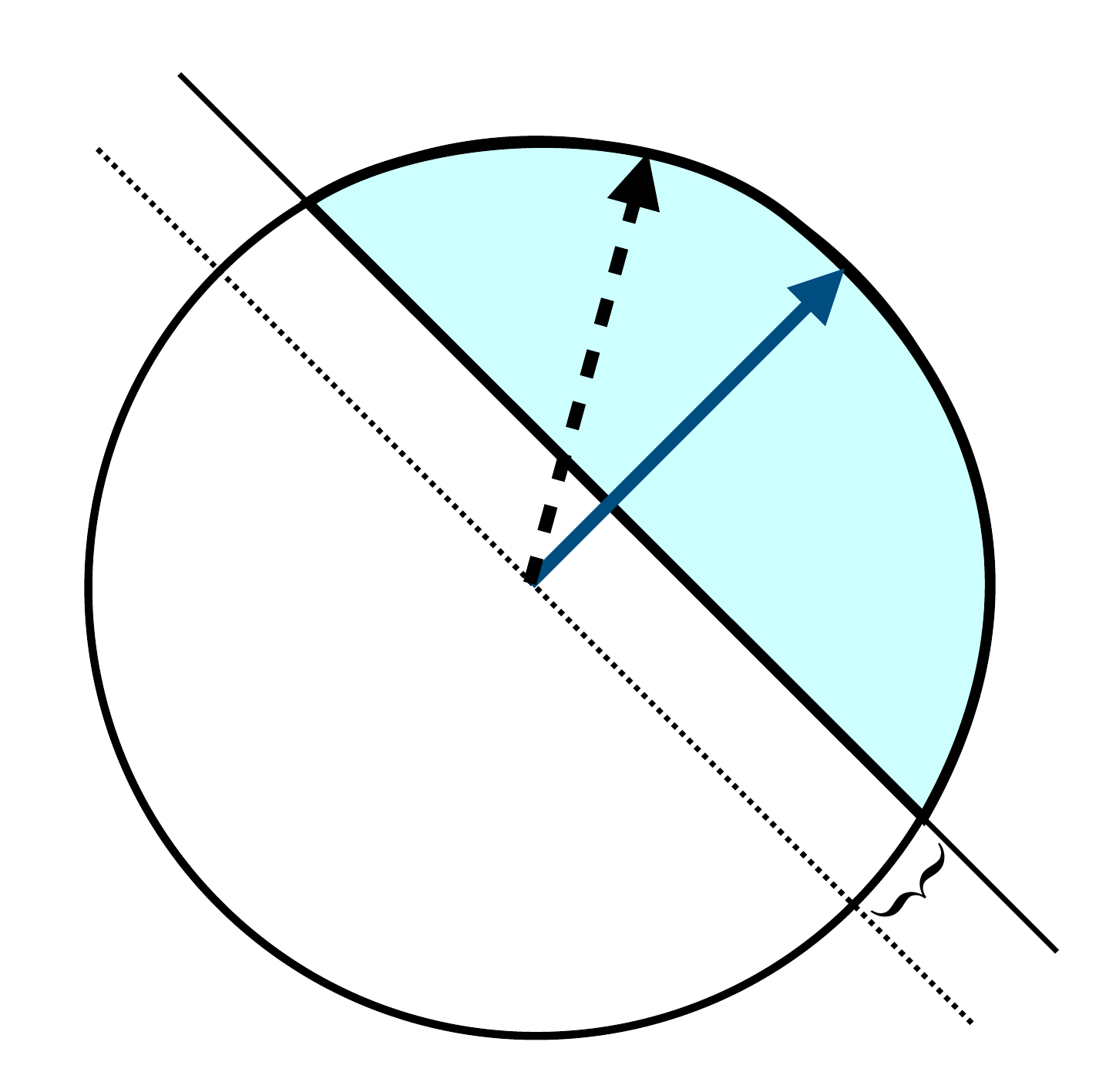}
        \put(76,75){\LARGE $\textcolor{darkblue}{u}$}
        \put(57,86){\LARGE $v$}
        \put(65,50){\large w.p.\ $p$}
        \put(20,30){\large w.p.\ $1 - p$}
        \put(83.5,15){\large $\gamma$}
      \end{overpic}
      &
      \begin{minipage}{.45\columnwidth}
        \vspace{-7cm}
        \caption{\label{fig:unit-vector-sampling} Private sampling scheme
          for the $\ell_2$ ball, $\PrivUnit$ in
          Alg.~\ref{alg:unit_mech}. Input unit vector $u$ is resampled,
          chosen uniformly from the shaded spherical cap---parameterized by
          the distance $\gamma \in [0, 1]$ from the equator---with
          probability $p$ and uniformly from the complement with probability
          $1 - p$.}
      \end{minipage}
    \end{tabular}
  \end{center}
\end{figure}
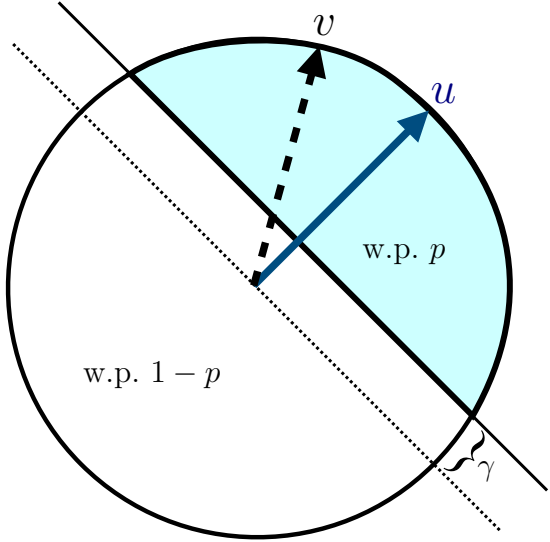

We begin with the Euclidean case, which arises in most classical
applications of stochastic gradient-like
methods~\cite{Zinkevich03,PolyakJu92, NemirovskiJuLaSh09}. In this case, we
have a vector $u \in \sphere^{d-1}$ (i.e.\ $\ltwo{u} = 1$), and we wish to
generate an $\diffp$-differentially private vector $Z$ with the property
that
\begin{equation*}
  \E[Z \mid u] = u
  ~~~ \mbox{for~all~} u \in \sphere^{d-1},
\end{equation*}
where the size $\ltwo{Z}$ is as small as possible to maximize the efficiency
in Corollary~\ref{corollary:priv-asympt}.

We modify the sampling strategy of \citet{DuchiJoWa18} to develop an optimal
mechanism. Given a vector $v \in \sphere^{d-1}$, we draw a vector $V$ from a
spherical cap $\{v \in \sphere^{d-1} \mid \<v, u\> \ge \gamma\}$ with some
probability $p \ge \half$ or from its complement $\{v \in \sphere^{d-1} \mid
\<v,u\> < \gamma\}$ with probability $1 - p$, where $\gamma \in [0, 1]$ and
$p$ are constants we shift to trade accuracy and privacy more precisely. In
Figure~\ref{fig:unit-vector-sampling}, we give a visual representation of this
mechanism, which we term $\PrivUnit$ (see Algorithm~\ref{alg:unit_mech}); in
the next subsection we demonstrate the choices of $\gamma$ and scaling
factors to make the scheme differentially private and unbiased. Given its
inputs $u, \gamma$, and $p$, Algorithm~\ref{alg:unit_mech} returns $Z$
satisfying $\E[Z \mid u] = u$. We set the quantity $m$ in
Eq.~\eqref{eqn:norm-of-W} to guarantee this normalization, where
\begin{equation*}
  B(x;\alpha,\beta) \defeq \int_{0}^x
  t^{\alpha - 1} (1- t)^{\beta-1}dt
  ~~ \mbox{where} ~~
  B(\alpha,\beta) \defeq
  B(1;\alpha,\beta) = \frac{\Gamma(\alpha)
    \Gamma(\beta)}{\Gamma(\alpha+\beta)}
\end{equation*}
denotes the incomplete beta function.  It is possible to sample from this
distribution using inverse CDF transformations and
continued fraction approximations to the log incomplete beta
function~\cite{PressFlTeVe92}.


\begin{algorithm}[t]
\caption{Privatized Unit Vector: $\PrivUnit$}
\label{alg:unit_mech}
\begin{algorithmic}
\Require $u \in \sphere^{d-1}$, $\gamma \in [0,1]$, $p \ge \half$.
\State Draw random vector $V$ according to the distribution
\State \begin{equation}
	  \label{eqn:w-flip-mechanism}
	  V = \begin{cases}
	    \mbox{uniform on } \{v \in \sphere^{d-1} \mid
	    \<v, u\> \ge \gamma\} & \mbox{with~probability~} p \\
	    \mbox{uniform on } \{v \in \sphere^{d-1} \mid
	    \<v, u\> < \gamma \} & \mbox{otherwise}.
	  \end{cases}
\end{equation}
\State Set $\alpha = \frac{d-1}{2}$, $\tau = \frac{1+\gamma}{2}$, and
\begin{equation}
  m = \frac{(1 - \gamma^2)^\alpha}{2^{d-2} (d - 1)}
  \left[\frac{p}{B(\alpha,\alpha) - B(\tau; \alpha,\alpha)}
    - \frac{1 - p}{B(\tau; \alpha, \alpha)}\right]
  \label{eqn:norm-of-W}
\end{equation}
\Return $Z = \frac{1}{m} \cdot V$
\end{algorithmic}
\end{algorithm}

In the remainder of this subsection, we describe the privacy preservation,
bias, and variance properties of
Algorithm~\ref{alg:unit_mech}.

\subsubsection{Privacy analysis}

Most importantly, Algorithm~\ref{alg:unit_mech} protects privacy for
appropriate choices of the spherical cap level $\gamma$.  Indeed, the next
result shows that $\gamma_\diffp \approx \sqrt{\diffp / d}$ is sufficient to
guarantee $\diffp$-differential privacy.

\begin{theorem}
  \label{thm:PrivUnitPriv}
  Let $\gamma \in [0,1]$ and $p_0 = \frac{e^{\diffp_0}}{1 +
    e^{\diffp_0}}$. Then algorithm \emph{$\PrivUnit(\cdot, \gamma, p_0)$}
  is $(\diffp+\diffp_0)$-differentially private whenever $\gamma \ge 0$ is
  such that
  \begin{subequations}
    \label{eqn:sufficient-gamma}
     \begin{equation}
      \label{eqn:small-suff-gamma}
      \diffp \ge
      \log\frac{ 1+\gamma \cdot \sqrt{ 2(d-1) / \pi } }{
        \hinge{1 - \gamma \cdot \sqrt{ 2(d-1) / \pi } } },
      ~~~ \mbox{i.e.} ~~~
      \gamma \le \frac{e^\diffp - 1}{e^\diffp + 1} \sqrt{\frac{\pi}{2(d-1)}},
     \end{equation}
     or
     \begin{equation}
       \label{eqn:big-suff-gamma}
       \diffp \ge \half \log(d) + \log 6 - \frac{d - 1}{2}
       \log (1 - \gamma^2) + \log \gamma
       ~~ \mbox{and} ~~
       \gamma \ge \sqrt{\frac{2}{d}}.
     \end{equation}
  \end{subequations}
\end{theorem}
\begin{proof}
  We again leverage the results in
  Appendix~\ref{sec:preliminaries-uniform-concentration}.
  The random vector $V \in \sphere^{d-1}$ in
  Alg.~\ref{alg:unit_mech} has density (conditional on $u \in
  \sphere^{d-1}$)
  \begin{equation*}
    p(v \mid u)
    \propto \begin{cases} p_0 / \P(\<U, u\> \ge \gamma) & \mbox{if~}
      \<v, u \> \ge \gamma \\
      (1 - p_0) / \P(\<U, u\> < \gamma) & \mbox{if~} \<v, u\> < \gamma.
    \end{cases}
  \end{equation*}
  We use that
  $\gamma \mapsto \P(\<U, u\> < \gamma)$ is increasing in $\gamma$ to obtain
  (by definition of $p_0$) that
  \begin{equation}
    \label{eqn:little-burrito}
    \frac{p(v \mid u)}{p(v \mid u')} \le e^{\diffp_0} \cdot \frac{\P(\<U,
      u'\> < \gamma)}{\P(\<U, u\> \ge \gamma)},
    ~~ \mbox{all~} u, u' \in \sphere^{d-1}.
  \end{equation}
  It is thus sufficient to prove that the last fraction has upper bound
  $e^\diffp$.

  We consider two cases in inequality~\eqref{eqn:little-burrito}.
  In the first, suppose that $\gamma \ge \sqrt{2/d}$. Then
  Lemma~\ref{lemma:uniform-bounds} implies
  \begin{equation*}
    \frac{\P(\<U, u'\> < \gamma)}{\P(\<U, u\> \ge \gamma)}
    \le \frac{6 \gamma \sqrt{d}}{(1 - \gamma^2)^\frac{d-1}{2}},
  \end{equation*}
  which is bounded by $e^\diffp$ when $\log 6 + \half \log d + \log \gamma -
  \frac{d-1}{2} \log(1 - \gamma^2) \le \diffp$.  In the second case,
  Lemma~\ref{lemma:small-threshold-uniform-bound-v2} implies
  \begin{equation*}
    \frac{\P(\<U, u'\> < \gamma)}{\P(\<U, u\> \ge \gamma)}
    \le \frac{1 + \gamma \sqrt{2(d-1) / \pi}}{
      1 - \gamma \sqrt{2(d - 1)/\pi}},
  \end{equation*}
  which is bounded by $e^\diffp$ if and only if $\gamma \le \frac{e^\diffp -
    1}{e^\diffp + 1} \sqrt{\pi / (2(d-1))}$.
\end{proof}

\subsubsection{Bias and variance}

\newcommand{\ball}{\mathbb{B}}

We now turn to optimality and error properties of
Algorithm~\ref{alg:unit_mech}. Our first result is an lower bound
on the $\ell_2$-accuracy of any private mechanism, which
follows from the paper of~\citet{DuchiRo19}.


\begin{proposition}
  \label{proposition:noise-good-means}
  Assume that the mechanism $M: \sphere^{d-1} \to \R^d$ is any of
  $\diffp$-differentially private, $(\diffp, \delta)$-differentially private
  with $\delta \le \half$, or $(\diffp \wedge \diffp^2,
  \renparam)$-R\'{e}nyi differentially private for input $x \in
  \sphere^{d-1}$, all with $\diffp \le d$. Then for $X$ uniformly
  distributed in $\{\pm 1 / \sqrt{d}\}^d$,
  \begin{equation*}
    \E[\ltwo{M(X) - X}^2] \ge c \cdot
    \frac{d}{\diffp \wedge \diffp^2} \wedge 1,
  \end{equation*}
  where $c > 0$ is a numerical constant. Moreover, if $M$ is unbiased,
  then $\E[\ltwo{M(X) - X}^2] \ge c \frac{d}{\diffp \wedge \diffp^2}$.
\end{proposition}
\begin{proof}
  The first result is immediate by the result~\cite[Corollary~3]{DuchiRo19}.
  For the unbiasedness lower bound, note that if for a constant $c_0 < c$
  we have $\E[\ltwo{M(X) - X}^2] \le c_0 \frac{d}{\diffp \wedge \diffp^2}$,
  then given a sample $X_1, \ldots, X_n \in \{\pm 1 / \sqrt{d}\}^d$ drawn
  i.i.d.\ from a population with mean $\theta = \E[X_i]$, setting
  $Z_i = M(X_i)$ we would
  have $\E[\ltwo{\wb{Z}_n - \theta}^2]
  \le \frac{c_0 d}{n (\diffp \wedge \diffp^2)}$. For small enough
  constant $c_0$, this contradicts~\cite[Cor.~3]{DuchiRo19}.
\end{proof}

As a consequence of Proposition~\ref{proposition:noise-good-means}, we can
show that Algorithm~\ref{alg:unit_mech} is order optimal for all privacy
levels $\diffp \le d \log 2 - \log \frac{4}{3}$, improving on all previously
known mechanisms for (locally) differentially private vector release. To see
this, we show that $\PrivUnit$ indeed produces an unbiased estimator with
small norm. See Appendix~\ref{sec:proof-l2-unbiased} for a proof
of the next lemma.
\begin{lemma}
  \label{lem:unbiased}
  Let $Z = \PrivUnit(u,\gamma, p)$ for some $u \in \sphere^{d-1}$, $\gamma
  \in [0,1]$, and $p \in [\half, 1]$. Then $\E[Z] = u$.
\end{lemma}

Letting $\gamma$ satisfy either of the sufficient
conditions~\eqref{eqn:sufficient-gamma} in $\PrivUnit(\cdot, \gamma,
p_0)$, where $p_0 = e^{\diffp_0} / (1 + e^{\diffp_0})$,
ensures that it is $(\diffp+\diffp_0)$-differentially private.  With
these choices of $\gamma$, we
then have the following utility guarantee for the privatized vector $Z$.
\begin{proposition}
  Assume that $0 \le \diffp \le d$. Let $u \in \sphere^{d-1}$ and $p \ge
  \half$.  Then there exists a numerical constant $c < \infty$ such that if
  $\gamma$ saturates either of the two inequalities
  \eqref{eqn:sufficient-gamma}, then $\gamma \gtrsim \min\{\diffp / \sqrt{d},
  \sqrt{\diffp / d}\}$, and the output $Z = \PrivUnit(u,\gamma, p)$
  satisfies
  \begin{equation*}
    \ltwo{Z}
    \le c \cdot \sqrt{\frac{d}{\diffp} \vee
      \frac{d}{(e^\diffp - 1)^2}}.
  \end{equation*}
  Additionally, $\E[\ltwo{Z - u}^2] \lesssim
  \frac{d}{\diffp} \vee \frac{d}{(e^\diffp - 1)^2}$.
 \label{proposition:ltwo-utility}
\end{proposition}
\noindent
See Appendix~\ref{sec:proof-ltwo-utility} for a proof.

The salient point here is that the mechanism of
Alg.~\ref{alg:unit_mech} is order optimal---achieving unimprovable
dependence on the dimension $d$ and privacy level $\diffp$---and
substantially improving the earlier results of \citet{DuchiJoWa18}, who
provide a different mechanism that achieves order-optimal guarantees only
when $\diffp \lesssim 1$.
More generally, as we see presently, this mechanism forms
the lynchpin for minimax optimal stochastic optimization.


\subsection{Privatizing unit $\ell_\infty$ vectors with high accuracy}

We now consider privatization of vectors on the surface of the unit
$\ell_\infty$ box, $\cube^d \defeq [-1, 1]^d$, constructing an
$\diffp$-differentially private vector $Z$ with the property that $\E[Z \mid
  u] = u$ for all $u \in \cube^d$. The importance of this setting arises in
very high-dimensional estimation and statistical learning problems,
specifically those in which the dimension $d$ dominates the sample size $n$.
In these cases, mirror-descent-based methods~\cite{NemirovskiJuLaSh09,
  BeckTe03} have convergence rates for stochastic optimization problems that
scale as $\frac{M_\infty R_1 \sqrt{\log d}}{\sqrt{T}}$, where $M_\infty$
denotes the $\ell_\infty$-radius of the gradients $\nabla \loss$ and $R_1$
the $\ell_1$-radius of the constraint set $\Theta$ in the
problem~\eqref{eqn:pop-risk}. With the $\ell_2$-based mechanisms in the
previous section, we thus address the two most important scenarios for
online and stochastic optimization.

Our procedure parallels that for the $\ell_2$ case, except that we now use
caps of the hypercube rather than the sphere.  Given $u \in \cube^d$, we
first round each coordinate randomly to $\pm 1$ to generate $\what{u} \in
\{-1, 1\}^d$ with $\E[\what{u} \mid u] = u$.  We then sample a privatized
vector $V \in \{-1,+1 \}^{d}$ such that with probability $p \ge \half$ we
have $V \in \{v \mid \<v, \what{u} \> > \kappa\}$, while with the remaining
probability $V \in \{v \mid \<v, \what{u} \> \leq \kappa\}$, where $\kappa
\in \{0, \cdots, d-1 \}$.  We debias the resulting vector to construct
$Z$ satisfying $\E[Z \mid u] = u$. See
Algorithm~\ref{alg:unit_mechINFTY}.
\begin{algorithm}
\caption{Privatized Unit Vector: $\PrivUnitInfty$}
\label{alg:unit_mechINFTY}
\begin{algorithmic}
\Require $u \in [-1,1]^d$, $\kappa \in \{0,\cdots, d-1\}, p \ge \half$.
\State Round each coordinate of $u \in [-1, 1]^d$ to a corner of $\cube^{d}$:
\begin{equation*}
  \what{U}_j = \begin{cases}
    1 & \text{w.p.} ~ \frac{1 + u_j }{2} \\
    -1 & \text{otherwise}
  \end{cases}
  ~~ \mbox{for~} j \in [d].
\end{equation*}
\State Draw random vector $V$ via
\State \begin{equation}
  \label{eqn:w-flip-mechanismINFTY}
  V = \begin{cases}
    \mbox{uniform on } \{v \in \{-1,+1 \}^d \mid
    \<v, \what{U}\> > \kappa \} & \mbox{with~probability~} p \\
    \mbox{uniform on } \{v \in \{-1,+1 \}^d \mid
    \<v, \what{U}\> \leq \kappa \} & \mbox{otherwise}.
  \end{cases}
\end{equation}
\State Set $\tau = \frac{\lceil \frac{d+\kappa+1}{2} \rceil }{d}$
and
\begin{equation*}
  m = p\frac{\choose{d-1}{d \tau - 1} }{\sum_{\ell = \tau \cdot d}^d \choose{d}{\ell} }
  - (1 - p) \frac{ \choose{d-1}{d \tau -1 } }{\sum_{\ell = 0}^{d\tau -1 } \choose{d}{\ell} }
\end{equation*}
\Return $Z = \frac{1}{m} \cdot V$
\end{algorithmic}
\end{algorithm}


As in Section~\ref{sec:private-l2}, we divide our analysis into a proof
that Algorithm~\ref{alg:unit_mechINFTY} provides privacy and an argument
for its utility.

\subsubsection{Privacy analysis}

We follow a similar analysis to Theorem~\ref{thm:PrivUnitPriv} to give the
precise quantity that we need to bound to ensure (local) differential
privacy.  
We defer the proof to Appendix~\ref{appendix:proof-privacy-infinity}.

\begin{theorem}
  Let $\kappa \in \{0, \cdots, d-1 \}$, $p_0 = \frac{e^{\diffp_0}}{1 +
    e^{\diffp_0}}$ for some $\diffp_0 \ge 0$, and $\tau \defeq \frac{\lceil
    \frac{d+\kappa+1}{2} \rceil }{d}$.  If
  \begin{equation}
    \log\left( \sum_{\ell = 0}^{d \tau-1 } \choose{d}{\ell} \right) - \log\left( \sum_{\ell = d\tau }^{d} \choose{d}{\ell} \right) \leq \diffp
    \label{eq:sufficient_kappaGEN}
  \end{equation}
  then $\PrivUnitInfty( \cdot, \kappa, p_0)$ is
  $(\diffp+\diffp_0)$-differentially private.
  \label{theorem:privacy-infinity}
\end{theorem}

By approximating \eqref{eq:sufficient_kappaGEN}, we can understand the
scaling for $\kappa$ on the dimension and the privacy parameter $\epsilon$.
Specifically, we show that when $\epsilon = \Omega(\log(d))$, setting
$\kappa \approx \sqrt{\diffp d}$ guarantees $\diffp$
differential privacy; similarly, for any $\diffp = O(1)$, setting $\kappa
\approx \diffp \sqrt{d}$ gives $\diffp$-differential privacy.
\begin{corollary}
  Assume that $d, \kappa \in \Z$ are both even and let $p_0 = e^{\diffp_0} /
  (1 + e^{\diffp_0})$.  If $0 \leq \kappa < \sqrt{3/2d+1}$ and
  \begin{equation}
    \diffp \geq   \log\left( 1 +\kappa\cdot \sqrt{\frac{2}{3d +2} } \right) - \log\left(1 - \kappa\cdot  \sqrt{\frac{2}{3d +2} } \right),
    \label{eq:sufficient_kappa1}
  \end{equation}
  then $\PrivUnitInfty(\cdot, \kappa, p_0)$ is $(\diffp+\diffp_0)$-DP.
  Let $\kappa_2 = \kappa + 2$. Then if
  \begin{align}
    \diffp & \geq \half\log(2)+\frac{1}{2}
    \log\left(d - \tfrac{\kappa_2^2}{d}\right)
    + \frac{d}{2} \left[
      \left(1 + \frac{\kappa_2}{d}\right)
      \log \left(1 + \frac{\kappa_2}{d}\right)
      + \left(1 - \frac{\kappa_2}{d}\right)
      \log \left(1 - \frac{\kappa_2}{d} \right)\right],
    \label{eq:sufficient_kappa3}
  \end{align}
  $\PrivUnitInfty(\cdot, \kappa, p_0)$ is $(\diffp+\diffp_0)$-DP.  
\label{corollary:ell_infty}
\end{corollary}

\subsubsection{Bias and variance}

Paralleling our analysis of the $\ell_2$-case, we now analyze the utility of
our $\ell_\infty$-privatization mechanism.  We first prove that
$\PrivUnitInfty$ indeed produces an unbiased estimator.
\begin{lemma}
  \label{lem:unbiasedINFTY}
  Let $Z = \PrivUnitInfty(u,\kappa)$ for some $\kappa \in \{0, \cdots, d \}$
  and $u \in [-1, 1]$. Then $\E[Z] = u$.
\end{lemma}
\noindent
See Appendix~\ref{sec:proof-unbiased-infinity} for a proof.

The results of \citet{DuchiJoWa18} imply that for $u \in \{-1,1\}^d$ the
output $Z = \PrivUnitInfty(u,\kappa = 0, p)$ has magnitude $\linf{Z}
\lesssim \sqrt{d} \tfrac{p}{1 - p}$, which is $\sqrt{d}\tfrac{e^\diffp +
  1}{e^\diffp - 1}$ for $p = e^{\diffp} / (1 + e^{\diffp})$. We can,
however, provide stronger guarantees. Letting $\kappa$ satisfy the
sufficient condition~\eqref{eq:sufficient_kappaGEN} in
$\PrivUnitInfty(\cdot, \kappa, p_0)$ for $p_0 =
\frac{e^{\diffp_0}}{e^{\diffp_0} + 1}$ ensures that $Z$ is
$(\diffp+\diffp_0)$-differentially private, and we have the
utility bound
\begin{proposition}
  \label{proposition:linf-utility}
  Let $u \in \{-1,1 \}^d$, $p \ge \half$, and $Z = \PrivUnitInfty(u,\kappa,
  p)$. Then $\E[Z] = u$, and there exist numerical constants $0 < c_0, c_1 <
  \infty$ such that the following holds.
  \begin{enumerate}[(i)]
  \item \label{item:linf-large-eps}
    Assume that $\diffp \ge \log d$.  If $\kappa$ saturates the
    bound~\eqref{eq:sufficient_kappa3}, then $\kappa \ge c_0 \sqrt{\diffp
      d}$ and
    \begin{equation*}
      \linf{Z}  \leq c_1
      \sqrt{\frac{d}{\diffp}}.
    \end{equation*}
  \item Assume that $\diffp < \log d$. If
    $\kappa$ saturates the bound~\eqref{eq:sufficient_kappa1},
    then $\kappa \ge c_0 \min\{\sqrt{d}, \diffp \sqrt{d}\}$,
    and
    \begin{equation*}
      \linf{Z}
      \leq c_1 \frac{\sqrt{d}}{\min\{1, \diffp\}}.
    \end{equation*}
  \end{enumerate}
\end{proposition}
\noindent
See Appendix~\ref{sec:proof-linf-utility} for a proof.

Thus, comparing to the earlier guarantees of \citet{DuchiJoWa18}, we see
that this hypercube-cap-based method we present in
Algorithm~\ref{alg:unit_mechINFTY} obtains no worse error in all cases of
$\diffp$, and when $\diffp \ge \log d$, the dependence on $\diffp$ is
substantially better. An argument paralleling that for
Proposition~\ref{proposition:noise-good-means} shows that the bounds on the
$\ell_\infty$-norm of $Z$ are unimprovable except for $\diffp \in [1, \log
  d]$; we believe a slightly more careful probabilistic argument should show
that case~\eqref{item:linf-large-eps} holds for $\diffp \ge 1$.

%


\subsection{Privatizing the magnitude}

The final component of our mechanisms for releasing unbiased vectors is to
privately release single values $r \in [0, r_{\max}]$ for some $r_{\max} <
\infty$.  The first (Sec.~\ref{sec:abs-error-onevar}) provides a
randomized-response-based mechanism achieving order optimal scaling for the
mean-squared error $\E[(Z - r)^2]$, which is $r_{\max}^2 e^{-2 \diffp
  / 3}$ for $\diffp \ge 1$ (see Corollary 8 in~\cite{GengVi16}).  In
the second (Sec.~\ref{sec:rel-error-onevar}), we provide a mechanism
that achieves better relative error guarantees---important for
statistical applications in which we wish to adapt to the ease of a problem
(recall the introduction), so that ``easy'' (small magnitude update) examples
indeed remain easy.


\subsubsection{Absolute error}
\label{sec:abs-error-onevar}

We first discuss a generalized randomized-response-based scheme for
differentially private release of values $r \in [0, r_{\max}]$, where
$r_{\max}$ is some \emph{a priori} upper bound on $r$.
We fix a value $k \in \N$ and then follow a three-phase procedure:
first, we randomly round $r$ to an index value $J$ taking values in $\{0, 1,
2, \ldots, k\}$ so that
\begin{equation*}
  \E[r_{\max} J / k \mid r] = r
  ~~ \mbox{and} ~~
  \floor{kr / r_{\max}} \le J \le \ceil{kr / r_{\max}}.
\end{equation*}
In the second step, we employ randomized response \cite{Warner65} over $k$
outcomes.  The third step debiases this randomized quantity to obtain
the estimator $Z$ for $r$.  We formalize the procedure in
Algorithm~\ref{alg:Scalar}, \texttt{ScalarDP}.

\begin{algorithm}
\caption{Privatize the magnitude with absolute error: $\ScalarDP$}
\label{alg:Scalar}
\begin{algorithmic}
  \Require Magnitude $r$, privacy parameter $\diffp > 0$, $k \in \N$,
  bound $r_{\max}$
  \State $r \gets \min\{r,r_{\max} \} $
  \State Sample $J \in \{ 0,1, \cdots, k\}$ such that 
  \begin{equation*}
    J = \begin{cases}
      \floor{k r/r_{\max}}  &  \text{w.p. }
      \left(\ceil{kr/r_{\max}} - k r/r_{\max}   \right) \\
      \ceil{ k r/r_{\max} } & \text{otherwise.}
    \end{cases}
  \end{equation*}
  \State Use randomized response to obtain
  \begin{equation*}
    \what{J} \mid (J = i) = \begin{cases}
      i & \mbox{w.p.}~ \frac{e^\diffp}{e^\diffp + k} \\
      \mbox{uniform in~} \{0, \ldots, k\} \setminus i
      & \mbox{w.p.}~ \frac{k}{e^\diffp + k}. \end{cases}
  \end{equation*}
  \State Debias $\what{J}$, by setting
  \begin{equation*}
    Z = a \left(\what{J} - b\right) ~~ \mbox{for~} a = \left( \frac{e^\diffp + k}{e^\diffp - 1} \right)
    \frac{r_{\max}}{k} ~~ \mbox{and} ~~
    b = \frac{k(k + 1)}{2 (e^\diffp + k)}.
  \end{equation*}
  \Return $Z$
\end{algorithmic}
\end{algorithm}


Importantly, the mechanism $\ScalarDP$ is $\diffp$-differentially private,
and we can control its accuracy via the next lemma, whose proof we defer to
Appendix~\ref{sec:proof-scalar-utility}.
\begin{lemma}
  \label{lem:ScalarUtil}
  Let $\diffp > 0$, $k \in \N$, and $0 \le r_{\max} < \infty$. Then the
  mechanism $\ScalarDP(\cdot,\diffp; k, r_{\max})$ is
  $\diffp$-differentially private and for $Z = \ScalarDP(r,\diffp; k,
  r_{\max})$, if $0 \le r \le r_{\max}$, then $\E[Z] = r$ and
  \begin{align*}
    \E[(Z-r)^2] & \le
    \frac{k + 1}{e^\diffp - 1}
    \left[r^2
      + \frac{r_{\max}^2}{4 k^2}
      + \frac{(2k + 1)(e^\diffp + k) r_{\max}^2}{6k (e^\diffp - 1)}\right]
    + \frac{r_{\max}^2}{4k^2}.
  \end{align*}
\end{lemma} 
\noindent
By choosing $k$ appropriately, we immediately
see that we can achieve optimal~\cite{GengVi16}
mean-squared error as $\diffp$ grows:
\begin{lemma}
  \label{lemma:one-dim-absolute-error}
  Let $k = \ceil{e^{\diffp / 3}}$.
  Then for $Z = \ScalarDP(r, \diffp ; k, r_{\max})$,
  \begin{equation*}
    \sup_{r \in [0, r_{\max}]}
    \E[(Z - r)^2 \mid r]
    \le C \cdot r_{\max}^2 e^{-2 \diffp / 3}
  \end{equation*}
  for a universal (numerical) constant $C$ independent of
  $r_{\max}$ and $\diffp$.
\end{lemma}

It is also possible to develop relative error bounds rather than absolute
error bounds; as the focus of the current paper is on large-scale
statistical learning and stochastic optimization rather than scalar
sampling, we include these relative error bounds and some related discussion
in Appendix~\ref{sec:rel-error-onevar}.  They can in some
circumstances provide stronger error guarantees than the absolute guarantees
in Lemma~\ref{lemma:one-dim-absolute-error}.

\subsection{Asymptotic analysis with local privacy}
\label{sec:asymptotics-local}

Finally, with our development of private vector sampling
mechanisms complete, we revisit the statistical risk minimization
problem~\eqref{eqn:pop-risk} and our development of asymptotics in
Section~\ref{sec:asympt}.  Recall that we wish to minimize
$\risk(\theta) = \E_P[\loss(\theta, X)]$ using a sample $X_t \simiid P$, $t
= 1, \ldots, T$.  We consider a stochastic gradient procedure, where we
privatize each stochastic gradient $\nabla \loss(\theta, X)$ using a
separated mechanism that obfuscates both the direction $\nabla \loss /
\ltwo{\nabla \loss}$ and magnitude $\ltwo{\nabla \loss}$.
Our scheme is $\diffp$-differentially private,
and we let $\diffp_1 + \diffp_2 = \diffp$, where we use $\diffp_1$ as
the privacy level for the direction and $\diffp_2$ as the privacy level for
the magnitude. For fixed $\diffp_1$, we let $\gamma(\diffp_1)$ be the
largest value of $\gamma$ satisfying one of the
inequalities~\eqref{eqn:sufficient-gamma} so that
Algorithm~\ref{alg:unit_mech} ($\PrivUnit$) is $\diffp_1$-differentially
private and $\gamma(\diffp) \gtrsim \min\{\diffp, \sqrt{\diffp}\} /
\sqrt{d}$ (recall Proposition~\ref{proposition:ltwo-utility}). We use
Alg.~\ref{alg:Scalar} to privatize the magnitude (with a maximum scalar
value $r_{\max}$ to be chosen), and thus we define the
$\diffp$-differentially private mechanism for privatizing a vector $w$ by
\begin{equation}
  M(w) \defeq
  \PrivUnit\left( \frac{w}{\ltwo{w}}  ; \gamma(\diffp_1), p = \half\right)
  \cdot \ScalarDP(\ltwo{w}, \diffp_2; k = \lceil e^{\diffp_2/3} \rceil, r_{\max}).
  \label{eqn:separated-product-mechanism}
\end{equation}
Using the mechanism~\eqref{eqn:separated-product-mechanism},
we define $Z(\theta; x) \defeq M(
\nabla \loss( \theta;x))$, where we assume a known upper bound
$r_{\max}$ on $\ltwo{\nabla \loss(\theta; x)}$.
The the private stochastic gradient
method then iterates
\begin{equation*}
  \label{eqn:sgm-iteration-private}
  \theta^{(t+1)} \gets \theta^{(t)} - \stepsize_t  \cdot Z(\theta^{(t)};X_t)
\end{equation*}
for $t = 1, 2, \ldots$ and $X_t \simiid P$, where $\stepsize_t$ is a
stepsize sequence.

To see the asymptotic behavior of the average $\wb{\theta}^{(T)} = \frac{1}{T}
\sum_{t = 1}^T \theta^{(t)}$, we will use
Corollary~\ref{corollary:priv-asympt}. We begin by
computing the asymptotic variance $\Sigma^{\mathrm{priv}}
= \E[Z(\theta\opt; X) Z(\theta\opt; X)^\top]$.
\begin{lemma}
  \label{lemma:asymptotic-variance-private}
  Assume that $0 < \diffp_1, \diffp_2 \le d$ and let $Z$ be defined as
  above.  Let $\Sigma = \cov(\nabla
  \loss(\theta\opt; X))$ and $\Sigma_{\rm norm} = \cov(\nabla
  \loss(\theta\opt; X) / \ltwo{\nabla\loss(\theta\opt; X)})$. Assume
  additionally that $\ltwo{\nabla \loss(\theta\opt; X)} \le r_{\max}$ with
  probability 1. Then there exists a numerical constant $C < \infty$
  such that
  \begin{equation*}
    \Sigma^{\rm priv}
    \preceq C \cdot \frac{d r^2_{\max} e^{-2\diffp_2/3}}{\diffp_1
    \wedge \diffp_1^2}\cdot
    \left(\Sigma_{\rm norm} + \frac{\tr(\Sigma_{\rm norm})}{d} I_d\right)
    + C \cdot \frac{d}{\diffp_1 \wedge \diffp_1^2} \cdot
    \left(\Sigma + \frac{\tr(\Sigma)}{d} I_d\right).
  \end{equation*}
\end{lemma}
\noindent
See Appendix~\ref{sec:proof-asymptotic-variance-private}
for the proof.

Lemma~\ref{lemma:asymptotic-variance-private} is the key result from
which our main convergence theorem builds. Combining this
result with Corollary~\ref{corollary:priv-asympt}, we obtain
the following theorem, which highlights the asymptotic convergence
results possible when we use somewhat larger privacy parameters $\diffp$.
\begin{theorem}
  \label{theorem:final-asymptotic-normality}
  Let the conditions of Lemma~\ref{lemma:asymptotic-variance-private}
  hold. Define the optimal asymptotic covariance $\Sigma\subopt \defeq
  \cov(\nabla \loss(\theta\opt; X))$, and assume that
  $\lambda_{\min}(\Sigma\subopt) = \lambda_{\min} > 0$.  Let the privacy
  levels $0 < \diffp_1, \diffp_2$ satisfy $\diffp_2 \ge \frac{3}{2} \log
  \frac{d}{\diffp_1 \lambda_{\min}}$ and $0 < \diffp_1 \le d$. Assume that
  the stepsizes $\stepsize_t \propto t^{-\beta}$ for some $\beta \in (1/2,
  1)$, and let $\theta^{(t)}$ be generated by the private
  stochastic gradient method~\eqref{eqn:sgm-iteration-private}. Then
  \begin{equation*}
    \sqrt{T} \left(\wb{\theta}^{(T)} - \theta\opt\right)
    \cd \normal\left(0, \Sigma^{\textup{priv}}\right)
  \end{equation*}
  where
  \begin{equation*}
    \Sigma^{\textup{priv}}
    \preceq O(1) \frac{d}{\diffp_1 \wedge \diffp_1^2}
    \nabla^2 \risk(\theta\opt)^{-1}
    \left(\Sigma\subopt + \frac{\tr(\Sigma\subopt)}{d}
    I_d \right) \nabla^2\risk(\theta\opt)^{-1}.
  \end{equation*}
\end{theorem}

\newcommand{\Sigmaworst}{\Sigma^{\textup{max}}}
\newcommand{\normratio}{\sigma^2_{\textup{ratio}}}

\subsubsection{Optimality and alternative mechanisms}

We provide some commentary on
Theorem~\ref{theorem:final-asymptotic-normality} by considering alternative
mechanisms and optimality results.  We begin with the latter.  It is first
instructive to compare the asymptotic covariance $\Sigma^{\rm priv}$
Theorem~\ref{theorem:final-asymptotic-normality} to the optimal asymptotic
covariance without privacy, which is $\nabla^2 \loss(\theta\opt)^{-1}
\Sigma\subopt \nabla^2 \loss(\theta\opt)^{-1}$
(cf.~\cite{DuchiRu19,LeCamYa00,VanDerVaart98}).  When the privacy level
$\diffp_1$ scales with the dimension, our asymptotic covariance can is
within a numerical constant of this optimal value whenever
\begin{equation*}
  \frac{\tr(\Sigma\subopt)}{d}
  \nabla^2 \risk(\theta\opt)^{-2}
  \preceq O(1) \cdot \nabla^2 \risk(\theta\opt)^{-1} \Sigma\subopt
  \nabla^2 \risk(\theta\opt)^{-1}.
\end{equation*}
When $\Sigma\subopt$ is near identity, for example, this domination in the
semidefinite order holds.  We can of course never quite achieve optimal
covariance, because the privacy channel forces some loss of efficiency, but
this loss of efficiency is now bounded.  Even when $\diffp_1$ is smaller,
however, the results of \citet{DuchiRo19} imply that in a (local) minimax
sense, there \emph{must} be a multiplication of at least $O(1) d /
\min\{\diffp, \diffp^2\}$ on the covariance $\nabla^2 \loss(\theta\opt)^{-1}
\Sigma\subopt \nabla^2 \loss(\theta\opt)^{-1}$, which
Theorem~\ref{theorem:final-asymptotic-normality} exhibits. Thus,
the mechanisms we have developed are indeed minimax rate optimal.

Let us consider alternative mechanisms, including related asymptotic
results. First, consider \citeauthor{DuchiJoWa18}'s results generalized
linear model estimation~\cite[Sec.~5.2]{DuchiJoWa18}. In their case, in the
identical scenario, they achieve $\sqrt{T}(\wb{\theta}^{(T)} - \theta\opt)
\cd \normal(0, \Sigmaworst)$ where the asymptotic variance $\Sigmaworst$
satisfies
\begin{equation*}
  \Sigmaworst \succeq
  \Omega(1) \left(\frac{e^\diffp + 1}{e^{\diffp} - 1}\right)^2
  \nabla^2 \risk(\theta\opt)^{-1}
  \left(
  d \Sigma\subopt +
  \sup_{x,\theta} \ltwo{\nabla \loss(\theta; x)}^2 I_d\right)
  \nabla^2 \risk(\theta\opt)^{-1}.
\end{equation*}
There are two sources of looseness in this covariance, which is minimax
optimal for some classes of problems~\cite{DuchiJoWa18}.  First,
$\sup_{x,\theta} \ltwo{\nabla \loss(\theta; x)}^2 > \tr(\Sigma\subopt) =
\E[\ltwos{\nabla \loss(\theta\opt; X)}^2]$. Second, the error
does not decrease for $\diffp > 1$. Letting
$\normratio \defeq \sup_{x,\theta} \ltwo{\nabla \loss(\theta; x)}^2 /
\E[\ltwo{\nabla\loss(\theta\opt; X)}^2] > 1$ denote the ratio of the worst
case norm to its expectation, which may be arbitrarily
large, we have asymptotic $\ell_2^2$-error scaling
as
\begin{align*}
  \tr(\Sigmaworst)
  & \gtrsim \frac{d}{\diffp^2 \wedge 1} \cdot
  \tr(\nabla^2 \risk(\theta\opt)^{-1} \Sigma\subopt \nabla^2
  \risk(\theta\opt)^{-1})
  + \frac{\normratio}{\diffp^2 \wedge 1} \cdot
  \tr(\Sigma\subopt) \tr(\nabla^2 \risk(\theta\opt)^{-2})
  \\
  \tr(\Sigma^{\rm priv})
  & \lesssim \frac{d}{\diffp^2 \wedge \diffp}
  \cdot \tr(\nabla^2 \risk(\theta\opt)^{-1} \Sigma\subopt \nabla^2
  \risk(\theta\opt)^{-1})
  + \frac{1}{\diffp^2 \wedge \diffp} \cdot
  \tr(\Sigma\subopt) \tr(\nabla^2 \risk(\theta\opt)^{-2}).
\end{align*}
The scaling of $\tr(\Sigma^{\rm priv})$ reveals the importance of separately
encoding the magnitude of $\ltwo{\nabla\loss(\theta; X)}$ and its
direction---we can be adaptive to the scale of $\Sigma\subopt$ rather than
depending on the worst-case value $\sup_{x,\theta}\ltwo{\nabla \loss(\theta;
  x)}^2$.





Given the numerous relaxations of differential privacy~\cite{Mironov17,
  DworkRo16, BunSt16} (recall Sec.~\ref{sec:diffp-definitions}, a natural
idea is to simply add noise satisfying one of these weaker definitions to
the normalized vector $u = \nabla \loss / \ltwo{\nabla \loss}$ in our
updates. Three considerations argue against this idea.  First, these
weakenings can never actually protect against a reconstruction breach for
all possible observations $z$ (Definition~\ref{defn:recon-breach})---they
can only protect conditional on the observation $z$ lying in some
appropriately high probability set (cf.~\cite[Thm.~1]{BarberDu14a}).
Second, most standard mechanisms add more noise than ours. Third, in a
minimax sense---none of the relaxations of privacy even allow convergence
rates faster than those achievable by pure $\diffp$-differentially private
mechanisms~\cite{DuchiRo19}.

Let us touch briefly on the second claim above about noise addition.  In
brief, our $\diffp$-differentially private mechanisms for privatizing a
vector $u$ with $\ltwo{u} \le 1$ in
Sec.~\ref{sec:mechanisms} release $Z$ such that $\E[Z \mid u] = u$ and
$\E[\ltwo{Z - u}^2 \mid u] \lesssim d \max\{\diffp^{-1}, \diffp^{-2}\}$,
which is unimprovable.  In contrast, the Laplace mechanism and its $\ell_2$
extensions~\cite{DworkMcNiSm06} satisfy $\E[\ltwo{Z - u}^2 \mid u] \gtrsim
d^2 / \diffp^2$, which which yields worse dependence on the dimension
$d$. Approximately differentially private schemes, which allow a $\delta$
probability of failure where $\delta$ is typically assumed sub-polynomial in
$n$ and $d$~\cite{DworkKeMcMiNa06}, allow mechanisms such as Gaussian noise
addition, where $Z = u + W$ for $W \sim \normal(0, \frac{C \log
  \frac{1}{\delta}}{\diffp^2})$ for $C$ a numerical constant. Evidently,
these satisfy $\E[\ltwo{Z - u}^2 \mid u] \gtrsim d \log \frac{1}{\delta} /
\diffp^2$, which again is looser than the mechanisms we provide whenever
$\diffp \lesssim \log \frac{1}{\delta}$. Other relaxations---R\'{e}nyi
differential privacy~\cite{Mironov17} and concentrated differential
privacy~\cite{DworkRo16,BunSt16}---similarly cannot yield improvements in a
minimax sense~\cite{DuchiRo19}, and they provide guarantees that the
posterior beliefs of an adversary change little only on average.


\section{Empirical Results}\label{sec:results}

We present a series of empirical results in different settings,
demonstrating the performance of our (minimax optimal) procedures
for stochastic optimization in a variety of scenarios.
In the settings we consider---which simulate a large dataset distributed
across multiple devices or units---the non-private alternative is to
communicate and aggregate model updates without local or centralized
privacy.
We perform both simulated experiments
(Sec.~\ref{sec:simulated-logreg})---where we can more precisely show losses
due to privacy---and experiments on a large image classification task and
language modeling. Because of the potential applications in modern practice,
we use both classical (logistic regression) models as well as modern deep
network architectures~\cite{LeCunBeHi15}, where we of course cannot prove
convergence but still guarantee privacy.

In each experiment, we use the $\ell_2$-spherical cap sampling mechanisms of
Alg.~\ref{alg:unit_mech} in $(\diffp_1,\diffp_2)$-separated
differentially private mechanisms~\eqref{eqn:separated-product-mechanism}.
Letting $\gamma(\diffp)$ be the largest value of $\gamma$ satisfying the
privacy condition~\eqref{eqn:sufficient-gamma} in our $\ell_2$ mechanisms
and $p(\diffp) = \frac{e^{\diffp}}{1 + e^{\diffp}}$, for any vector $w \in
\R^d$, we use
\begin{equation}
  \label{eqn:lentil-pasta}
  M(w) \defeq \PrivUnit\left(\frac{w}{\ltwo{w}};
  \gamma(0.99 \diffp_1), p(.01 \diffp_1)\right)
  \cdot \ScalarDP\left(\ltwo{w}, \diffp_2, k = \ceil{e^{\diffp_2/3}},
  r_{\max}\right).
\end{equation}
In our experiments, we set $\diffp_2 = 10$, which is large enough (recall
Theorem~\ref{theorem:final-asymptotic-normality}) so that its contribution
to the final error is negligible relative to the sampling error in
$\PrivUnit$ but of smaller order than $\diffp_1$.  In each experiment,
we vary $\diffp_1$, the dominant term in the asymptotic convergence
of Theorem~\ref{theorem:final-asymptotic-normality}.

Our goal is to investigate whether private federated statistical
learning---which includes separated differentially private mechanisms
(providing local privacy protections against reconstruction) and central
differential privacy---can perform nearly as well as models fit without
privacy. We present results both for models trained \emph{tabula rasa} (from
scratch, with random initialization) as well as those pre-trained on other
data, which is natural when we wish to update a model to better
reflect a new population. Within each figure plotting results, we
plot the accuracy of the current model $\theta_k$ at iterate $k$
versus the best accuracy achieved by a \emph{non-private} model,
providing error bars over multiple trials. In short, we find the following:
we can get to reasonably strong accuracy---nearly comparable with
non-private methods---for large values of local privacy parameter $\diffp$.
However, with smaller values, even using (provably) optimal procedures can
cause substantial performance degradation.


  
\paragraph{Centralized aggregation}
In
our large-scale real-data experiments, we include
the centralized privacy protections by projecting~\eqref{eq:server-agg} the
updates onto an $\ell_2$-ball of radius $\cliprad$, adding
$\normal(0, \sigma^2 I)$ noise with $\sigma^2 = 
Tq^2 \cliprad^2 / \diffp_\renparam$,
where $q$ is the fraction of users we subsample, $T$ the
total number of updates, and $\diffp_\renparam$ the R\'{e}nyi-privacy
parameter we choose.  In our
experiments, we report the \emph{resulting} centralized privacy levels for
each experiment.


We make a concession to computational feasibility, slightly reducing the
value $\sigma$ that we actually use in our experiments beyond the
theoretical recommendations. In particular, we use batch size $m$,
corresponding to $q = m / N$, of at most $200$ and test $\sigma \in \{.001,
.002, .005, .01\}$, depending on our experiment, which of course requires
either larger $\diffp_\renparam$ above or larger subsampling rate $q$ than
our effective rate. As \citet{DPFedLearn17} note, increasing this batch size
has negligible effect on the accuracy of the centralized model, so that we
report results (following~\cite{DPFedLearn17}) that use this inflated batch
size estimate from a population of size $N =$ 10,000,000.

\subsection{Simulated logistic regression experiments}
\label{sec:simulated-logreg}

\newcommand{\snr}{\tau}

Our first collection of experiments focuses on a logistic regression
experiment in which we can exactly evaluate population losses and errors in
parameter recovery. We generate data pairs $(X_i, Y_i) \in \sphere^{d-1}
\times \{\pm 1\}$ according to the logistic model
\begin{equation*}
  p_\theta(y \mid x) = \frac{1}{1 + \exp(-y \theta^T x)},
\end{equation*}
where the vectors $X_i$ are i.i.d.\ uniform on the sphere $\sphere^{d-1}$.
In each experiment we choose the true parameter $\theta\opt$ uniformly on
$\snr \cdot \sphere^{d-1}$ so that $\snr > 0$ reflects the
signal-to-noise ratio in the problem.  In this case, we perform the
stochastic gradient method as in Sections~\ref{sec:asympt}
and~\ref{sec:asymptotics-local} on the logistic loss $\loss(\theta; (x, y))
= \log(1 + e^{-y x^T \theta})$.  For a given privacy level $\diffp$, in the
update~\eqref{eqn:separated-product-mechanism} we use the parameters
\begin{equation*}
  \diffp_1 = \frac{13 \diffp}{16},
  ~~
  \diffp_2 = \frac{\diffp}{8},
  ~~
  p = \frac{e^{\diffp/16}}{1 + e^{\diffp/16}},
\end{equation*}
that is, we privatize the direction $g / \ltwo{g}$ with $\diffp_1
= \frac{13}{16} \diffp$-local privacy and flip probability
$p = \frac{e^{\diffp/16}}{1 + e^{\diffp/16}}$ and the magnitude $\ltwo{g}$
using $\frac{\diffp}{8}$ privacy.

Within each experiment, we draw a sample $(X_i, Y_i) \simiid P$ of size $N$
as above, and then perform $N$ stochastic gradient iterations using the
mechanism~\eqref{eqn:separated-product-mechanism}.  We choose stepsizes
$\stepsize_k = \stepsize_0 k^{-\beta}$ for $\beta = .51$, where the choice
$\stepsize_0 = \sqrt{\diffp / d}$, so that (for large magnitude noises) the
stepsize is smaller---this reflects the ``optimal'' stepsize tuning in
standard stochastic gradient methods~\cite{NemirovskiJuLaSh09}.  Letting
$L(\theta) = \E[\loss(\theta; X, Y)]$ for the given distribution, we then
evaluate $L(\theta_k) - L(\theta\opt)$ and $\ltwo{\theta_k - \theta\opt}$
over iterations $k$, where $\theta\opt = \argmin_\theta L(\theta)$.  In
Figure~\ref{fig:logistic-plots}, we plot the results of experiments using
dimension $d = 500$, sample size $N = 10^5$, and signal size $\snr =
4$. (Other dimensions, sample sizes, and signal strengths yield
qualitatively similar results.) We perform 50 independent experiments,
plotting aggregate results. On the left plot, we plot the error of the
private stochastic gradient methods, which---as we note earlier---are
(minimax) optimal for this problem. We give $90\%$ confidence intervals
across all the experiments, and we see roughly the expected behavior: as the
privacy parameter $\diffp$ increases, performance approaches that of the
non-private stochastic gradient estimator. The right plot provides box plots
of the error for the stochastic gradient methods as well as the non-private
maximum likelihood estimator.

\begin{figure}[t]
  \begin{center}
    \begin{tabular}{cc}
      \begin{overpic}[width=.49\columnwidth]{
          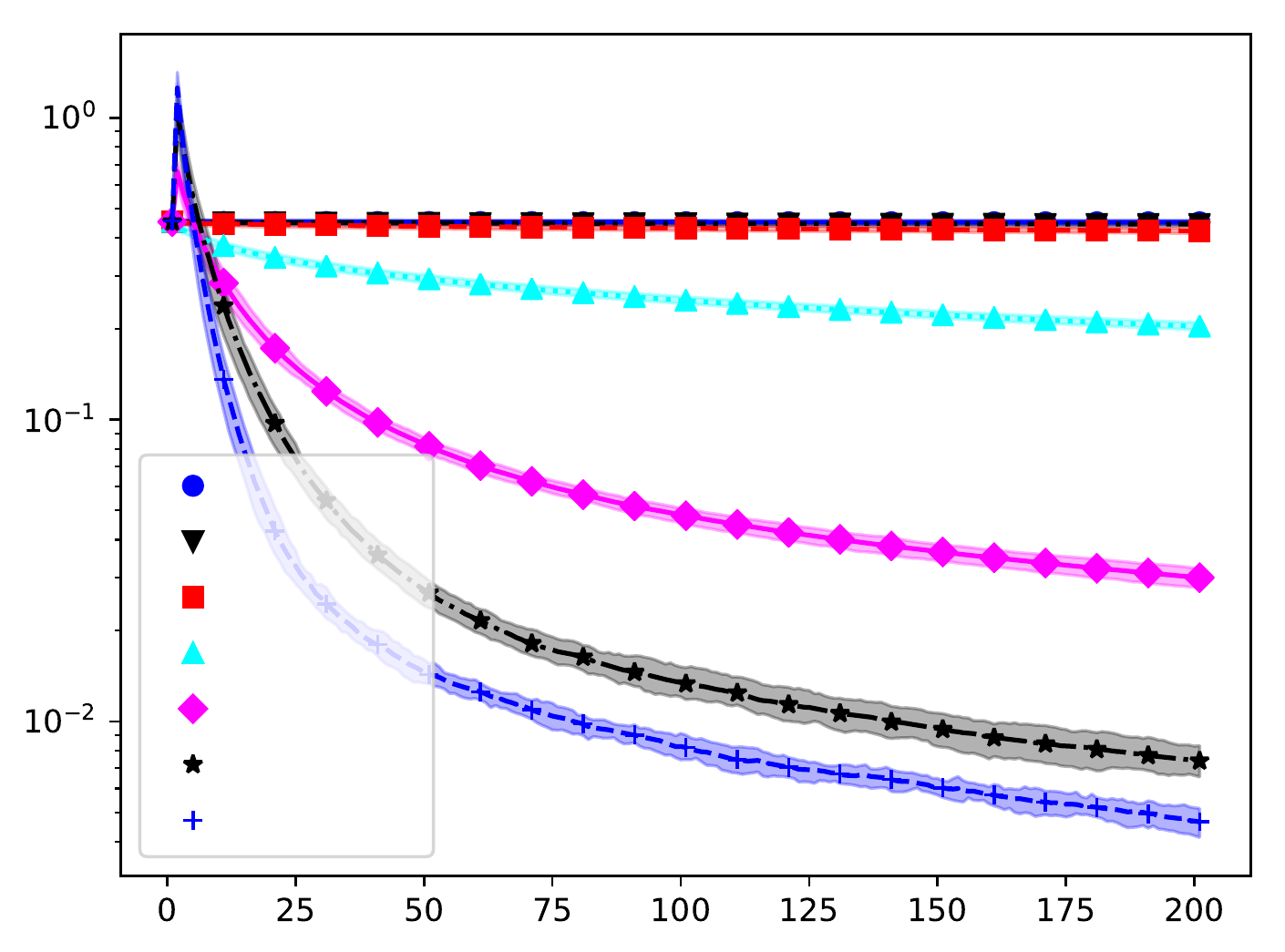}
        \put(10,2){
          \tikz{\path[draw=white,fill=white] (0, 0) rectangle (7cm,.25cm);}
        }
        \put(-5,24){
          \rotatebox{90}{{\small $L(\theta_k) - L(\theta\opt)$}}}
        \put(20,9.5){\scriptsize $\diffp = \infty$}
        \put(20,14){\scriptsize $\diffp = 250$}
        \put(20,18){\scriptsize $\diffp = 125$}
        \put(20,22){\scriptsize $\diffp = 62.5$}
        \put(20,26.5){\scriptsize $\diffp = 31.2$}
        \put(20,31){\scriptsize $\diffp = 15.6$}
        \put(20,35){\scriptsize $\diffp = 7.8$}
        \put(12.2,3){\tiny $0$}
        \put(30,3){\tiny $N/4$}
        \put(50,3){\tiny $N/2$}
        \put(70,3){\tiny $3N/4$}
        \put(92,3){\tiny $N$}
        \put(40,-2){\small Iteration $k$}
      \end{overpic} &
      \begin{overpic}[width=.47\columnwidth]{
          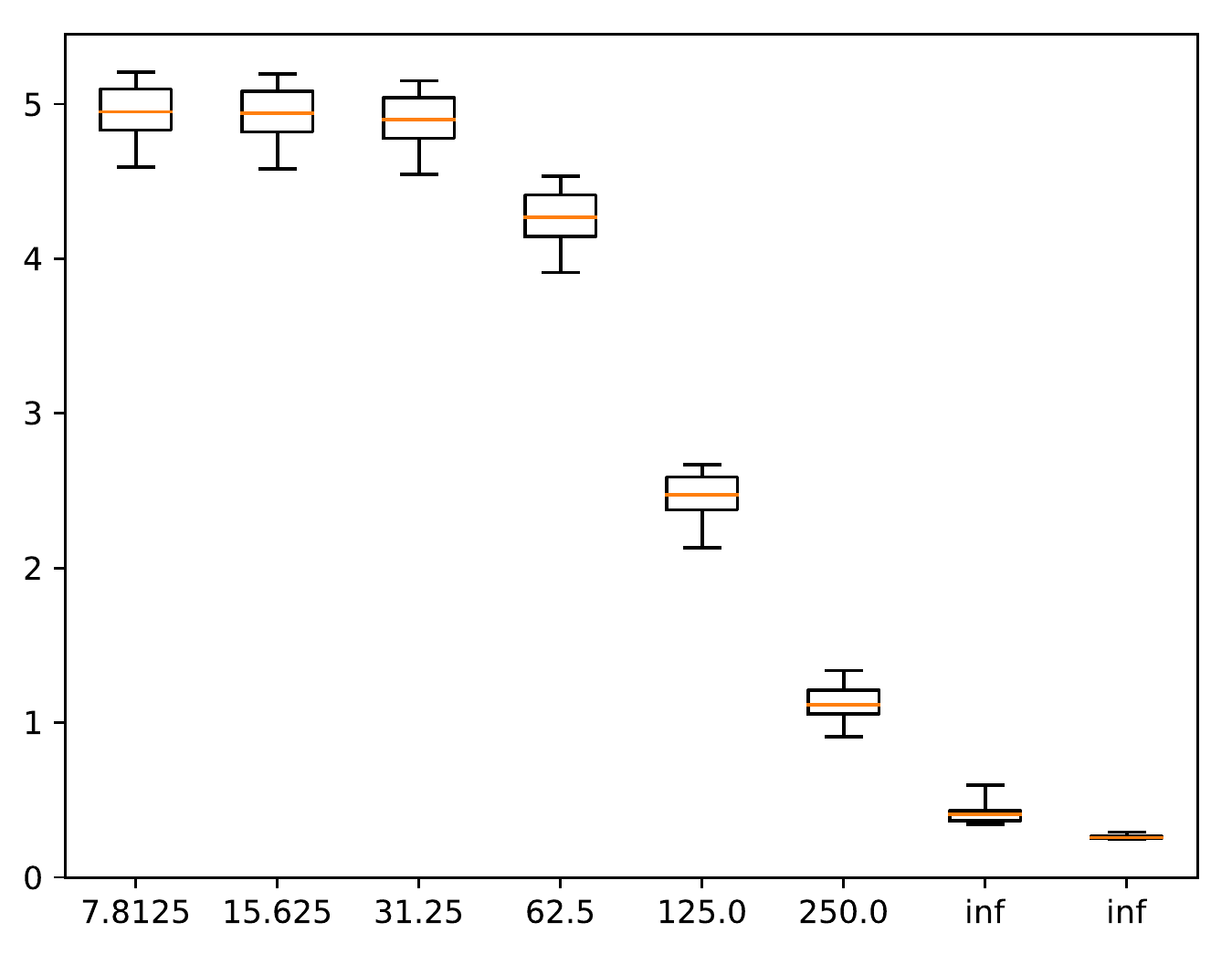}
        \put(-3,20){\rotatebox{90}{{\small error $\ltwos{\theta - \theta\opt}$}}}
        \put(48,0){{\small $\diffp$}}
        \put(75,1){
          \tikz{\path[draw=white,fill=white] (0, 0) rectangle (1.7cm,.25);}
        }
        \put(78,2){{\small $\infty$}}
        \put(89,2){{\scriptsize $\what{\theta}_{\rm mle}$}}
      \end{overpic}
    \end{tabular}
    \caption{\label{fig:logistic-plots} Logistic regression simulations with
      sample size $N = 10^5$ and dimension $d = 500$. Left: optimization
      error versus iteration $k$ in the stochastic gradient iteration, with
      95\% error bars.  Right: box plot error $\ltwos{\theta - \theta\opt}$
      of the averaged iterate $\wb{\theta}_N = \frac{1}{N} \sum_{k = 1}^N
      \theta_k$ over the stochastic gradient methods.
      The horizontal axis indexes privacy level $\diffp$,
      and $\what{\theta}_{\rm mle}$ denotes the error of the maximum
      likelihood estimator.}
  \end{center}
\end{figure}
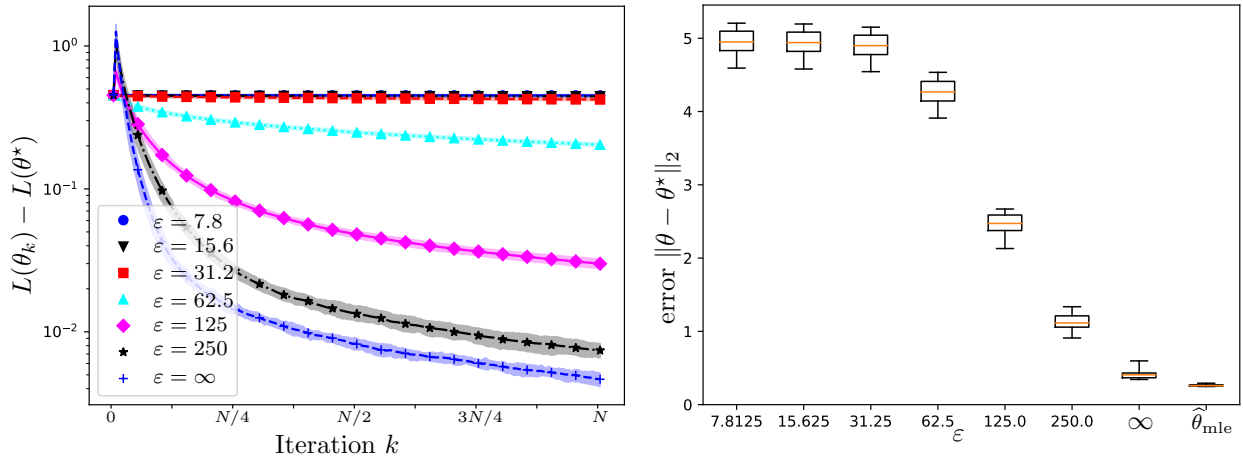

Perhaps the most salient point here is that, to maintain utility, we require
non-trivially large privacy parameters for this (reasonably)
high-dimensional problem; without $\diffp \ge d/8$ the performance is
essentially no better than that of a model using $\theta = 0$, that is,
random guessing. (And alternative stepsize choices $\stepsize_0$ do
not help.)

\subsection{Fitting deep models tabula rasa}

We now present results on deep network model fitting for image
classification tasks when we initialize the model to have
i.i.d.\ $\normal(0,1)$ parameters.  Recall our
mechanism~\eqref{eqn:lentil-pasta}, so that we allocate $\gamma =
\gamma(0.99 \diffp_1)$ for the spherical cap threshold and $p =
p(.01\cdot\diffp_1)$ for the probability with which we choose a particular
spherical cap in the randomization $\PrivUnit(\cdot, \gamma,p)$, which
ensures $\diffp_1$-differential privacy.


\begin{table}[ht]
  \centering
  \begin{tabular}{| c| c | c || c | c |c |}
    \hline
    \multicolumn{3}{|c|| }{MNIST parameters}
    & \multicolumn{3}{|c|}{CIFAR parameters} \\
    \hline
    $\diffp_1$ & $\gamma(0.99\diffp_1)$ &
    $p = \frac{e^{.01 \diffp_1}}{1 + e^{.01 \diffp_1}}$ &
    $\diffp_1$ & $\gamma(0.99\diffp_1)$ &
    $p = \frac{e^{.01 \diffp_1}}{1 + e^{.01 \diffp_1}}$ \\
    \hline
    $500$ & $0.01729$ & 1.0 & $5000$ & $0.09598$ & 1.0 \\
    $250$ &$0.01217$ & 0.924 & $1000$ &$0.04291$ & 1.0 \\
    $100$ &$0.00760$ & 0.731 & $500$ &$0.03027$ & 0.993 \\
    $50$ &$0.00526$ & 0.622 & $100$ &$0.01331$ & 0.731 \\ [1ex]
    \hline
  \end{tabular}
  \caption{\label{table:mnist}
    Parameters in experiments training six-layer CNNs
    from random initializations.}
\end{table}

\paragraph{MNIST}
We begin with results on the MNIST handwritten digit recognition
dataset~\cite{LeCunBoBeHa98}.  We use the default six-layer convolutional
neural net (CNN) architecture of the TensorFlow
tutorial~\cite{TensorFlowTutorial18} with default optimizer.  The network
contains $d = 3,274,634$ parameters.  We proced in iterations $t = 1, 2,
\ldots, T$. In each round, we randomly sample $B = 200$ sets of $m = 100$
images, then on each batch $b = 1, 2, \ldots, B$, of $m$ images, approximate
the update~\eqref{eqn:prox-point-update} by performing $5$ local gradient
steps on $\frac{1}{m} \sum_{j = 1}^m \loss(\theta, x_{b,j})$ for batch $b$
to obtain local update $\Delta_b$.  To sample magnitude $\ltwo{\Delta_b}$ of
these local updates, we use Alg.~\ref{alg:Scalar}, $\ScalarDP(\cdot,
\diffp_2=10, k = \ceil{e^{2\diffp_2 / 3}} = 29, r_{\max} = 5)$, and for the
unit vector direction privatization we use Alg.~\ref{alg:unit_mech}
($\PrivUnit$) and vary $\diffp_1$ across experiments.
Table~\ref{table:mnist} summarizes the privacy parameters
we use, with corresponding spherical cap radius $\gamma$ and
probability of sampling the correct cap $p$ in Alg.~\ref{alg:unit_mech}.
We use update radius (the $\ell_2$-ball to which we project the
stochastic gradient updates) $\cliprad = 100$ and standard deviation
$\sigma = .005$, so that the moment-accountant~\cite{Abadietal16}
guarantees that (if we use a population of size $N = 10^7$) and
100 rounds, the resulting model enjoys $(\diffp_{\textup{cent}} = 1.9,
\delta = 10^{-9})$-central differential privacy (Def.~\ref{definition:dp}).
We plot standard errors over 20 trials in Fig.~\ref{fig:tabula-rasa}(a).



\paragraph{CIFAR10}

We now present results on the CIFAR10 dataset~\cite{Krizhevsky09}. We use
the same CNN model architecture as in the Tensorflow tutorial
\citep{CifarTutorial18} with an Adam optimizer and dimension $d =$
1,068,298.  We preprocess the data as in the Tensorflow tutorial so that the
inputs are $24\times 24$ with $3$ channels.  In analogy to our experiment
for MNIST, we shuffle the training images into $B = 75$ batches, each with
$m = 500$ images, approximating the update~\eqref{eqn:prox-point-update} via
5 local gradient steps on the $m$ images.  As in the
mechanism~\eqref{eqn:lentil-pasta}, we use $\ScalarDP(\cdot, \diffp_2=10, k
= \lceil e^{\diffp_2/3}\rceil, r_{\max} = 2)$ to sample the magnitude of the
updates and $\PrivUnit(\cdot, \gamma(0.99\cdot\diffp_1), p(0.01 \cdot
\diffp_1))$, varying $\diffp_1$, for the direction.
See
Table~\ref{table:mnist} for the privacy parameters that we set in each
experiment.
We present the results in Figure~\ref{fig:tabula-rasa}(b) for
mechanisms that satisfy $(\diffp_1,\diffp_2 = 10)$-separated DP where
$\diffp_1 \in \{100,500,1000,5000 \}$.  The corresponding
$\ell_2$-projection radius $\cliprad = 30$ and centralized noise addition of
variance $\sigma = .002$ guarantee, again via the
moments-accountant~\cite{Abadietal16}, that with a ``true'' population size
$N = 10^7$ and $T= 200$ rounds, the final model is $(\diffp_{\textup{cent}}
= 1.76,\delta = 10^{-9})$-differneitally private. We plot the difference in
accuracies between federated learning and our private federated learning
system with standard errors over 20 trials.

\begin{figure}[H]
  \centering
  \begin{tabular}{cc}
    \includegraphics[width =.53\columnwidth]{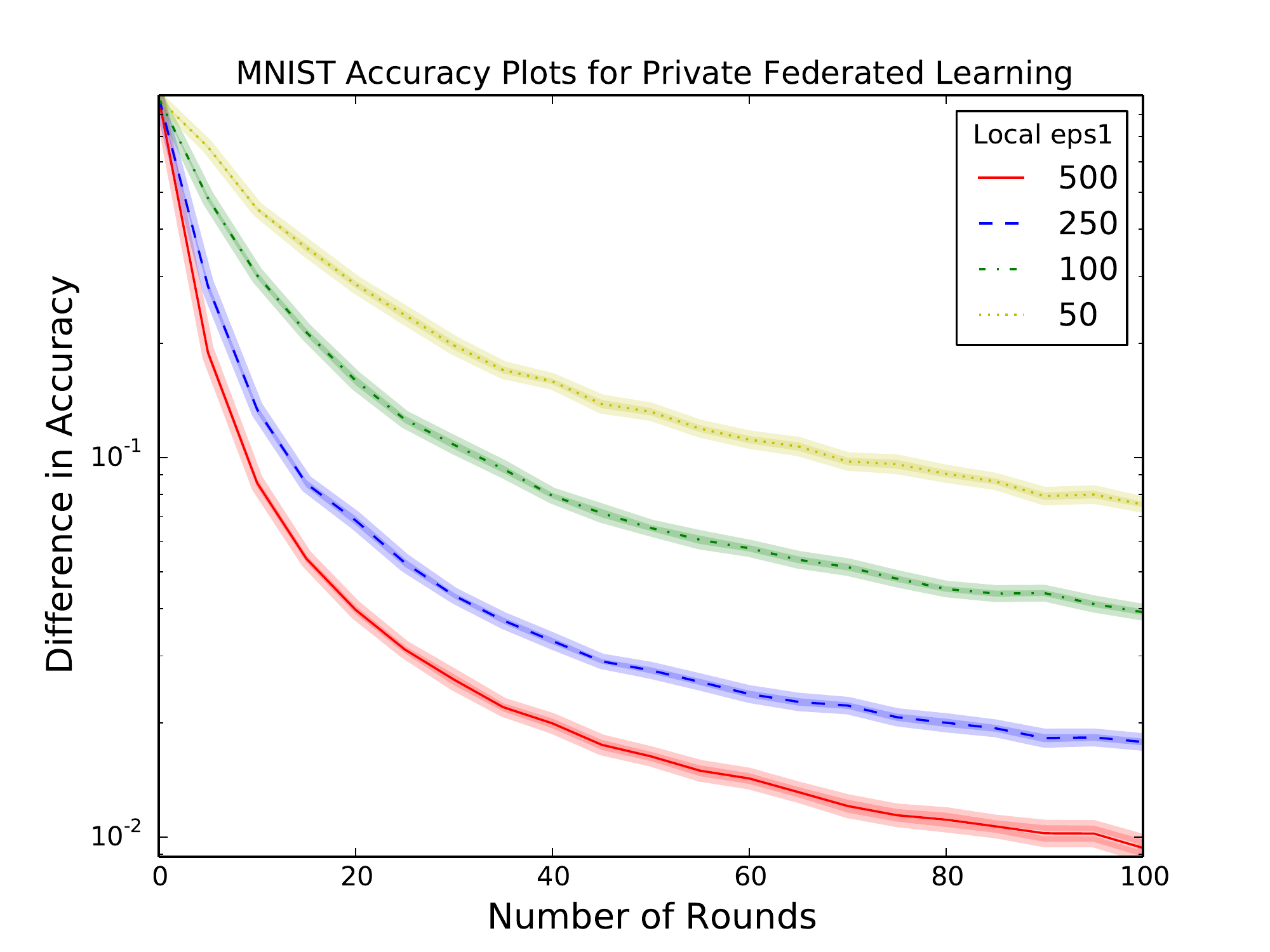}
    &
    \hspace{-1cm}
    \includegraphics[width=.53\columnwidth]{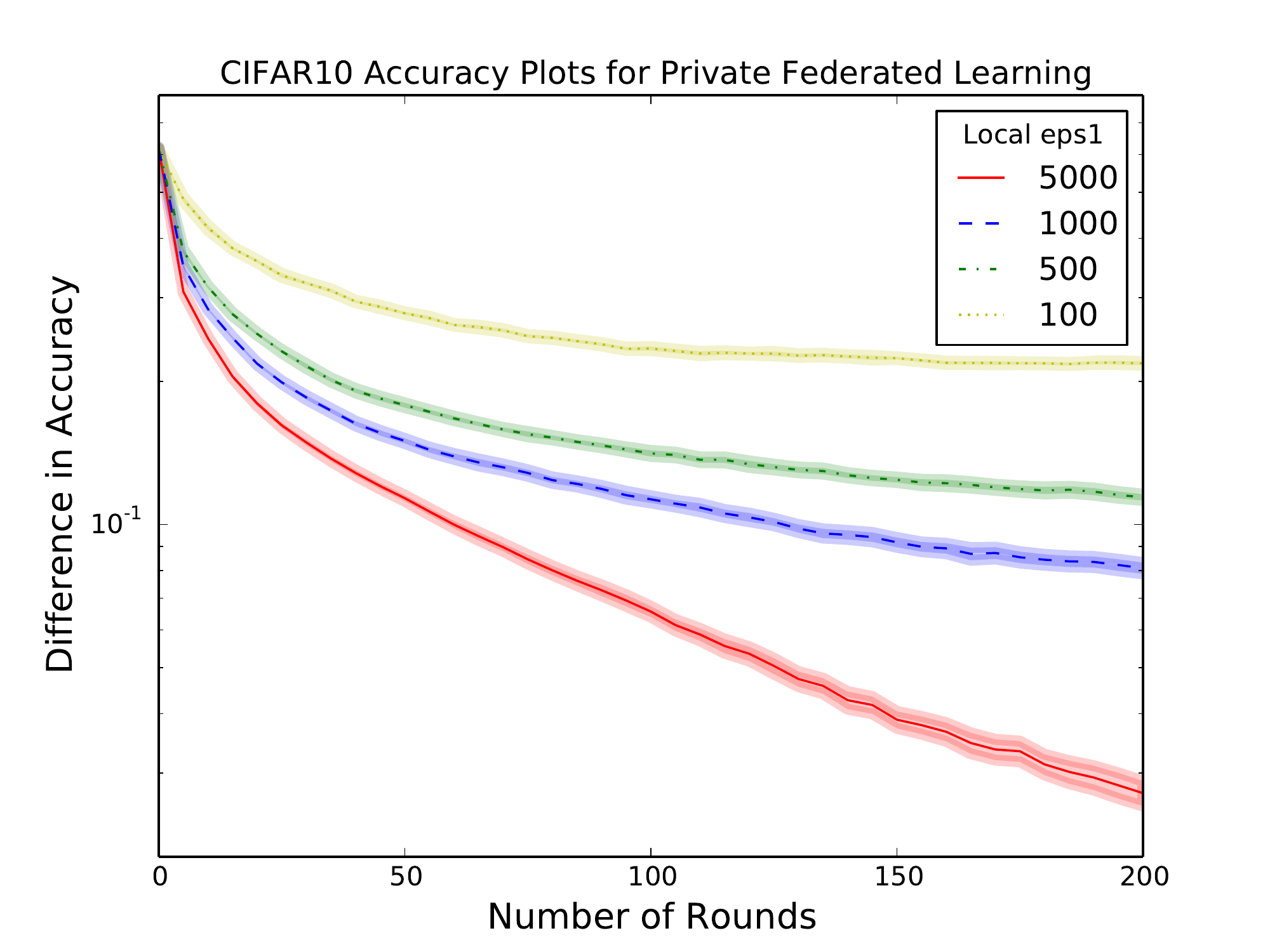} \\
    (a) & (b)
  \end{tabular}
  \caption{\label{fig:tabula-rasa} Accuracy plots for image classification
    comparing the private federated learning approach (indexed by privacy
    parameter $\diffp_1$) with non-private model updates.
    Horizontal axis indexes number of stochastic gradient
    updates, vertical the gap in test-set accuracy between
    the final non-private model and the private model $\theta_k$ at
    the given round. (a) MNIST
    dataset. (b) CIFAR-10 dataset.}
\end{figure}

\subsection{Pretrained models}
\label{sec:pretrained-models}

Our final set of experiments investigates refitting a model on a new
population. Given the large number of well-established and downloadable deep
networks, we view this fine-tuning as a realistic use case for private
federated learning.


\paragraph{Image classification on Flickr over 100 classes}
We perform our first model tuning experiment on a pre-trained ResNet50v2
network~\cite{HeZhReSu16b} fit using ImageNet
data~\cite{DengDoSoLiLiFe09}, whose reference implementation is available at
the website~\cite{TensorNets}.
Beginning from the pre-fit model,
we consider only the final (softmax) layer and final convolutional
layer of the network to be modifiable, refitting the model to
perform 100-class multiclass classification on a subset of
the Flickr corpus~\cite{ThomeeShFrElNiPoBoLi16}.
There are $d =$ 1,255,524 parameters.

We construct a subsample for each of our experiments as follows.
We choose 100  classes (uniformly at random) and $2000$ images
from each class, yielding $2 \cdot 10^5$ images. We randomly permute these
into a  9:1 train/test split. In each iteration of the stochastic
gradient method, we randomly choose $B = 100$ sets of
$m = 128$ images, performing an approximation to the proximal-point
update~\eqref{eqn:prox-point-update} using 15 gradient steps
for each batch $b = 1, 2, \ldots, B$.
%
Again following the mechanism~\eqref{eqn:lentil-pasta}, for the magnitude of
each update we use $\ScalarDP(\cdot, \diffp_2=10, k = \lceil
e^{\diffp_2/3}\rceil, r_{\max} = 10)$, while for the unit direction we use
$\PrivUnit(\cdot, \gamma(0.99\cdot\diffp_1), p(0.01 \cdot \diffp_1))$ while
varying $\diffp_1$.
We present the results in Figure
\ref{fig:Complex}(a) for mechanisms that satisfy $(\diffp_1,\diffp_2 =
10)$-separated DP where $\diffp_1 \in \{50,100,500,5000 \}$.
We plot the difference in accuracies between federated learning and our
private federated learning system with standard errors over 12 trials.


\paragraph{Next Word Prediction}
For our final experiment, we investigate performance of the private
federated learning system for next word prediction in a deep word-prediction
model.  We pretrain an LSTM on a corpus consisting of all Wikipedia entries
as of October 1, 2016~\cite{wikipedia}.  Our model architecture consists of
one long-term-short-term memory (LSTM) cell~\cite{Schmidhuber15} with a word
embedding matrix~\cite{MikolovSuChCoDe13} that maps each of 25,003 tokens
$j$ (including an unknown, end of sentence, and beginning of sentence
tokens) to a vector $w_j$ in dimension 256.  We use the Natural Language
Toolkit (NLTK) tokenization procedure to tokenize each sentence
and word~\cite{NLTK}. The LSTM cell has 256 units, which leads to 526,336
trainable parameters. Then we decode back into 25,003 tokens. In total,
there are $d = $ 13,352,875 trainable parameters in the LSTM.

We refit this pretrained LSTM on a corpus of all user comments on the
website Reddit from November 2017~\cite{Baumgartner17}, again using the NLTK
tokenization procedure~\cite{NLTK}.
Each stochastic update~\eqref{eq:server-agg} consists of
choosing a random batch of $B = 200$ collections of $m = 1000$ sentences,
performing the local updates~\eqref{eqn:prox-point-update}
approximately by computing 10 gradient steps within each batch
$b = 1, \ldots, B$. 
We use update parameters $\ScalarDP(\cdot, \diffp_2=10, k = \lceil
e^{\diffp_2/3}\rceil, r_{\max} = 5)$ and $\PrivUnit(\cdot,
\gamma(0.99\cdot\diffp_1), p(0.01 \cdot \diffp_1))$ while varying $\diffp_1$.
We present results in Figure~\ref{fig:Complex}(b) for mechanisms
that satisfy $(\diffp_1,\diffp_2 = 10)$-separated DP where $\diffp_1 \in
\{100,500,2500, 10000 \}$.
We choose the centralized projection $\cliprad = 100$ and
noise $\sigma$ in the aggregation~\eqref{eq:server-agg} to
guarantee $(\diffp_{\textup{cent}} = 3, \delta = 10^{-9})$ differential
privacy after $T = 200$ rounds.
We plot the difference in accuracies between federated learning and our
private federated learning system with standard errors over 20 trials.

\begin{figure}[H]
  \centering
  \begin{tabular}{cc}
    \includegraphics[width = 0.49\columnwidth]{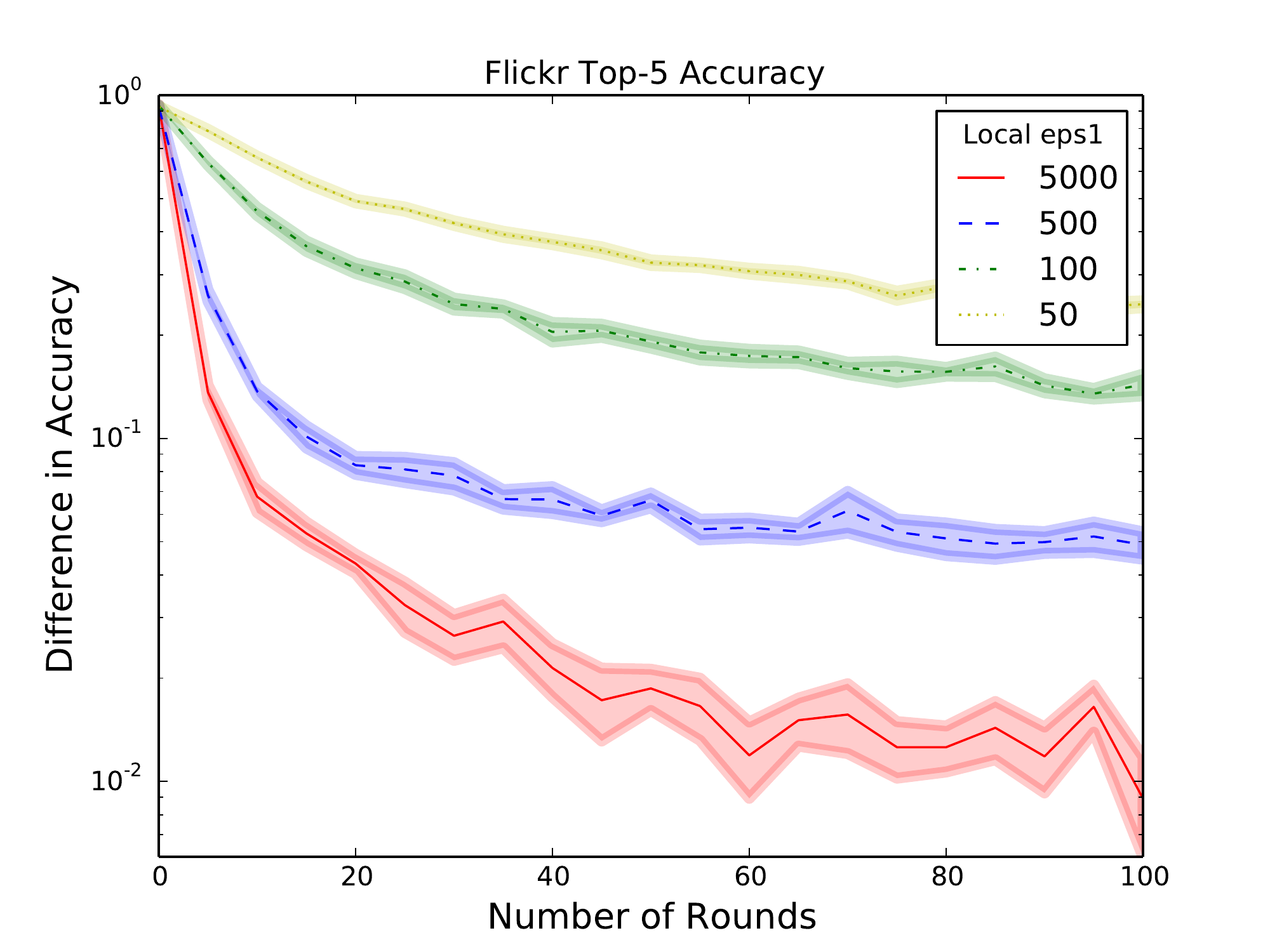} &
    \includegraphics[width=.49\columnwidth]{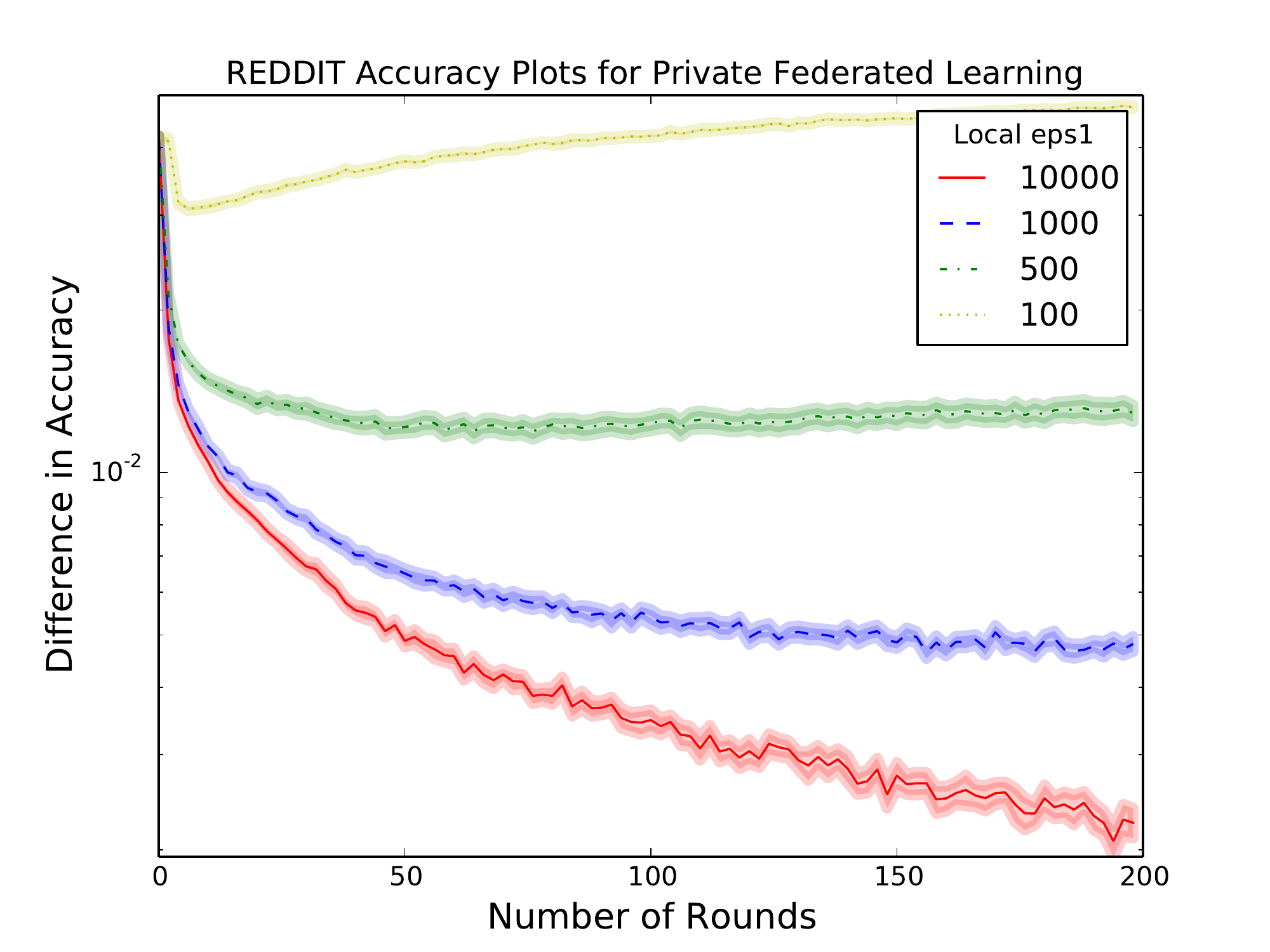} \\
    (a) & (b)
  \end{tabular}
  \caption{\label{fig:Complex}
    Accuracy plots for
    pretrained models comparing our private federated learning approach
    (labeled SDP with the corresponding $\diffp_1$ parameter) with various
    privacy parameter $\diffp$ and federated learning with clear model
    updates (labeled Clear).
    (a) Top 5 accuracy image classification: top
    100 classes from Flickr data with Resnet50v2 model pretrained on
    ImageNet.  (b) Next Word Prediction: Reddit data with pretrained LSTM on
    Wikipedia data with initial accuracy roughly $15.5\%$.}
\end{figure}


\section{Discussion and conclusion}

In this paper, we have described the analysis and implementation---with new,
minimax optimal privatization mechanisms---of a system for large-scale
distributed model fitting, or federated learning. In such systems, users may
prefer local privacy protections, though as this and previous
work~\cite{DuchiJoWa18,DuchiRo19} and make clear, providing small
$\diffp$-local differential privacy makes model-fitting extremely
challenging.  Thus, it is of substantial interest to understand what is
possible in large $\diffp$ regimes, and the corresponding types of privacy
such mechanisms provide; we have provided one such justification via
prior beliefs and reconstruction probabilities from oblivious adversaries.
We believe understanding appropriate privacy
barriers, which provide different types of protections at different levels,
will be important for the practical adoption of private procedures, and
we hope that the current paper provides impetus in this direction.

\section{Acknowledgements}\label{acks}

We thank Aaron Roth for helpful discussions on earlier versions of this
work. His comments helped shape the direction of the paper.

\clearpage

\bibliography{bib}

\begin{thebibliography}{71}
\providecommand{\natexlab}[1]{#1}
\providecommand{\url}[1]{\texttt{#1}}
\expandafter\ifx\csname urlstyle\endcsname\relax
  \providecommand{\doi}[1]{doi: #1}\else
  \providecommand{\doi}{doi: \begingroup \urlstyle{rm}\Url}\fi

\bibitem[Abadi et~al.(2016)Abadi, Chu, Goodfellow, McMahan, Mironov, Talwar,
  and Zhang]{Abadietal16}
M.~Abadi, A.~Chu, I.~Goodfellow, B.~McMahan, I.~Mironov, K.~Talwar, and
  L.~Zhang.
\newblock Deep learning with differential privacy.
\newblock In \emph{23rd ACM Conference on Computer and Communications Security
  (ACM CCS)}, pages 308--318, 2016.
\newblock URL \url{https://arxiv.org/abs/1607.00133}.

\bibitem[{Adelman-McCarthy \emph{et al.}}(2008)]{Adelman-McCarthyEtAl08}
J.~{Adelman-McCarthy \emph{et al.}}
\newblock The sixth data release of the {Sloan Digital Sky Survey}.
\newblock \emph{The Astrophysical Journal Supplement Series}, 175\penalty0
  (2):\penalty0 297--313, 2008.
\newblock \doi{10.1086/524984}.

\bibitem[{Apple Differential Privacy Team}(2017)]{ApplePrivacy17}
{Apple Differential Privacy Team}.
\newblock Learning with privacy at scale, 2017.
\newblock Available at
  \url{https://machinelearning.apple.com/2017/12/06/learning-with-privacy-at-scale.html}.

\bibitem[Ash(1990)]{Ash90}
R.~Ash.
\newblock \emph{Information Theory}.
\newblock Dover Books on Advanced Mathematics. Dover Publications, 1990.
\newblock ISBN 9780486665214.
\newblock URL \url{https://books.google.com/books?id=yZ1JZA6Wo6YC}.

\bibitem[Asi and Duchi(2019{\natexlab{a}})]{AsiDu19}
H.~Asi and J.~C. Duchi.
\newblock The importance of better models in stochastic optimization.
\newblock \emph{arXiv:1903.08619 [math.OC]}, 2019{\natexlab{a}}.

\bibitem[Asi and Duchi(2019{\natexlab{b}})]{AsiDu19siopt}
H.~Asi and J.~C. Duchi.
\newblock Stochastic (approximate) proximal point methods: Convergence,
  optimality, and adaptivity.
\newblock \emph{SIAM Journal on Optimization}, To Appear, 2019{\natexlab{b}}.
\newblock URL \url{https://arXiv.org/abs/1810.05633}.

\bibitem[Baldi et~al.(2014)Baldi, Sadowski, and Whiteson]{BaldiSaWh14}
P.~Baldi, P.~Sadowski, and D.~Whiteson.
\newblock Searching for exotic particles in high-energy physics with deep
  learning.
\newblock \emph{Nature Communications}, 5, July 2014.

\bibitem[Ball(1997)]{Ball97}
K.~Ball.
\newblock An elementary introduction to modern convex geometry.
\newblock In S.~Levy, editor, \emph{Flavors of Geometry}, pages 1--58. MSRI
  Publications, 1997.

\bibitem[Barber and Duchi(2014)]{BarberDu14a}
R.~F. Barber and J.~C. Duchi.
\newblock Privacy and statistical risk: Formalisms and minimax bounds.
\newblock \emph{arXiv:1412.4451 [math.ST]}, 2014.

\bibitem[Baumgartner(2017)]{Baumgartner17}
J.~Baumgartner.
\newblock Reddit comments, 2017.
\newblock URL \url{http://files.pushshift.io/reddit/comments/}.

\bibitem[Beck and Teboulle(2003)]{BeckTe03}
A.~Beck and M.~Teboulle.
\newblock Mirror descent and nonlinear projected subgradient methods for convex
  optimization.
\newblock \emph{Operations Research Letters}, 31:\penalty0 167--175, 2003.

\bibitem[Bertsekas(2011)]{Bertsekas11}
D.~P. Bertsekas.
\newblock Incremental proximal methods for large scale convex optimization.
\newblock \emph{Mathematical Programming, Series B}, 129:\penalty0 163--195,
  2011.

\bibitem[Bertsekas and Tsitsiklis(1989)]{BertsekasTs89}
D.~P. Bertsekas and J.~N. Tsitsiklis.
\newblock \emph{Parallel and Distributed Computation: Numerical Methods}.
\newblock Prentice-Hall, Inc., 1989.

\bibitem[Bonawitz et~al.(2017)Bonawitz, Ivanov, Kreuter, Marcedone, McMahan,
  Patel, Ramage, Segal, and Seth]{SMCFedLearn17}
K.~Bonawitz, V.~Ivanov, B.~Kreuter, A.~Marcedone, H.~B. McMahan, S.~Patel,
  D.~Ramage, A.~Segal, and K.~Seth.
\newblock Practical secure aggregation for privacy-preserving machine learning.
\newblock In \emph{Proceedings of the 2017 ACM SIGSAC Conference on Computer
  and Communications Security}, pages 1175--1191, New York, NY, USA, 2017. ACM.
\newblock URL \url{http://doi.acm.org/10.1145/3133956.3133982}.

\bibitem[Bottou and Bousquet(2007)]{BottouBo07}
L.~Bottou and O.~Bousquet.
\newblock The tradeoffs of large scale learning.
\newblock In \emph{Advances in Neural Information Processing Systems 20}, 2007.

\bibitem[Boucheron et~al.(2013)Boucheron, Lugosi, and Massart]{BoucheronLuMa13}
S.~Boucheron, G.~Lugosi, and P.~Massart.
\newblock \emph{Concentration Inequalities: a Nonasymptotic Theory of
  Independence}.
\newblock Oxford University Press, 2013.

\bibitem[Boyd et~al.(2011)Boyd, Parikh, Chu, Peleato, and
  Eckstein]{BoydPaChPeEc11}
S.~Boyd, N.~Parikh, E.~Chu, B.~Peleato, and J.~Eckstein.
\newblock Distributed optimization and statistical learning via the alternating
  direction method of multipliers.
\newblock \emph{Foundations and Trends in Machine Learning}, 3\penalty0 (1),
  2011.

\bibitem[Bun and Steinke(2016)]{BunSt16}
M.~Bun and T.~Steinke.
\newblock Concentrated differential privacy: Simplifications, extensions, and
  lower bounds.
\newblock In \emph{Theory of Cryptography Conference (TCC)}, pages 635--658,
  2016.

\bibitem[Chaudhuri et~al.(2011)Chaudhuri, Monteleoni, and
  Sarwate]{ChaudhuriMoSa11}
K.~Chaudhuri, C.~Monteleoni, and A.~D. Sarwate.
\newblock Differentially private empirical risk minimization.
\newblock \emph{Journal of Machine Learning Research}, 12:\penalty0 1069--1109,
  2011.

\bibitem[CifarTutorial()]{CifarTutorial18}
CifarTutorial.
\newblock Advanced convolutional neural networks.
\newblock \url{https://www.tensorflow.org/tutorials/images/deep_cnn}, 2018.

\bibitem[Clauset et~al.(2009)Clauset, Shalizi, and Newman]{ClausetShNe09}
A.~Clauset, C.~Shalizi, and M.~E.~J. Newman.
\newblock Power-law distributions in empirical data.
\newblock \emph{SIAM Review}, 51\penalty0 (4):\penalty0 661--703, 2009.

\bibitem[Dajani et~al.(2017)Dajani, Lauger, Singer, Kifer, Reiter,
  Machanavajjhala, Garfinkel1, Dahl, Graham, Karwa, Kim, Leclerc, Schmutte,
  Sexton, Vilhuber, and Abowd]{DajaniLaSiKiReMaGaDaGrKaKiLeScSeViAb17}
A.~N. Dajani, A.~D. Lauger, P.~E. Singer, D.~Kifer, J.~P. Reiter,
  A.~Machanavajjhala, S.~L. Garfinkel1, S.~A. Dahl, M.~Graham, V.~Karwa,
  H.~Kim, P.~Leclerc, I.~M. Schmutte, W.~N. Sexton, L.~Vilhuber, and J.~M.
  Abowd.
\newblock The modernization of statistical disclosure limitation at the {U.S.}
  {C}ensus bureau.
\newblock Available online at
  \url{https://www2.census.gov/cac/sac/meetings/2017-09/statistical-disclosure-limitation.pdf},
  2017.

\bibitem[Davis and Drusvyatskiy(2018)]{DavisDr18a}
D.~Davis and D.~Drusvyatskiy.
\newblock Stochastic model-based minimization of weakly convex functions.
\newblock \emph{arXiv:1803.06523 [math.OC]}, 2018.

\bibitem[Dean et~al.(2012)Dean, Corrado, Monga, Chen, Devin, Le, Mao, Ranzato,
  Senior, Tucker, Yang, and Ng]{DeanCoMoChDeMaRaSeTuYaNg12}
J.~Dean, G.~S. Corrado, R.~Monga, K.~Chen, M.~Devin, Q.~V. Le, M.~Z. Mao,
  M.~Ranzato, A.~Senior, P.~Tucker, K.~Yang, and A.~Y. Ng.
\newblock Large scale distributed deep networks.
\newblock In \emph{Advances in Neural Information Processing Systems 25}, 2012.

\bibitem[Deng et~al.(2009)Deng, Dong, Socher, Li, Li, and
  Fei-Fei]{DengDoSoLiLiFe09}
J.~Deng, W.~Dong, R.~Socher, L.~Li, K.~Li, and L.~Fei-Fei.
\newblock Image{N}et: a large-scale hierarchical image database.
\newblock In \emph{Proceedings of the IEEE Conference on Computer Vision and
  Pattern Recognition}, 2009.

\bibitem[Duchi and Rogers(2019)]{DuchiRo19}
J.~C. Duchi and R.~Rogers.
\newblock Lower bounds for locally private estimation via communication
  complexity.
\newblock In \emph{Proceedings of the Thirty Second Annual Conference on
  Computational Learning Theory}, 2019.
\newblock URL \url{https://arXiv.org/abs/1902.00582}.

\bibitem[Duchi and Ruan(2018)]{DuchiRu18b}
J.~C. Duchi and F.~Ruan.
\newblock The right complexity measure in locally private estimation: It is not
  the {F}isher information.
\newblock \emph{arXiv:1806.05756 [stat.TH]}, 2018.

\bibitem[Duchi and Ruan(2019)]{DuchiRu19}
J.~C. Duchi and F.~Ruan.
\newblock Asymptotic optimality in stochastic optimization.
\newblock \emph{Annals of Statistics}, To Appear, 2019.

\bibitem[Duchi et~al.(2011)Duchi, Hazan, and Singer]{DuchiHaSi11}
J.~C. Duchi, E.~Hazan, and Y.~Singer.
\newblock Adaptive subgradient methods for online learning and stochastic
  optimization.
\newblock \emph{Journal of Machine Learning Research}, 12:\penalty0 2121--2159,
  2011.

\bibitem[Duchi et~al.(2013)Duchi, Jordan, and Wainwright]{DuchiJoWa13_focs}
J.~C. Duchi, M.~I. Jordan, and M.~J. Wainwright.
\newblock Local privacy and statistical minimax rates.
\newblock In \emph{54th Annual Symposium on Foundations of Computer Science},
  pages 429--438, 2013.

\bibitem[Duchi et~al.(2018)Duchi, Jordan, and Wainwright]{DuchiJoWa18}
J.~C. Duchi, M.~I. Jordan, and M.~J. Wainwright.
\newblock Minimax optimal procedures for locally private estimation (with
  discussion).
\newblock \emph{Journal of the American Statistical Association}, 113\penalty0
  (521):\penalty0 182--215, 2018.

\bibitem[Dwork and Roth(2014)]{DworkRo14}
C.~Dwork and A.~Roth.
\newblock The algorithmic foundations of differential privacy.
\newblock \emph{Foundations and Trends in Theoretical Computer Science},
  9\penalty0 (3 \& 4):\penalty0 211--407, 2014.
\newblock \doi{10.1561/0400000042}.
\newblock URL \url{http://dx.doi.org/10.1561/0400000042}.

\bibitem[Dwork and Rothblum(2016)]{DworkRo16}
C.~Dwork and G.~Rothblum.
\newblock Concentrated differential privacy.
\newblock \emph{arXiv:1603.01887 [cs.DS]}, 2016.

\bibitem[Dwork et~al.(2006{\natexlab{a}})Dwork, Kenthapadi, McSherry, Mironov,
  and Naor]{DworkKeMcMiNa06}
C.~Dwork, K.~Kenthapadi, F.~McSherry, I.~Mironov, and M.~Naor.
\newblock Our data, ourselves: Privacy via distributed noise generation.
\newblock In \emph{Advances in Cryptology (EUROCRYPT 2006)},
  2006{\natexlab{a}}.

\bibitem[Dwork et~al.(2006{\natexlab{b}})Dwork, McSherry, Nissim, and
  Smith]{DworkMcNiSm06}
C.~Dwork, F.~McSherry, K.~Nissim, and A.~Smith.
\newblock Calibrating noise to sensitivity in private data analysis.
\newblock In \emph{Proceedings of the Third Theory of Cryptography Conference},
  pages 265--284, 2006{\natexlab{b}}.

\bibitem[Evfimievski et~al.(2003)Evfimievski, Gehrke, and
  Srikant]{EvfimievskiGeSr03}
A.~V. Evfimievski, J.~Gehrke, and R.~Srikant.
\newblock Limiting privacy breaches in privacy preserving data mining.
\newblock In \emph{Proceedings of the Twenty-Second Symposium on Principles of
  Database Systems}, pages 211--222, 2003.

\bibitem[Fredrikson et~al.(2015)Fredrikson, Jha, and Ristenpart]{FJR15}
M.~Fredrikson, S.~Jha, and T.~Ristenpart.
\newblock Model inversion attacks that exploit confidence information and basic
  countermeasures.
\newblock In \emph{Proceedings of the 22Nd ACM SIGSAC Conference on Computer
  and Communications Security}, pages 1322--1333, New York, NY, USA, 2015. ACM.
\newblock \doi{10.1145/2810103.2813677}.
\newblock URL \url{http://doi.acm.org/10.1145/2810103.2813677}.

\bibitem[Geng and Viswanath(2016)]{GengVi16}
Q.~Geng and P.~Viswanath.
\newblock The optimal noise-adding mechanism in differential privacy.
\newblock \emph{IEEE Transactions on Information Theory}, 62\penalty0
  (2):\penalty0 925--951, 2016.

\bibitem[Hardt and Talwar(2010)]{HardtTa10}
M.~Hardt and K.~Talwar.
\newblock On the geometry of differential privacy.
\newblock In \emph{Proceedings of the Forty-Second Annual ACM Symposium on the
  Theory of Computing}, pages 705--714, 2010.
\newblock URL \url{http://arxiv.org/abs/0907.3754}.

\bibitem[Hastie et~al.(2009)Hastie, Tibshirani, and Friedman]{HastieTiFr09}
T.~Hastie, R.~Tibshirani, and J.~Friedman.
\newblock \emph{The Elements of Statistical Learning}.
\newblock Springer, second edition, 2009.

\bibitem[He et~al.(2016)He, Zhang, Ren, and Sun]{HeZhReSu16b}
K.~He, X.~Zhang, S.~Ren, and J.~Sun.
\newblock Identity mappings in deep residual networks.
\newblock In \emph{European Conference on Computer Vision}, pages 630--645,
  2016.
\newblock ISBN 978-3-319-46493-0.

\bibitem[Kallenberg(1997)]{Kallenberg97}
O.~Kallenberg.
\newblock \emph{Foundations of Modern Probability}.
\newblock Springer, 1997.

\bibitem[Karampatziakis and Langford(2011)]{KarampatziakisLa11}
N.~Karampatziakis and J.~Langford.
\newblock Online importance weight aware updates.
\newblock In \emph{Proceedings of the 27th Conference on Uncertainty in
  Artificial Intelligence}, 2011.

\bibitem[Kazarinoff(1961)]{Kazarinoff}
N.~D. Kazarinoff.
\newblock \emph{Geometric Inequalities}.
\newblock Mathematical Association of America, 1961.
\newblock \doi{10.5948/UPO9780883859223}.

\bibitem[Krizhevsky(2009)]{Krizhevsky09}
A.~Krizhevsky.
\newblock Learning multiple layers of features from tiny images.
\newblock Technical report, University of Toronto, 2009.
\newblock URL \url{https://www.cs.toronto.edu/~kriz/cifar.html}.

\bibitem[Kulis and Bartlett(2010)]{KulisBa10}
B.~Kulis and P.~Bartlett.
\newblock Implicit online learning.
\newblock In \emph{Proceedings of the 27th International Conference on Machine
  Learning}, 2010.

\bibitem[{Le Cam} and Yang(2000)]{LeCamYa00}
L.~{Le Cam} and G.~L. Yang.
\newblock \emph{Asymptotics in Statistics: Some Basic Concepts}.
\newblock Springer, 2000.

\bibitem[LeCun et~al.(1998)LeCun, Bottou, Bengio, and Haffner]{LeCunBoBeHa98}
Y.~LeCun, L.~Bottou, Y.~Bengio, and P.~Haffner.
\newblock Gradient-based learning applied to document recognition.
\newblock In \emph{Advances in Neural Information Processing Systems 11}, 1998.

\bibitem[LeCun et~al.(2015)LeCun, Bengio, and Hinton]{LeCunBeHi15}
Y.~LeCun, Y.~Bengio, and G.~Hinton.
\newblock Deep learning.
\newblock \emph{Nature}, 521\penalty0 (7553):\penalty0 436--444, 2015.

\bibitem[Lee(2018)]{TensorNets}
T.~Lee.
\newblock Tensornets.
\newblock \url{https://github.com/taehoonlee/tensornets}, 2018.

\bibitem[Loper and Bird(2002)]{NLTK}
E.~Loper and S.~Bird.
\newblock {NLTK}: The natural language toolkit.
\newblock In \emph{Proceedings of the ACL-02 Workshop on Effective Tools and
  Methodologies for Teaching Natural Language Processing and Computational
  Linguistics}, pages 63--70, Stroudsburg, PA, USA, 2002. Association for
  Computational Linguistics.
\newblock \doi{10.3115/1118108.1118117}.
\newblock URL \url{https://doi.org/10.3115/1118108.1118117}.

\bibitem[Mallat(2008)]{Mallat08}
S.~Mallat.
\newblock \emph{A Wavelet Tour of Signal Processing: The Sparse Way (Third
  Edition)}.
\newblock Academic Press, 2008.

\bibitem[McMahan et~al.(2017{\natexlab{a}})McMahan, Moore, Ramage, Hampson, and
  y~Arcas]{FedLearn17}
H.~B. McMahan, E.~Moore, D.~Ramage, S.~Hampson, and B.~A. y~Arcas.
\newblock Communication-efficient learning of deep networks from decentralized
  data.
\newblock In \emph{Proceedings of the 20th International Conference on
  Artificial Intelligence and Statistics (AISTATS)}, 2017{\natexlab{a}}.
\newblock URL \url{http://arxiv.org/abs/1602.05629}.

\bibitem[McMahan et~al.(2017{\natexlab{b}})McMahan, Moore, Ramage, Hampson, and
  y~Arcas]{McMahanMoRaHaAr17}
H.~B. McMahan, E.~Moore, D.~Ramage, S.~Hampson, and B.~A. y~Arcas.
\newblock Communication-efficient learning of deep networks from decentralized
  data.
\newblock In \emph{Proceedings of the 20st International Conference on
  Artificial Intelligence and Statistics}, 2017{\natexlab{b}}.

\bibitem[McMahan et~al.(2017{\natexlab{c}})McMahan, Ramage, Talwar, and
  Zhang]{DPFedLearn17}
H.~B. McMahan, D.~Ramage, K.~Talwar, and L.~Zhang.
\newblock Learning differentially private language models without losing
  accuracy.
\newblock \emph{arXiv:1710.06963 [cs.LG]}, 2017{\natexlab{c}}.
\newblock URL \url{http://arxiv.org/abs/1710.06963}.

\bibitem[Melis et~al.(2018)Melis, Song, Cristofaro, and Shmatikov]{MSCS18}
L.~Melis, C.~Song, E.~D. Cristofaro, and V.~Shmatikov.
\newblock Inference attacks against collaborative learning.
\newblock \emph{arXiv/1805.04049 [cs.CR]}, 2018.
\newblock URL \url{http://arxiv.org/abs/1805.04049}.

\bibitem[Micciancio and Voulgaris(2010)]{MicciancioVo10}
D.~Micciancio and P.~Voulgaris.
\newblock Faster exponential time algorithms for the shortest vector problem.
\newblock In \emph{Proceedings of the Twenty-First ACM-SIAM Symposium on
  Discrete Algorithms (SODA)}, 2010.

\bibitem[Mikolov et~al.(2013)Mikolov, Sutskever, Chen, Corrado, and
  Dean]{MikolovSuChCoDe13}
T.~Mikolov, I.~Sutskever, K.~Chen, G.~Corrado, and J.~Dean.
\newblock Distributed representations of words and phrases and their
  compositionality.
\newblock In \emph{Advances in Neural Information Processing Systems 26}, 2013.

\bibitem[Mironov(2017)]{Mironov17}
I.~Mironov.
\newblock R\'{e}nyi differential privacy.
\newblock In \emph{30th IEEE Computer Security Foundations Symposium (CSF)},
  pages 263--275, 2017.

\bibitem[Nemirovski et~al.(2009)Nemirovski, Juditsky, Lan, and
  Shapiro]{NemirovskiJuLaSh09}
A.~Nemirovski, A.~Juditsky, G.~Lan, and A.~Shapiro.
\newblock Robust stochastic approximation approach to stochastic programming.
\newblock \emph{SIAM Journal on Optimization}, 19\penalty0 (4):\penalty0
  1574--1609, 2009.

\bibitem[Polyak and Juditsky(1992)]{PolyakJu92}
B.~T. Polyak and A.~B. Juditsky.
\newblock Acceleration of stochastic approximation by averaging.
\newblock \emph{SIAM Journal on Control and Optimization}, 30\penalty0
  (4):\penalty0 838--855, 1992.

\bibitem[Press et~al.(1992)Press, Flannery, Teukolsky, and
  Vetterling]{PressFlTeVe92}
W.~H. Press, B.~P. Flannery, S.~A. Teukolsky, and W.~T. Vetterling.
\newblock \emph{Numerical Recipes in C: The Art of Scientific Computing, Second
  Edition}.
\newblock Cambridge University Press, 1992.

\bibitem[Schmidhuber(2015)]{Schmidhuber15}
J.~Schmidhuber.
\newblock Deep learning in neural networks: An overview.
\newblock \emph{Neural networks}, 61:\penalty0 85--117, 2015.

\bibitem[TensorFlowTutorial()]{TensorFlowTutorial18}
TensorFlowTutorial.
\newblock Build a convolutional neural network using estimators.
\newblock \url{https://www.tensorflow.org/tutorials/estimators/cnn}, 2018.

\bibitem[Thomee et~al.(2016)Thomee, Shamma, Friedland, Elizalde, Ni, Poland,
  Borth, and Li]{ThomeeShFrElNiPoBoLi16}
B.~Thomee, D.~Shamma, G.~Friedland, B.~Elizalde, K.~Ni, D.~Poland, D.~Borth,
  and L.~Li.
\newblock Yahoo {F}lickr {C}reative {C}ommons 100{M}: The new data in
  multimedia research.
\newblock \emph{Communications of the ACM}, 2\penalty0 (59):\penalty0 64--73,
  2016.

\bibitem[van~der Vaart(1998)]{VanDerVaart98}
A.~W. van~der Vaart.
\newblock \emph{Asymptotic Statistics}.
\newblock Cambridge Series in Statistical and Probabilistic Mathematics.
  Cambridge University Press, 1998.
\newblock ISBN 0-521-49603-9.

\bibitem[Warner(1965)]{Warner65}
S.~Warner.
\newblock Randomized response: a survey technique for eliminating evasive
  answer bias.
\newblock \emph{Journal of the American Statistical Association}, 60\penalty0
  (309):\penalty0 63--69, 1965.

\bibitem[Wasserman and Zhou(2010)]{WassermanZh10}
L.~Wasserman and S.~Zhou.
\newblock A statistical framework for differential privacy.
\newblock \emph{Journal of the American Statistical Association}, 105\penalty0
  (489):\penalty0 375--389, 2010.

\bibitem[Wikipedia()]{wikipedia}
Wikipedia.
\newblock Wikimedia downloads, 2016.
\newblock URL \url{https://dumps.wikimedia.org}.
\newblock Accessed: 10-01-2016.

\bibitem[Zhang(2004)]{Zhang04}
T.~Zhang.
\newblock Solving large scale linear prediction problems using stochastic
  gradient descent algorithms.
\newblock In \emph{Proceedings of the Twenty-First International Conference on
  Machine Learning}, 2004.

\bibitem[Zinkevich(2003)]{Zinkevich03}
M.~Zinkevich.
\newblock Online convex programming and generalized infinitesimal gradient
  ascent.
\newblock In \emph{Proceedings of the Twentieth International Conference on
  Machine Learning}, 2003.

\end{thebibliography}
\bibliographystyle{abbrvnat}
 
 \newpage
\appendix


\section{Technical proofs}
\label{app:proof}

\subsection{Proof of Lemma~\ref{lemma:precision-recall-bounds}}
\label{sec:proof-precision-recall-bounds}

We prove each result in turn. We begin with the precision bound.
Fix $p > 0$, and
assume that $p v^T \ones \ge v^T \E[X]$. Then
using that $\var(v_j (X_j - \E[X_j]))
\le v_j \frac{m}{j}$, and is 0 for $j < m$.
By Bernstein's inequality, we then have
\begin{align}
  \P(\precision(v, X) \ge p)
  & = \P(v^T (X - \E[X]) \ge p v^T \ones - v^T \E[X])
  \nonumber \\
  & \le \exp\left(-\frac{(p v^T \ones - v^T \E[X])^2}{
    2 \sum_{j > m} v_j \var(X_j)
    + \frac{2}{3} (p v^T \ones - v^T \E[X])}\right)
  \nonumber \\
  & \le \exp\left(-
  \min\left\{\frac{(p v^T \ones - v^T \E[X])^2}{
    4 \sum_{j > m} v_j \var(X_j)},
  \frac{3}{4} (p v^T \ones - v^T \E[X])\right\}
  \right)
  \label{eqn:exponential-precision}
\end{align}
For $v^T \ones \ge \gamma m$, assuming that
$p \ge \frac{1}{\gamma}$, we have
\begin{equation*}
  p v^T \ones - v^T \E[X]
  \ge p \gamma m - \sum_{j = 1}^{\gamma m} \E[X_j]
  = p \gamma m - m - \sum_{j = m+1}^{\gamma m}
  \frac{m}{j}
  \ge (p \gamma - 1) m -
  m \int_m^{\gamma m} \frac{1}{t} dt
  = (p \gamma - 1 - \log \gamma) m.
\end{equation*}
For the first term in the exponent~\eqref{eqn:exponential-precision},
the ratio again is maximized by
$v = (\ones_{\gamma m}, \zeros_{d - \gamma m})$, and we have
\begin{equation*}
  \sum_{j = m+1}^d v_j \var(X_j)
  \le \sum_{j > m}^{\gamma m} \frac{m}{j}
  \le m \int_{m}^{\gamma m} \frac{1}{t} dt
  = m \log \gamma.
\end{equation*}
Substituting these bounds in Eq.~\eqref{eqn:exponential-precision},
we have
\begin{equation*}
  \P(\precision(v, X) \ge p)
  \le \exp\left(-
  \min\left\{\frac{(p \gamma - 1 - \log \gamma)^2 m}{4 \log \gamma},
  \frac{3}{4} (p \gamma - 1 - \log \gamma) m \right\}\right).
\end{equation*}

For the recall bounds, we perform a similar derivation, assuming
$r \le \half$ so that $(1 - r) \ge r$. We temporarily
assume $(r \ones - v)^T \E[X] \ge 0$; we shall see that this
holds.
By Bernstein's inequality, we have
\begin{align}
  \nonumber
  \P(\recall(v, X) \ge r)
  & = \P(v^T X \ge r \ones^T X)
  = \P((v - r \ones)^T (X - \E[X])
  \ge (r \ones - v)^T \E[X]) \\
  & \le \exp\left(-\frac{((r \ones - v)^T \E[X])^2}{
    2 \sum_{j > m} (r - v_j)^2 \var(X_j)
    + \frac{2}{3}
    (r \ones - v)^T \E[X]}\right) \nonumber \\
  & \le \exp\left(-\min\left\{
  \frac{((r \ones - v)^T \E[X])^2}{
    4 \sum_{j > m} (r - v_j)^2 \var(X_j)},
  \frac{3}{4} (r \ones - v)^T \E[X]\right\}\right).
  \label{eqn:exponential-recall}
\end{align}
We consider each term in the minimum~\eqref{eqn:exponential-recall}
in turn.
When $\sum_j v_j \le \gamma m$, we have
\begin{equation*}
  \sum_{j > m} (r - v_j) \var(X_j)^2
  \le m(1 - r)^2 \int_m^d \frac{1}{t} dt
  = m (1 - r)^2 \log \frac{d}{m},
\end{equation*}
while
\begin{align*}
  (r \ones - v)^T \E[X]
  & \ge r \ones^T \E[X] -
  \sum_{j = 1}^{\gamma m} \E[X]
  = r m + r \sum_{j = m+1}^d \frac{m}{j}
  - m - \sum_{j = m+1}^{\gamma m} \E[X_j] \\
  & \ge r m + r m \int_{m + 1}^d \frac{1}{t} dt
  - m - m \int_m^{\gamma m} \frac{1}{t} dt
  = m \left[r \left(1 +  \log\frac{d}{m + 1}\right)
    - 1 - \log \gamma\right].
\end{align*}
Substituting into inequality~\eqref{eqn:exponential-recall}
gives the second result of the lemma.

\subsection{Proof of Theorem~\ref{theorem:privacy-infinity}}
\label{appendix:proof-privacy-infinity}

Let $u \in \{-1, +1 \}^d$ and
$U \sim \uniform(\{-1,+1 \}^d)$.  The vector $V \in \{-1, 1\}^d$
sampled as in \eqref{eqn:w-flip-mechanismINFTY}, has p.m.f.\
\begin{equation*}
  p(v \mid u)
  \propto \begin{cases} 1 / \P(\<U, u\> > \kappa) & \mbox{if~}
    \<v, u \> > \kappa \\
    1 / \P(\<U, u\> < \kappa) & \mbox{if~} \<v, u\> \leq \kappa.
  \end{cases}
\end{equation*}
The event that $\<U, u \> = \kappa$ when $\frac{d+\kappa+1}{2} \in \Z$
implies that $U$ and $u$ match in exactly $\frac{d+\kappa+1}{2}$
coordinates; the number of such matches is
$\choose{d}{(d+\kappa+1)/2}$. Computing the binomial sum, we have
\begin{align*}
  \P(\< U, u \> > \kappa) =  \frac{1}{2^d}
  \sum_{\ell = \lceil \frac{d+\kappa+1}{2} \rceil}^d
  \choose{d}{\ell}
  \quad \mbox{and}
  \quad
  \P(\< U, u \> \leq \kappa) =  \frac{1}{2^d}
  \sum_{\ell = 0 }^{\lceil\frac{d+\kappa+1}{2}\rceil -1} \choose{d}{\ell}.
\end{align*}
As $\P(\< U ,u\> > \kappa)$ is decreasing in $\kappa$ for any $u, u' \in
\{-1,+1 \}^d$ and $v \in \{-1, 1\}^d$ we have
\begin{equation*}
  \frac{p(v \mid u) }{p(v \mid u')} \leq \frac{p_0}{1 - p_0}
  \cdot \frac{\P(\< U, u' \> \leq \kappa)}{\P(\< U, u \> > \kappa)}
  = e^{\diffp_0}
  \cdot
  \frac{\sum_{\ell =0}^{d \tau-1 }\choose{d}{\ell}}{
    \sum_{\ell = d\tau }^d \choose{d}{\ell}},
\end{equation*}
where $\tau = \ceil{(d+\kappa)/2}$.
Bounding this by $e^{\diffp+\diffp_0}$ gives the result.

\subsection{Proof of Corollary~\ref{corollary:ell_infty}}

Using Theorem~\ref{theorem:privacy-infinity}, we seek to bound the
quantity~\eqref{eq:sufficient_kappaGEN} for various $\kappa$ values.  We
first analyze the case when $\kappa \leq \sqrt{3/2d+ 1}$.  We use the
following claim to bound each term in the summation in this case.
\begin{claim}[See Problem 1 in \cite{Kazarinoff}]
For even $d \geq 2$, we have 
$$
\choose{d}{d/2} \leq \frac{2^{d+1/2}}{\sqrt{3 d + 2}}
$$
\label{claim:middle_binomial}
\end{claim}
Thus, when $\kappa < \sqrt{\tfrac{3d+2}{2}}$,
\begin{align*}
  \log\left( \sum_{\ell = 0}^{d \tau -1 } \choose{d}{\ell} \right) - \log\left( \sum_{\ell = d \tau}^{d} \choose{d}{\ell} \right)
& = 
\log\left( 1/2 + \frac{1}{2^d} \cdot \sum_{\ell = d/2}^{d\tau-1} \choose{d}{\ell} \right) - \log\left(1/2 -\frac{1}{2^d} \sum_{\ell = d/2}^{d \tau-1 } \choose{d}{\ell}  \right)
\\
&\leq  \log\left(1 +\kappa\cdot \left(\sqrt{\frac{2}{3d +2}}  \right) \right) - \log\left(1 - \kappa\cdot \left( \sqrt{\frac{2}{3d +2} } \right) \right).
\end{align*}
Hence, to ensure $\diffp$-differential privacy, it suffices to have
$$
\diffp \geq \log\left( 1 +\kappa\cdot \sqrt{\frac{2}{3d +2} } \right) - \log\left(1 - \kappa\cdot  \sqrt{\frac{2}{3d +2} } \right),
$$ so Eq.~\eqref{eq:sufficient_kappa1} follows.

We now consider the case where $\epsilon = \Omega(\log(d))$. We use the
following claim.
\begin{claim}[Lemma 4.7.2 in \citep{Ash90}]
Let $Z \sim \Bin{d}{1/2}$. Then for $0< \lambda < 1$, we have 
$$
\P(Z \geq d\lambda)
\geq \frac{1}{\sqrt{8 d\lambda(1-\lambda)}} \exp\left( -d \dkl{\lambda}{1/2}\right)
$$
\end{claim}
\noindent We then use this to obtain the following bound:
\begin{align*}
  \log\left( \sum_{\ell = 0}^{d \tau -1} \choose{d}{\ell} \right) - \log\left( \sum_{\ell =d \tau}^{d} \choose{d}{\ell} \right)
& \leq \half \log\left(8 d\tau(1-\tau) \right) + d \dkl{\tau}{1/2}
\end{align*}
To ensure $\diffp$-differential privacy, it is thus sufficient, for $\tau
\defeq \frac{\lceil \frac{d+\kappa+1}{2} \rceil }{d}$, to have
\begin{align*}
  \diffp &\geq \frac{1}{2} \log\left(8\cdot d \tau(1-\tau)\right) + d \cdot \dkl{\tau}{1/2},
\end{align*}
which implies the final claim of the corollary.

\subsection{Proof of Lemma~\ref{lemma:asymptotic-variance-private}}
\label{sec:proof-asymptotic-variance-private}

For shorthand, define the radius $R =
\ltwo{\nabla\loss( \theta\opt; X)}$, $\gamma = \gamma(\diffp_1)$,
let $U = \nabla \loss(\theta\opt; X) / \ltwo{\nabla \loss(\theta\opt; X)}$,
and write
\begin{align*}
  Z_1  = \PrivUnit\left(U  ; \gamma, p = 1/2\right)
  \qquad
  Z_2  = \ScalarDP(R, \diffp_2; k = \lceil e^{\diffp_2/3} \rceil, r_{\max})
\end{align*}
so that $Z = Z_1 Z_2$.  Using that $\E[(Z_2 - R)^2 \mid R] \le
O(r_{\max}^2 e^{-2 \diffp_2 / 3})$ by
Lemma~\ref{lemma:one-dim-absolute-error},
we have
\begin{align}
  \nonumber
  \E[Z(\theta\opt;X) Z(\theta\opt;X)^\top]
  & = \E\left[ \E\left[Z_1 Z_1^\top \mid U\right]
    \cdot \E\left[ Z_2^2\mid R \right] \right] \\
  & = \E\left[ \E\left[Z_1 Z_1^\top \mid U\right]
    \cdot \E\left[ (Z_2-R)^2\mid R \right] \right]
  + \E\left[R^2 \cdot \E[ Z_1 Z_1^\top \mid U] \right]
  \nonumber \\
  & \preceq
  O(1) r_{\max}^2 e^{-2\diffp_2/3} \cdot
  \E\left[ \E\left[Z_1 Z_1^\top \mid U\right]\right]
  + \E\left[R^2 \cdot \E[ Z_1 Z_1^\top \mid U] \right].
  \label{eqn:bound-zztop}
\end{align}
We now focus on the term $\E[ Z_1 Z_1^\top \mid U]$.
Recall that $V$
is uniform on $\{v \in \sphere^{d-1} : \<v, u\> \ge \gamma\}$ with
probability $\half$ and uniform on the
complement $\{v \in \sphere^{d-1} : \<v, u\> < \gamma\}$ otherwise
(Eq.~\eqref{eqn:w-flip-mechanism}).
Then for $W \sim \uniform(\sphere^{d-1})$, we obtain
$\E[VV^\top \mid u]
= \half \E[WW^\top \mid \<W, u\> \ge \gamma]
+ \half \E[WW^\top \mid \<W, u\> < \gamma]$, where
\begin{equation*}
  \E[WW^\top \mid \<W, u\> \ge \gamma]
  \preceq uu^\top + \frac{1 - \gamma^2}{d-1}
  (I_d - uu^\top)
  ~~~ \mbox{and} ~~~
  \E[WW^\top \mid \<W, u\> \le \gamma]
  \preceq uu^\top + \frac{1}{d} (I_d - uu^\top).
\end{equation*}
Both of these are in turn bounded by $uu^\top + (1/d) I_d$.
Using that the normalization $m$ defined in Eq.~\eqref{eqn:norm-of-W}
satisfies $m \gtrsim \min\{\diffp_1, \sqrt{\diffp_1}\} / \sqrt{d}$ by
Proposition~\ref{proposition:ltwo-utility},
we obtain
\begin{align*}
  \E[Z_1 Z_1^\top \mid U = u]
  & \preceq \frac{1}{m^2} uu^\top + \frac{1}{d m^2} I_d
  \preceq \frac{d}{\diffp_1 \wedge \diffp_1^2}
  uu^\top + \frac{1}{\diffp_1 \wedge \diffp_1^2} I_d.
\end{align*}

Substituting this bound into our earlier
inequality~\eqref{eqn:bound-zztop} and using that $R U = \nabla
\loss(\theta\opt; X)$, we obtain
\begin{align*}
  & \E[Z(\theta\opt;X) Z(\theta\opt;X)^\top]
  \preceq O(1) \cdot \frac{d r^2_{\max} e^{-2\diffp_2/3}}{\diffp_1
    \wedge \diffp_1^2 }
  \cdot \E[UU^\top + (1/d) I_d] \\
  & \qquad ~ 
  + O(1) \cdot \left(
  \frac{d}{\diffp_1 \wedge \diffp_1^2} \cdot
  \E\left[ \nabla\loss(\theta\opt; X) \nabla\loss(\theta\opt;X)^\top\right]
  + \frac{1}{\diffp_1 \wedge \diffp_1^2} \cdot
  \E\left[\ltwo{\nabla \loss(\theta\opt; X)}^2 I_d\right] \right).
\end{align*}
Noting that $\tr(\cov(W))
= \E[\ltwo{W}^2]$ for any random vector $W$ gives the lemma.

\section{Scalar sampling}
\label{sec:proofs-scalar-sampling}

\subsection{Proof of Lemma~\ref{lem:ScalarUtil}}
\label{sec:proof-scalar-utility}

That the mechanism is $\diffp$-differentially private
is immediate as randomized response is $\diffp$-differentially private.
To prove that $Z$ is
unbiased, we note that
\begin{equation*}
  \E[J] = \frac{k r}{r_{\max}}
  ~~ \mbox{and} ~~
  \E[Z \mid J] = \frac{r_{\max}}{k} J,
\end{equation*}
so that $\E[Z] = \frac{r_{\max}}{r} \E[J] = r$.
To develop the bounds on the variance of $Z$,
we 
use the standard decomposition of variance into conditional
variances, as
\begin{equation*}  
  \var(Z)
  = \E[(Z - r)^2] = \E\left[ \var\left[ Z \mid J \right] \right]
  + \var\left[ \E\left[ Z \mid J \right] \right]
  = \E[\var[Z \mid J]] + \frac{r_{\max}^2}{k^2} \var[J].
\end{equation*}
We have $\var[Z \mid J] = a^2 \cdot\var[\what{J} \mid J]$.  Further, we have 
\begin{equation*}
  \E[\what{J} ^2 \mid J] =
  \left( \frac{e^\diffp - 1}{e^\diffp + k} \right) \cdot  J^2
  + \left(\frac{1}{e^\diffp + k} \right) \cdot
  \sum_{j = 0}^k j^2
  ~~~ \mbox{and} ~~~
  \E[\what{J} \mid J] =  \left( \frac{e^\diffp - 1}{e^\diffp + k} \right) \cdot  J
  + \left(\frac{1}{e^\diffp + k} \right) \cdot
  \sum_{j = 0}^k j.
\end{equation*}
Combining the equalities and using that
$\var(Y) = \E[Y^2] - \E[Y]^2$, we have
\begin{align*}
  \var[\what{J} \mid J]
  & = \left(  \frac{e^\diffp - 1}{e^\diffp + k} \right)^2 \cdot \left( \left( \frac{e^\diffp + k}{e^\diffp - 1} - 1 \right) \cdot J^2 - \frac{k(k+1)}{e^\diffp - 1} \cdot J + \frac{k(k+1)(2k+1)(e^\diffp + k)}{6(e^\diffp - 1)^2} - \frac{k^2(k+1)^2}{4(e^\diffp - 1)^2} \right) .
\end{align*}
We then have the conditional variance
\begin{align*}
  \var[Z \mid J]
  & = \frac{r_{\max}^2}{(e^\diffp - 1)^2} \cdot
  \left( \frac{(k+1)(e^\diffp - 1)}{k^2} \cdot J^2
  - \frac{(k+1)(e^\diffp - 1)}{k} J + \frac{ (k+1)(2k+1)(e^\diffp + k)}{6k}
  -\frac{(k+1)^2}{4}\right).
\end{align*}
Now, note that
$J$ is a Bernoulli sample taking values in $\{\floor{kr /
  r_{\max}}, \ceil{kr / r_{\max}}\}$. Thus, recalling the notation
$\decimal(t) = \ceil{t} - t$ for the decimal part of a number, we have
$\var(J) = \decimal(\frac{kr}{r_{\max}}) (1 -
\decimal(\frac{kr}{r_{\max}}))$. Substituting all of these above and using
that $\E[J^2] = \var(J) + \E[J]^2 = \var(J) + \frac{k^2 r^2}{r_{\max}^2}$,
we obtain
\begin{align*}
  \lefteqn{\var(Z)} \\
  & = \frac{r_{\max}^2}{(e^\diffp - 1)^2}
  \cdot \bigg(\frac{(k + 1) (e^\diffp - 1)}{k^2}
  \left(\frac{k^2 r^2}{r_{\max}^2} + \decimal(kr / r_{\max})
  (1 - \decimal(kr / r_{\max}))\right) ~ \ldots \\
  & \qquad\qquad\qquad ~ 
  - (k + 1) (e^\diffp - 1) \frac{r}{r_{\max}}
  + \frac{(k+ 1)(2k + 1)(e^\diffp + k)}{6k} - \frac{(k + 1)^2}{4}
  \bigg) \\
  & \qquad ~
  + \frac{r_{\max}^2}{k^2} \decimal(kr / r_{\max}) (1 - \decimal(kr / r_{\max}))
  \\
  & \le \frac{k + 1}{e^\diffp - 1}
  \left[r^2
    + \frac{r_{\max}^2}{4 k^2}
    - r r_{\max} + \frac{(2k + 1)(e^\diffp + k)r_{\max}^2}{6k (e^\diffp - 1)}
    - \frac{k+1}{4 (e^\diffp - 1)}\right]
  + \frac{r_{\max}^2}{4k^2}
\end{align*}
where we have used that $p(1 - p) \le \frac{1}{4}$ for all $p \in [0, 1]$.
Ignoring the negative terms gives the result.

\subsection{Sampling scalars with relative error}
\label{sec:rel-error-onevar}

As our discussion in the introductory Section~\ref{sec:local-is-hard} shows,
to develop optimal learning procedures it is frequently important to know
when the problem is easy---observations are low variance---and for this,
releasing scalars with relative error can be important. Consequently, we
consider an alternative mechanism that first breaks the range $[0,
  r_{\max}]$ into intervals of increasing length based on a fixed accuracy
$\alpha > 0$, $k \in \N$, and $\nu > 1$, where we define the intervals
\begin{equation}
  \interval_0 = [0, \nu \alpha],
  ~~ \interval_i = [\nu^i \alpha, \nu^{i+1} \alpha]
  ~ \mbox{for~} i = 1, \ldots, {k-1}.
  \label{eq:intervals}
\end{equation}
The resulting mechanism works as follows: we determine the interval that $r$
belongs to, we randomly round $r$ to an endpoint of the interval (in an
unbiased way), then use randomized response to obtain a differentially
private quantity, which we then debias.  We formalize the algorithm in
Algorithm~\ref{alg:scalar-relative}.
\begin{algorithm}
\caption{Privatize the magnitude with relative error: $\ScalarRelDP$}
\label{alg:scalar-relative}
\begin{algorithmic}
  \Require Magnitude $r$, privacy parameter $\diffp > 0$, integer $k$,
  accuracy $\alpha > 0$, $\nu >1$, bound $r_{\max}$.
  \State $r \gets \min\{r,r_{\max} \} $
\State Form the intervals $\{\interval_0, \interval_1, \cdots, \interval_{k-1} \}$ given in \eqref{eq:intervals} and let $i^*$ be the index such that $r \in \interval_{i^*}$.
\State Sample $J \in \{ 0,1, \cdots, k\}$ such that 
\begin{equation*}
  J = \begin{cases} 0 & \mbox{w.p.} ~ \frac{\nu\alpha - r}{\nu \alpha}
    \\
    1 & \mbox{w.p.}~ \frac{r}{\nu \alpha} \end{cases}
  ~~ \mbox{if~} i^* = 0 \,
  ~~~ \mbox{and}~~~
  J = \begin{cases} i^* & \mbox{w.p.}~ \frac{\nu^{i^* + 1} \alpha - r}{
      \nu^{i^*}(\nu - 1) \alpha} \\
    i^* + 1 & \mbox{w.p.}~ \frac{r - \nu^{i^*} \alpha}{\nu^{i^*}(\nu - 1) \alpha}
  \end{cases}
  ~~ \mbox{if~} i^* \ge 1.
\end{equation*}
\State Use randomized response to obtain $\what{J}$
\begin{equation*}
  \what{J} \mid (J = i) = \begin{cases}
    i & \mbox{w.p.}~ \frac{e^\diffp}{e^\diffp + k} \\
    \mbox{uniform in~} \{0, \ldots, k\} \setminus i
    & \mbox{w.p.}~ \frac{k}{e^\diffp + k}. \end{cases}
\end{equation*}
\State Set $\tilde{J} = \nu^{\hat{J}} \cdot \1{\hat{J} \ge 1 }$
\State Debias $\tilde{J}$, by setting
\begin{equation*}
   Z = a \left(\tilde{J} - b\right) ~~ \mbox{for~} a = \alpha \cdot \left( \frac{e^\diffp + k}{e^\diffp - 1} \right) ~~ \mbox{and} ~~
  b =  \frac{1}{e^\diffp + k} \cdot \sum_{j=1}^k \nu^{j}.
  \end{equation*}
\Return $Z$
\end{algorithmic}
\end{algorithm}

As in the absolute erorr case, we can provide an upper bound on the
error of the mechanism \ScalarRelDP, though in this case---up to the
small accuracy $\alpha$---our error guarantee is relative.
\begin{lemma}
  \label{lemma:one-dim-relative-error}
  Fix $\alpha >0$, $k \in \N$, and $\nu > 1$. Then $Z = \ScalarRelDP(\cdot,
  \diffp; k,\alpha, \nu, r_{\max})$ is $\diffp$-differentially private and
  for $r < r_{\max}$, we have $\E[Z\mid r] = r$ and
  \begin{align*}
    \frac{\E[(Z - r)^2 \mid r ]}{(r \vee \alpha)^2}
      & \leq \frac{(k + 1)}{(e^\diffp -1 )}
      \nu^2 + \left( \frac{\nu^{2k}\cdot (e^\diffp + k)}{(e^\diffp -1)^2 }  \right)
      \left(\frac{1 - \nu^{-2k}}{1 - \nu^{-2}}\right) + (\nu- 1)^2.
  \end{align*}
\end{lemma}
\noindent
See Appendix~\ref{sec:proof-scalar-relative-utility} for a proof
of Lemma~\ref{lemma:one-dim-relative-error}.

We perform a few calculations with Lemma~\ref{lemma:one-dim-relative-error} when $k \leq e^{\diffp}$.
Let $\nu = 1 + \Delta$ for a $\Delta \in (0, 1)$ to be chosen,
and note that the choice
\begin{equation*}
  k = \ceil{\frac{\log\frac{r_{\max}}{\alpha}}{\log \nu}}
  \approx \frac{\log \frac{r_{\max}}{\alpha}}{\Delta}
\end{equation*}
is the smallest value of $k$ that
gives $\nu^{k} \alpha \ge r_{\max}$. Let $\alpha = r_{\max} \alpha_0$, so that $\nu^{2k} \approx \frac{1}{\alpha_0^2}$.  Hence, from
Lemma~\ref{lemma:one-dim-relative-error} we have the truncated relative error
bound
\begin{equation*}
  \frac{\E[(Z - r)^2 \mid r]}{(r \vee r_{\max} \alpha_0)^2}
  \lesssim \frac{k}{e^\diffp} + \frac{e^{-\diffp}}{\alpha_0^2 \cdot \Delta}
 + \Delta^2
  = e^{-\diffp}\left(\log \frac{1}{\alpha_0}
  + \frac{1}{\alpha_0^2} \right) \cdot \left(\frac{1}{\Delta} \right)+ \Delta^2
  \lesssim \frac{1}{e^\diffp \alpha_0^2 \Delta} + \Delta^2.
\end{equation*}
We assume the relative accuracy threshold $\alpha / r_{\max} = \alpha_0
\gtrsim e^{-\diffp / 2}$. Then setting $\Delta = \alpha_0^{-2/3} e^{-\diffp/3}$
gives
\begin{equation}
  \label{eqn:exp-spacing-bound}
  \frac{\E[(Z - r)^2\mid r]}{(r \vee r_{\max} \alpha_0)^2}
  \lesssim \alpha_0^{-4/3} e^{-2\diffp/3} .
\end{equation}

Let us compare with Lemma~\ref{lemma:one-dim-absolute-error}, which yields
\begin{equation*}
  \E[(Z - r)^2 \mid r] = O\left( r_{\max}^2 e^{-2 \diffp / 3}\right).
\end{equation*}
The choice $\alpha_0 = 1$ in inequality~\eqref{eqn:exp-spacing-bound} shows
that the geometrically-binned mechanism can recover this bound.  For $r \le
r_{\max} \alpha_0$, the former inequality~\eqref{eqn:exp-spacing-bound} is
stronger; for example, the choice $\alpha_0 = e^{-\diffp/4}$ yields
that $\E[(Z - r)^2 \mid r]
=O( r_{\max}^2 e^{-5 \diffp / 6})$ for $r \le r_{\max} e^{-\diffp/4}$,
and $\E[(Z - r)^2 \mid r]
=O( r^2 e^{-\diffp / 3})$ otherwise.

\subsection{Proof of Lemma~\ref{lemma:one-dim-relative-error}}
\label{sec:proof-scalar-relative-utility}

To see that $\ScalarRelDP(\cdot,k,\alpha, \nu; \diffp, r_{\max})$ is differentially private, we point out that randomized response is $\diffp$-DP and then DP is closed under post-processing. To see that $Z$ is unbiased, note that
\begin{equation*}
  \E[\tilde{J} \mid J = i]
  = \left(\frac{e^\diffp - 1}{e^\diffp + k} \right)
  \nu^i \1{i \ge 1} + \frac{1}{e^\diffp + k} \sum_{j = 1}^k \nu^j,
\end{equation*}
and thus for $r \in \interval_{i^*}$,
\begin{align*}
  \E[Z] & = a \cdot ( \E[\hat{J}] - b) =  \alpha \cdot \left( \frac{e^\diffp + k}{e^\diffp - 1} \right) \left( \left(\frac{e^\diffp - 1}{e^\diffp + k} \right)
  \E[\nu^J \1{J \ge 1}] + \frac{1}{e^\diffp + k} \cdot  \sum_{j = 1}^k \nu^j - \frac{1}{e^\diffp + k} \cdot \sum_{j=1}^k \nu^{j}\right) \\
  & = \alpha \cdot \E\left[ \nu^{J} \1{J \geq 1}\right] 
  \\
  & = \alpha \cdot  \left[ \1{i^* = 0} \cdot \left( \frac{r}{\alpha \nu} \right) \cdot \nu + \1{i^* \geq 1} \cdot \left( \left(  \frac{\nu^{i^* + 1} \alpha - r}{
      \nu^{i^*}(\nu - 1) \alpha}\right)  \cdot \nu^{i^*} + \left( \frac{r - \nu^{i^*} \alpha}{\nu^{i^*}(\nu - 1) \alpha}\right) \cdot \nu^{i^* +1}\right) \right] \\
  & = \1{i^* = 0} \cdot r + \1{i^* \geq 1} \cdot r = r
\end{align*}
Consider the following mean squared error terms
\begin{align*}
  \E\left[(r - \alpha \nu^{J} \cdot \1{J\ge 1})^2 \mid r
    \in \interval_{0} \right] & = r^2 \left(1 - \frac{r}{\alpha \nu} \right)  + (r - \alpha \nu)^2 \left(\frac{r}{\alpha \nu} \right)
  = r \left( \alpha \nu - r\right) \leq (\nu\alpha)^2
\end{align*}
And for $i^* \geq 1$
\begin{align*}
  \E\left[(r - \alpha \nu^{J} \cdot \1{J\ge 1})^2 \mid r
    \in \interval_{i^*} \right] 
  \le (\nu^{i^* + 1} \alpha - \nu^{i^*} \alpha)^2 = \nu^{2i^*} (\nu^{i^*} - 1)^2 \alpha^2.
\end{align*}
We then bound the conditional expectation
\begin{equation*}
  \var(Z \mid J = i^*)
  = a^2 \var(\tilde{J} \mid J = i^*)
  = \alpha^2 \cdot\left( \frac{e^\diffp + k}{e^\diffp - 1} \right) \cdot \var(\tilde{J} \mid J = i^*),
\end{equation*}
Now we note that
\begin{align*}
  \var(\tilde{J} \mid J = i^*)
  & =\left(\frac{e^\diffp - 1}{e^\diffp + k} \right)\cdot
  \nu^{2i^*} \1{i^* \ge 1} +  \left(\frac{1}{e^\diffp + k} \right)\cdot
  \sum_{j = 1}^k \nu^{2j}
  \\
  & \qquad - \left( \left(\frac{e^\diffp - 1}{e^\diffp + k} \right) \cdot  \nu^{i^*} \1{i^* \ge 1}
  + \left(\frac{1}{e^\diffp + k} \right)\cdot  \sum_{j = 1}^k \nu^j \right)^2 \\
  & \le \left( \frac{(k + 1)(e^\diffp - 1)}{(e^\diffp + k)^2} \right) \cdot
  \nu^{2i^*} \1{i^* \ge 1}
  + \left(\frac{1}{e^\diffp + k} \right)\cdot \nu^{2k} \cdot \sum_{j = 0}^{k-1} \nu^{-2j} \\
  & = \left(\frac{(k + 1)(e^\diffp - 1)}{(e^\diffp + k)^2} \right)\cdot
  \nu^{2i^*} \1{i^* \ge 1}
  + \left(\frac{1}{e^\diffp + k} \right)\cdot \nu^{2k} \cdot \left(\frac{1 - \nu^{-2k}}{1 - \nu^{-2}} \right).
\end{align*}
Now, if $r \in \interval_{i^*}$, we have that
$\nu^{i^*} \alpha \le (r \vee \alpha) \le \nu^{i^* + 1} \alpha$, and thus
\begin{align*}
  \frac{\E[(Z - r)^2]}{(r \vee \alpha)^2}
  & \le \frac{ \alpha^2}{(r \vee \alpha)^2}  \cdot\left( \frac{e^\diffp + k}{e^\diffp - 1} \right)^2 \cdot \left[
    \frac{(k + 1)(e^\diffp - 1)}{(e^\diffp + k)^2}
    \nu^{2(i^* + 1)}
    + \frac{\nu^{2k}}{e^\diffp + k} \left( \cdot \frac{1 - \nu^{-2k}}{1 - \nu^{-2}} \right) \right] \\
  & \qquad  + \frac{ \alpha^2}{(r \vee \alpha)^2}  \cdot \nu^{2i^*}(\nu - 1)^2 \\
  & \le \frac{(k + 1)}{(e^\diffp -1 )}
  \nu^2 + \left( \frac{\nu^{2k}\cdot (e^\diffp + k)}{(e^\diffp -1)^2 }  \right)
  \left(\frac{1 - \nu^{-2k}}{1 - \nu^{-2}}\right) + (\nu - 1)^2,
\end{align*}
as desired.


\section{Proofs of utility in private sampling mechanisms}
\label{sec:proofs-utility}

\subsection{Proof of Lemma~\ref{lem:unbiased}}
\label{sec:proof-l2-unbiased}

Let $u \in \sphere^{d-1}$ and $U \in \sphere^{d-1}$ be a uniform random
variable on the unit sphere.  Then
rotational symmetry implies that for some $\gamma_+ > 0 > \gamma_-$,
\begin{equation*}
  \E[U \mid \<U, u\> \ge \gamma]
  = \gamma_+ \cdot u
  ~~ \mbox{and} ~~
  \E[U \mid \<U, u\> < \gamma] = \gamma_- \cdot u,
\end{equation*}
and similarly, the random vector $V$ in Algorithm~\ref{alg:unit_mech}
satisfies
\begin{equation*}
  \E[V \mid u]
  = p \E[U \mid \<U, u\> \ge \gamma]
  + (1 - p) \cdot \E[U \mid \<U, u\> < \gamma]
  = p (\gamma_+ + \gamma_-) \cdot u.
\end{equation*}
By rotational symetry, we may assume $u = e_1$, the first standard
basis vector, without loss of generality. We now compute the normalization
constant. Letting
$U \sim \uniform(\sphere^{d-1})$ have coordinates $U_1, \ldots, U_d$, we
marginally $U_i \eqdist 2B - 1$ where $B \sim \betadist(\frac{d-1}{2},
\frac{d-1}{2})$. Now, we note that
\begin{equation*}
  \gamma_+ = \E[U_1 \mid U_1 \geq \gamma]
  = \E[2B - 1 \mid B \geq \tfrac{1+ \gamma}{2}]
  ~~~ \mbox{and}~~~
  \gamma_- = \E[U_1 \mid U_1 < \gamma]
  = \E[2B -1 \mid B < \tfrac{1+ \gamma}{2}]
\end{equation*}
and that if $B \sim \betadist(\alpha, \beta)$, then for $0\leq\tau\leq 1$,
we have
\begin{align*}
  \E[B \mid B \ge \tau]
  & = \frac{1}{B(\alpha,\beta)\cdot \P(B \ge \tau)}
  \int_\tau^1 x^\alpha (1 - x)^{\beta - 1} dx
  = \frac{B(\alpha + 1, \beta) - B(\tau; \alpha + 1, \beta)}{
    B(\alpha, \beta) - B(\tau; \alpha,\beta)}
\end{align*}
and similarly $\E[B \mid B < \tau] = \frac{B(\tau; \alpha +
  1,\beta)}{B(\tau;\alpha,\beta)}$.  Using that $B(\tau; \alpha + 1,
\alpha) = B(\alpha+1, \alpha)[
  \frac{B(\tau;\alpha,\alpha)}{B(\alpha,\alpha)} - \frac{\tau^\alpha(1 -
    \tau)^\alpha}{\alpha B(\alpha,\alpha)}]$ and $B(\alpha+1, \alpha) =
\half B(\alpha,\alpha)$, then substituting in our calculation for
$\gamma_+$ and $\gamma_-$, we have for $\tau = \frac{1 + \gamma}{2}$ and
$\alpha = \frac{d-1}{2}$ that
\begin{equation*}
  \gamma_+ = \frac{1}{\alpha \cdot 2^{d-1}} \cdot
  \frac{(1-\gamma^2)^\alpha}{B(\alpha, \alpha) - B(\tau; \alpha, \alpha ) }
  ~~~ \mbox{and} ~~~
  \gamma_- = \frac{-1}{\alpha 2^{d-1}} \cdot
  \frac{(1-\gamma^2)^{\alpha}}{B(\tau; \alpha,\alpha) }
\end{equation*}
Consequently, $\E[V \mid u] = (p \gamma_+ + (1 - p) \gamma_-) u$, and
if $Z = \PrivUnit(u, \gamma, p)$ we have $\E[Z] = u$ as desired.

\subsection{Proof of Proposition~\ref{proposition:ltwo-utility}}
\label{sec:proof-ltwo-utility}

Recall from the proof of Lemma~\ref{lem:unbiased} that $Z = \frac{1}{p
  \gamma_+ + (1 - p)\gamma_-} V$, where $\E[U \mid \<U, u\> \ge \gamma] =
\gamma_+ \cdot u$ and $\E[U \mid \<U, u\> < \gamma] = \gamma_- \cdot u$
for $U = (U_1, \cdots, U_d) \sim \uniform(\sphere^{d-1})$, so that
$\gamma_+=\E[U_1 \mid U_1 \ge \gamma] \ge \gamma$ and, using $\E[|U_1|]
\le \E[U_1^2]^{1/2} \le 1 / \sqrt{d}$, we have $\gamma_- = \E[U_1 \mid U_1
  \le \gamma] \in [-\E[|U_1|],0] \in [-1/\sqrt{d},0]$.
Summarizing, we always have
\begin{equation*}
  \gamma \le \gamma_+ \le 1, ~~~
  -\frac{1}{\sqrt{d}} \le \gamma_- \le 0,
  ~~~ \mbox{and} ~~~
  |\gamma_+| > |\gamma_-|.
\end{equation*}
As a consequence, as the norm of $V$ is 1 and $\ltwo{Z} = 1 / (p \gamma_+
+ (1 - p) \gamma_-)$, which is decreasing to $1 / \gamma_+$ as $p \uparrow
1$, we always have $\ltwo{Z} \le \frac{2}{\gamma_+ + \gamma_-}$ and we may
assume w.l.o.g.\ that $p = \half$ in the remainder of the derivation.

Now, we consider three cases in the
inequalities~\eqref{eqn:sufficient-gamma}, deriving lower bounds
on $\gamma_+ + \gamma_-$ for each.
\begin{enumerate}[1.]
\item First, assume $5 \le \diffp \le 2 \log d$. Let $\gamma = \gamma_0
  \sqrt{2/d}$ for some $\gamma_0 \ge 1$. Then the choice $\gamma_0 = 1$
  guarantees that second inequality of~\eqref{eqn:big-suff-gamma} holds,
  while we note that we must have $\gamma_0^2 \le \frac{d}{d-1} \diffp$,
  as otherwise $\frac{d-1}{2} \log(1 - \gamma_0^2 \frac{2}{d}) \le
  -\frac{d-1}{d} \gamma_0^2 < -\diffp$, contradicting the
  inequality~\eqref{eqn:big-suff-gamma}.  For $\gamma_0 \in [1,
    \sqrt{\frac{d}{d-1} \diffp}]$, we have $\log(1 - \gamma_0^2
  \frac{2}{d}) \ge -\frac{8\gamma_0^2}{3d}$ for sufficiently large $d$,
  and solving the first inequality in Eq.~\eqref{eqn:big-suff-gamma} we
  see it is sufficient that
  \begin{equation*}
    \frac{4(d-1)}{3d} \gamma_0^2
    \le \diffp - \log \frac{2d \diffp}{d-1} - \log 6
    ~~ \mbox{or} ~~
    \gamma^2
    \le \frac{6 \diffp - 6\log 6 -3 \log \frac{2d \diffp}{d-1}}{ 4(d - 1)}.
  \end{equation*}
  With this choice of $\gamma$, we obtain
  $\gamma_+ + \gamma_- \ge c \sqrt{\frac{d}{\diffp}}$ for a numerical
  constant $c$.
\item 
  For $d \ge \diffp \ge 5$, it is evident that (for some numerical
  constant $c$) the choice $\gamma = c \sqrt{\diffp / d}$ satisfies
  inequality~\eqref{eqn:big-suff-gamma}. Thus
  $\gamma_+ + \gamma_- \ge c \sqrt{\frac{d}{\diffp}}$.
\item Finally, we consider the last case that $\diffp \le 5$. In this case,
  the choice
  $\gamma^2 = \pi (e^\diffp - 1)^2 / (2d (e^\diffp + 1)^2)$
  satisfies inequality~\eqref{eqn:small-suff-gamma}. We need to
  control the difference
  $\gamma_+ + \gamma_- = \E[U_1 \mid U_1 \ge \gamma]
  + \E[U_1 \mid U_1 \le \gamma]$. In this case,
  let $p_+ = \P(\<U, u\> \ge \gamma)$ and
  $p_- = \P(\<U, u\> < \gamma)$, so that
  Lemma~\ref{lemma:small-threshold-uniform-bound-v2} implies that
  $p_+ \le \half - e^{-2} \gamma \sqrt{\frac{d-1}{2 \pi}}$ and
  $p_- \ge \half + \gamma \sqrt{\frac{d-1}{2 \pi}}$.
  Then
  \begin{align*}
    \E[U_1 \mid U_1 \ge \gamma]
    + \E[U_1 \mid U_1 < \gamma]
    & = \frac{1}{p_+} \E[U_1 \cdot \1{U_1 \ge \gamma}]
    + \frac{1}{p_-} \E[U_1 \cdot \1{U_1 < \gamma}] \\
    & = \left(\frac{1}{p_+} - \frac{1}{p_-}\right)
    \E[U_1 \cdot \1{U_1 \ge \gamma}]
    \ge \left(1 - \frac{p_+}{p_-}\right) \E[U_1 \mid U_1 \ge 0],
  \end{align*}
  where the second equality follows from the fact that $\E[U_1] = 0$.  Using
  that
  $\E[U_1 \mid U_1 \ge 0] \ge c d^{-1/2}$ for a numerical constant $c$ and
  \begin{equation*}
    1 - \frac{p_+}{p_-}
    \ge
    \frac{4 e^{-2} \sqrt{\tfrac{2(d-1)}{\pi}}}{
      1 + 2 e^{-2} \gamma \sqrt{\tfrac{2 (d-1)}{\pi}}}
    = \Omega \left(\frac{e^\diffp - 1}{e^\diffp + 1}\right)
  \end{equation*}
  by our choice of $\gamma$,
  we obtain $\gamma_+ + \gamma_- \gtrsim (e^\diffp - 1) \sqrt{d}$.
\end{enumerate}

Combining the three cases above, we use that $V$ in Alg.~\ref{alg:unit_mech}
has norm $\ltwo{V} = 1$ and
\begin{equation*}
  \ltwo{Z} \le \frac{2}{\gamma_+ + \gamma_-} \ltwo{V}
  \le c \sqrt{d \cdot \max\left\{\diffp^{-1}, \diffp^{-2}\right\}}
\end{equation*}
to obtain the first result of the proposition.

The final result of the proposition is immediate by the bound on
$\ltwo{Z}$.

\subsection{Proof of Lemma~\ref{lem:unbiasedINFTY}}
\label{sec:proof-unbiased-infinity}

Without loss of generality, assume that $u \in \{-1, 1\}^d$, as it is
clear that $\E[\what{U} \mid u] = u$ in
Algorithm~\ref{alg:unit_mechINFTY}.  We now show that given $u$,
$\E[V \mid u = u] = m \cdot u$.  Consider $U \sim
\uniform(\{-1,+1 \}^{d} )$, in which case
\begin{equation*}
  \E[V \mid u = u] = p \E[U \mid \<U, u \> > \kappa]  +
  (1 - p) \E[U \mid \<U, u \> \leq \kappa] .
\end{equation*}
For constants $\kappa_+, \kappa_-$, uniformity of $U$ implies that
\begin{equation*}
  \E[U \mid \<U, u \> > \kappa] = \kappa_+ \cdot u \qquad
  \mbox{and}
  \qquad \E[U \mid \<U, u \> \leq \kappa] =   \kappa_- \cdot u.
\end{equation*}
By symmetry it is no loss of generality to assume that
$u = (1, 1, \cdots, 1)$, so
$\< U, u\> = \sum_{\ell=1}^d U_\ell$.
We then have for $\tau = \lceil \frac{d+\kappa+1}{2} \rceil / d$ that
\begin{equation*}
  \E[U_1 \mid \<U, u \> > \kappa] = \frac{1}{2^d \cdot \P(\< U, u\> > \kappa) } \cdot \sum_{\ell = d \tau}^d \left( \choose{d-1}{\ell -1} -  \choose{d-1}{\ell} \right) 
  =  \frac{ \choose{d-1}{d \tau-1} }{\sum_{\ell = d\tau}^d \choose{d}{\ell} }
  \eqdef \kappa_+
\end{equation*}
and 
$$
\E[U_1 \mid \<U, u \> \leq \kappa] = \frac{1}{2^d\cdot \P(\< U, u\> \leq \kappa)} \cdot \sum_{\ell = 0}^{d \tau -1} \left( \choose{d-1}{\ell -1} -  \choose{d-1}{\ell} \right) 
=  \frac{ -\choose{d-1}{d \tau-1 } }{\sum_{\ell = 0}^{d \tau-1} \choose{d}{\ell} }
\eqdef \kappa_-.
$$
Putting these together and setting $m = p \kappa_+ + (1 - p) \kappa_-$,
we have $\E[(1/m) V \mid u] = u$.

\subsection{Proof of Proposition~\ref{proposition:linf-utility}}
\label{sec:proof-linf-utility}

  Using the notation of the proof of
  Lemma~\ref{lem:unbiasedINFTY},
  the debiasing multiplier $m$ satisfies
  $1/m = \frac{1}{p \kappa_+ +
    (1 - p)\kappa_-} \le \frac{2}{\kappa_+ + \kappa_-}$,
  as $p \ge \half$.
  Thus, we seek to lower bound $\kappa_+ + \kappa_-$
  by lower bounding $\kappa_+$ and $\kappa_-$ individually.
  We consider two cases: the case that $\diffp \ge \log d$ and the case
  that $\diffp < \log d$.

  \paragraph{Case 1:} when $\diffp \ge \log d$.
  We first focus on lower bounding
  \begin{equation*}
  \kappa_+ =  \frac{ \choose{d-1}{d \tau -1} }{\sum_{\ell = d\tau}^d \choose{d}{\ell} }=  \tau \cdot \frac{ \choose{d}{d\tau} }{\sum_{\ell = d\tau}^d \choose{d}{\ell} } = \tau \cdot \left( \sum_{\ell = d \tau}^d \choose{d}{d\tau}^{-1} \choose{d}{\ell} \right)^{-1}.
  \end{equation*}
  We argue that eventually the terms in the summation become
  small. For $\ell \ge d \tau$ defining
  \begin{align*}
    r \defeq \frac{d \tau + 1}{d - d\tau}
    \le \frac{\ell + 1}{d - \ell}
    = \frac{\choose{d}{\ell} }{ \choose{d}{\ell+1} }
    ~~~ \mbox{implies} ~~~ \choose{d}{\ell} \geq r \choose{d}{\ell+1},
  \end{align*}
  We then have 
  \begin{equation*}
    \sum_{\ell = d\tau}^d \choose{d}{d\tau}^{-1} \choose{d}{\ell} = \sum_{i = 0}^{d-d\tau} \choose{d}{d\tau}^{-1} \choose{d}{d\tau + i} 
    \leq  \sum_{i = 0}^{d-d \tau} \choose{d}{d\tau}^{-1} \choose{d}{d \tau} \left(\frac{1}{r}\right)^i = \sum_{i = 0}^{d-d \tau} \frac{1}{r^i} \leq \frac{1}{1 - 1/r},
  \end{equation*}
so that
\begin{equation}
  \label{eqn:kappa-plus-lb}
  \kappa_+ \geq 1 - 1/r
  =  1 - \frac{d - d \tau}{d\tau + 1}
  =  \frac{2d\tau - d + 1}{d \tau + 1}
  = \begin{cases}
    2 \cdot  \frac{\kappa + 2}{d+\kappa+3}	& \text{for odd } d + \kappa \\
    2 \cdot  \frac{\kappa + 3}{d+\kappa+4}	&  \text{for even } d + \kappa
  \end{cases}
  \ge \frac{\kappa}{d}.
\end{equation}

We now lower bound $\kappa_-$, where we use the fact that $d \tau - 1 \geq
\frac{d-1}{2}$ for $\kappa \in \{0,1,\cdots, d \}$
to obtain
\begin{align*}
  \kappa_-
  = \frac{ - \choose{d-1}{d \tau - 1} }{\sum_{\ell = 0}^{\tau \cdot d-1} \choose{d}{\ell} }
  = - \tau \frac{\choose{d}{d\tau}}{ \sum_{\ell = 0}^{d\tau -1} \choose{d}{\ell} }
  \ge -\tau  \frac{\choose{d}{d\tau} }{2^{d-1} }.
\end{align*}
Using Stirling's approximation and that
$\tau = \lceil\frac{d + \kappa +1}{2} \rceil / d$,
we have
\begin{equation*}
  \choose{d}{d \tau}
  = C(\kappa, d) \sqrt{\frac{1}{d}}
  \cdot 2^d \cdot \exp\left(-\frac{\kappa^2}{2d}\right),
\end{equation*}
where $C(\kappa, d)$ is upper and lower bounded by positive universal
constants.
Hence, we have for a constant $c < \infty$,
\begin{equation}
  \label{eqn:kappa-minus-lb}
  \kappa_- \geq - c \cdot  \left( \frac{1}{\sqrt{d}} \cdot \exp\left(  \frac{-\kappa^2}{2d} \right) \right).
\end{equation}

An inspection of the
bound~\eqref{eq:sufficient_kappa3} shows that the choice $\kappa = c
\sqrt{\diffp d}$ for some (sufficiently small) constant $c$ immediately
satisfies the sufficient condition for Algorithm~\ref{alg:unit_mechINFTY} to
be private. Substituting this choice of $\kappa$ into the lower
bounds~\eqref{eqn:kappa-plus-lb} and~\eqref{eqn:kappa-minus-lb} gives that
the normalizer $m^{-1} \le \frac{2}{\kappa_+ + \kappa_-} \lesssim \sqrt{d /
  \diffp}$, which is first result of
Proposition~\ref{proposition:linf-utility}.

\paragraph{Case 2:} when $\diffp < \log d$. In this case,
we use the bound~\eqref{eq:sufficient_kappa1} to obtain the result.  Let us
first choose $\kappa$ to saturate the bound~\eqref{eq:sufficient_kappa1},
for which it suffices to choose $\kappa = c \min\{\sqrt{d}, \diffp
\sqrt{d}\}$ for a numerical constant $c > 0$. We assume for simplicity that
$d$ is even, as extending the argument is simply notational.  Defining the
shorthand $s_\tau \defeq \frac{1}{2^{d-1}} \sum_{\ell = d/2}^{d\tau - 1}
\choose{d}{\ell}$, we recall the definitions of $\kappa_+$ and $\kappa_-$ to
find
\begin{equation*}
  \kappa_+ = \tau \cdot \choose{d}{d\tau} \cdot \frac{1}{1 - s_\tau}
  ~~~ \mbox{and} ~~~
  \kappa_- = \tau \cdot \choose{d}{d\tau} \cdot
  \frac{1}{1 + s_\tau}.
\end{equation*}
Using the definition of the debiasing normalizer
$m = p \kappa_+ + (1 - p)\kappa_-
\ge \half(\kappa_+ + \kappa_-)$, we obtain
\begin{align*}
  m & \ge
  \frac{\tau}{2^d} \choose{d}{d\tau}
  \left(\frac{1}{1 - s_\tau} - \frac{1}{1 + s_\tau}\right)
  = \frac{\tau}{2^d} \choose{d}{d\tau}
  \frac{2 s_\tau}{1 - s_\tau^2}.
\end{align*}
By definition of $s_\tau$, we have
$s_\tau \ge \frac{d \tau - d/2}{2^{d-1}} \choose{d}{d\tau}
\ge \frac{\kappa}{2^d} \choose{d}{d\tau}$,
so that
\begin{align*}
  m & \gtrsim
  \tau \kappa \left(2^{-d} \binom{d}{d\tau}\right)^2
  \gtrsim \frac{1}{d} \kappa \exp\left(-\frac{\kappa^2}{d}\right),
\end{align*}
where the second inequality uses Stirling's approximation.
Our choice of $\kappa = c \min\{\sqrt{d}, \diffp \sqrt{d}\}$ thus yields
$m \gtrsim \frac{1}{\sqrt{d}} \min\{1, \diffp\}$, which gives
the second result of Proposition~\ref{proposition:linf-utility}.
\section{Uniform random variables and concentration}

\newcommand{\threshold}{\gamma}

\label{sec:preliminaries-uniform-concentration}

In this section, we collect a number of results on the concentration
properties of variables uniform on the unit sphere $\sphere^{d-1}$, which
allow our analysis of the mechanism $\PrivUnit$ for privatizing vectors
in the $\ell_2$ ball in Algorithm~\ref{alg:unit_mech}.  For a vector $u \in
\sphere^{d-1}$, that is, satisfying $\ltwo{u} = 1$, and $a \in [0, 1]$ we
define the spherical cap
\begin{equation*}
  C(a, u) \defeq
  \left\{v \in \sphere^{d-1} \mid \<v, u\> > a \right\}.
\end{equation*}
There are a number of bounds on the probability
that $U \in C(a, u)$ for a fixed $u\in \sphere^{d-1}$ where $U \sim \uniform(\sphere^{d-1})$, which
the following lemma summarizes.
\begin{lemma}
  \label{lemma:uniform-bounds}
  Let $U$ be uniform on the unit sphere $\sphere^{d-1}$ in $\R^d$.
  Then for
  $\sqrt{2/d} \le a \le 1$ and $u \in \sphere^{d-1}$,
  \begin{equation}
    \frac{1}{6 a \sqrt{d}} (1 - a^2)^{\frac{d-1}{2}}
    \le 
    \P(U \in C(a, u))
    \le \frac{1}{2 a \sqrt{d}}
    (1 - a^2)^\frac{d-1}{2}.
        \label{eq:both_bounds}
  \end{equation}
  For all $0 \le a \le 1 $, we have
  \begin{equation}
    \frac{(1-a)}{2d} (1 - a^2)^\frac{d-1}{2}
    \le \P(U \in C(a, u))
    \label{eq:lower_bound}
  \end{equation}
  and for $a \in [0, 1 / \sqrt{2}]$,
  \begin{equation*}
    \P(U \in C(a, u)) \le (1 - a^2)^\frac{d}{2}.
  \end{equation*}
\end{lemma}
\begin{proof}
  The first result is \cite[Exercise 7.9]{BoucheronLuMa13}. The lower bound
  of the second is due to \cite[Lemma 4.1]{MicciancioVo10}, while the third
  inequality follows for all $a \in [0,
    1/\sqrt{2}]$ by
  \cite[Proof of Lemma 2.2]{Ball97}. 
\end{proof}
\noindent



We also require a slightly different lemma for small values of the
threshold $a$ in Lemma~\ref{lemma:uniform-bounds}.

\begin{lemma}
  \label{lemma:small-threshold-uniform-bound}
  Let $\threshold \ge 0$ and U be uniform on $\sphere^{d-1}$. Then
  \begin{equation*}
    \threshold \sqrt{\frac{d-1}{2\pi}}
    \exp\left(-\frac{1}{4d - 4} - 1\right)
    \1{\threshold \le \sqrt{2 / (d - 3)}}
    \le 
    \P(\<U, u\> \in [0, \threshold])
    \le
    \threshold \sqrt{\frac{d-1}{2\pi}}
  \end{equation*}
\end{lemma}
\begin{proof}
  For any fixed unit vector $u \in \sphere^{d-1}$ and $U \sim
  \uniform(\sphere^{d-1})$, we have that marginally $\<U, u\>
  \sim 2B - 1$ for $B \sim \betadist(\frac{d-1}{2}, \frac{d-1}{2})$.
  Thus we have
  \begin{align*}
    \P(\<U, u\> \in [0, \threshold])
    & = \P\left(\half \le B \le \frac{1 + \threshold}{2}\right)
  \end{align*}
  We will upper and lower bound the last probability above.
  Letting $\alpha = \frac{d-1}{2}$ for shorthand, we have
  \begin{align*}
    \P\left(\half \le B \le \frac{1 + \threshold}{2}\right)
    & = \frac{\Gamma(2\alpha)}{\Gamma(\alpha)^2}
    \int_{\half}^{\frac{1 + \threshold}{2}} t^{\alpha - 1}(1 - t)^{\alpha - 1} dt \\
    & \stackrel{(i)}{=} \frac{\Gamma(2\alpha)}{2 \Gamma(\alpha)^2}
    \int_0^\threshold \left(\frac{1 + u}{2}\right)^{\alpha - 1}
    \left(\frac{1 - u}{2}\right)^{\alpha - 1} du 
    = \frac{\Gamma(2 \alpha)}{2^{2\alpha - 1} \Gamma(\alpha)^2}
    \int_0^\threshold(1 - u^2)^{\alpha - 1} du,
  \end{align*}
  where equality~$(i)$ is the change of variables $u = 2t - 1$.
  Using Stirling's approximation, we have that
  \begin{equation*}
    \log \Gamma(2\alpha) - 2 \log \Gamma(\alpha)
    = (2 \alpha - 1) \log 2 + \half \log \frac{\alpha}{\pi}
    + \mbox{err}(\alpha),
  \end{equation*}
  where $\mbox{err}(\alpha) \in [-\frac{1}{8\alpha}, -\frac{1}{8\alpha +
      1}]$\rynote{Need to verify this error}.
  When $\threshold \leq \sqrt{1/(\alpha - 1)}$, we have 
  $$
  \int_0^\threshold (1-u^2)^{\alpha -1}du \geq \int_0^{\threshold} \left(1-\left( \sqrt{ \frac{1}{\alpha - 1} } \right)^2\right)^{\alpha-1} du \geq \threshold e^{-1}
  $$
 Otherwise, if $\threshold > \sqrt{1/(\alpha - 1)}$ we have the trivial bound $\int_0^\threshold (1-u^2)^{\alpha -1}du \geq 0$.  Furthermore, for $\threshold \geq 0$ we have $ \int_0^\threshold (1 - u^2)^{\alpha - 1}du \leq \threshold$.  Putting this all together, we have
  \begin{align*}
    \exp\left(-\frac{1}{8 \alpha}\right) \cdot
    & \left( \sqrt{\frac{\alpha}{\pi}} \right) \cdot
    \left(\threshold e^{-1}\right) \cdot 
    \1{\threshold \le (\alpha - 1)^{-\half}} \\
    & \le
    \exp(\mbox{err}(\alpha))
    \sqrt{\frac{\alpha}{\pi}} \int_0^\threshold(1 - u^2)^{\alpha - 1} du
    = \P\left(\half \le B \le \frac{1 + \threshold}{2}\right)
    \le \threshold \sqrt{\frac{\alpha}{\pi}}.
  \end{align*}
  Substituting $\alpha \mapsto \frac{d-1}{2}$ yields the desired
  upper and lower bounds on
  $\P(\<U, u\> \in [0,\threshold])$.
\end{proof}

As a consequence of Lemma~\ref{lemma:small-threshold-uniform-bound},
we have the following result.
\begin{lemma}
  \label{lemma:small-threshold-uniform-bound-v2}
  Let $\threshold \in [0, \sqrt{2 / (d - 3)}]$ and U be uniform on $\sphere^{d-1}$. Then
  \begin{equation*}
    \half - \threshold \sqrt{\frac{d - 1}{2\pi}}
    \le \P(U \in C(\threshold, u))
    \le \half - \threshold \sqrt{\frac{d-1}{2\pi}}
    e^{-\frac{4d - 3}{4d - 4}}.
  \end{equation*}
\end{lemma}
\begin{proof}
We have
$$
\P\left(U \in C(\threshold, u) \right) = 1 - \P\left(\<U,u \> < \threshold\right) = 1/2 - \P\left(\<U,u \> \in [0,\threshold)\right).
$$
We then use Lemma~\ref{lemma:small-threshold-uniform-bound}
\end{proof}

\end{document}